\documentclass[11pt]{article}
\usepackage{natbib}
\usepackage{bbm,amsmath,amsthm,amsfonts, amsopn, amssymb}
\usepackage{color}
\usepackage{graphicx, subcaption, epstopdf}
\usepackage{mathtools}
\usepackage[linesnumbered,lined,boxed,commentsnumbered]{algorithm2e}
\usepackage{algorithmic}

\newcommand{\defn}{\ensuremath{: \, =}}
\newcommand{\widgraph}[2]{\includegraphics[keepaspectratio,width=#1]{#2}}

\newtheorem{theorem}{Theorem}[section]

\newtheorem{corollary}[theorem]{Corollary}

\newtheorem{lemma}[theorem]{Lemma}

\newcommand{\BALD}{\begin{aligned}}
\newcommand{\EALD}{\end{aligned}}
\newcommand{\BALDS}{\begin{aligned*}}
\newcommand{\EALDS}{\end{aligned*}}
\newcommand{\BCAS}{\begin{cases}}
\newcommand{\ECAS}{\end{cases}}
\newcommand{\BEAS}{\begin{eqnarray*}}
\newcommand{\EEAS}{\end{eqnarray*}}
\newcommand{\BEQ}{\begin{equation}}
\newcommand{\EEQ}{\end{equation}}
\newcommand{\BIT}{\begin{itemize}}
\newcommand{\EIT}{\end{itemize}}
\newcommand{\BMAT}{\begin{bmatrix}}
\newcommand{\EMAT}{\end{bmatrix}}
\newcommand{\BNUM}{\begin{enumerate}}
\newcommand{\ENUM}{\end{enumerate}}

\newcommand{\BA}{\begin{array}}
\newcommand{\EA}{\end{array}}

\DeclareMathOperator*{\argmin}{\arg\min}

\newcommand{\norm}[1]{\left\| #1 \right\|}

\newcommand{\cA}{\mathcal{A}}
\newcommand{\cB}{\mathcal{B}}

\newcommand{\cE}{\mathcal{E}}

\newcommand{\cL}{\mathcal{L}}
\newcommand{\cLs}{\mathcal{L}^\ast}
\newcommand{\tcL}{\widetilde{\mathcal{L}}}
\newcommand{\cN}{\mathcal{N}}
\newcommand{\cM}{\mathcal{M}}

\newcommand{\cZ}{\mathcal{Z}}

\newcommand{\R}{\mathbb{R}}
\newcommand{\mum}{\mu_{-}}
\newcommand{\mup}{\mu_{+}}

\newcommand{\Ee}{\mathbb{E}}
\newcommand{\Pp}{\mathbb{P}}
\newcommand{\E}{E}

\newcommand{\hSigma}{{\widehat\Sigma}}
\newcommand{\tSigma}{{\widetilde\Sigma}}

\newcommand{\btheta}{\bar \theta}
\newcommand{\thetas} {\theta^\ast}
\newcommand{\htheta}{\widehat \theta}
\newcommand{\ttheta}{\widetilde \theta}
\newcommand{\apost}{\widetilde{\pi}_N}
\newcommand{\lambdamin}{\lambda_{\tiny{\min}}}

\newcommand{\matnorm}[2]{|\!|\!| #1 | \! | \!|_{#2}}

\def\bar{\overline}

\usepackage{JASA_manu}
\begin{document}

\title{Communication-Efficient Distributed Statistical Inference}
\author{Michael I. Jordan, Jason D. Lee, Yun Yang}

\maketitle

%

\newpage
\begin{center}
\textbf{Abstract}
\end{center}
We present a \emph{Communication-efficient Surrogate Likelihood} (CSL) framework 
for solving distributed statistical inference problems.  CSL provides a 
communication-efficient surrogate to the global likelihood that can be 
used for low-dimensional estimation, high-dimensional regularized estimation 
and Bayesian inference.  For low-dimensional estimation, CSL provably improves 
upon naive averaging schemes and facilitates the construction of confidence 
intervals.  For high-dimensional regularized estimation, CSL leads to a 
minimax-optimal estimator with controlled communication cost.  For Bayesian 
inference, CSL can be used to form a communication-efficient quasi-posterior 
distribution that converges to the true posterior.  This quasi-posterior 
procedure significantly improves the computational efficiency of MCMC 
algorithms even in a non-distributed setting.  We present both theoretical
analysis and experiments to explore the properties of the CSL approximation.
\vspace*{.3in}


\newpage

\section{Introduction}
What is the relevance of the underlying computational architecture to statistical
inference?  Classically, the answer has been ``not much''---the naive abstraction 
of a sequential program running on a single machine and providing instantaneous access 
to arbitrary data points has provided a standard conceptual starting point for 
statistical computing.  In the modern era, however, it is commonplace for data 
analyses to run on hundreds or thousands of machines, with the data distributed
across those machines and no longer available in a single central location.  
Moreover, the end of Moore's law has changed computer science---the focus is 
now on parallel, distributed architectures and, on the algorithmic front, 
on divide-and-conquer procedures.  This has serious implications for statistical 
inference.  Naively dividing datasets into subsets that are processed separately, 
with a naive merging of the results, can yield inference procedures that are 
highly biased and highly variable.  Naive application of traditional statistical 
methodology can yield procedures that incur exorbitant communication costs.

Historically, the area in which statisticians have engaged most deeply with practical
computing concerns has been in the numerical linear algebra needed to support regression 
and multivariate statistics, including both sparse and dense matrix algorithms.  It is
thus noteworthy that over the past decade there has been a revolution in numerical 
linear algebra in which new ``communication-avoiding'' algorithms have been developed
to replace classical matrix routines~\cite{DemmelEtAl2012}.  The new algorithms 
can run significantly faster than classical algorithms on parallel, distributed 
architectures.

A statistical literature on parallel and distributed inference has begun to emerge, 
both in a frequentist setting~\citep{DuchiAW12,Zhang13,Kannan14,KleinerEtAl2014,
shamir2014communication,MackeyEtAl2015,Zhang15,lee2015}, and Bayesian 
setting~\citep{SuchardEtAl2010,Cleveland2014,MaclaurinAdams2014,Wang15,
Xing15,RabinovichEtAl2016, ScottEtAl2016,TereninEtAl2016}.  This literature has focused on
data-parallel procedures in which the overall dataset is broken into subsets that
are processed independently.  To the extent that communication-avoiding procedures
have been discussed explicitly, the focus has been on ``one-shot'' or ``embarrassingly 
parallel'' approaches only use one round of communication in which estimators 
or posterior samples are obtained in parallel on local machines, are communicated 
to a center node, and then combined to form a global estimator or approximation 
to the posterior distribution~\citep{Zhang13,lee2015,Wang15,Xing15}.  In the 
frequentist setting, most one-shot approaches rely on averaging \citep{Zhang13}, 
where the global estimator is the average of the local estimators. 
\cite{lee2015} extends this idea  to high-dimensional sparse linear 
regression by combining local debiased Lasso estimates~\citep{vandegeer2014}.
Recent work by~\cite{Duchi15} show that under certain conditions, these 
averaging estimators can attain the information-theoretic complexity lower 
bound---for linear regression, at least $\mathcal{O}(dk)$ bits must be 
communicated in order to attain the minimax rate of parameter estimation, 
where $d$ is the dimension of the parameter and $k$ is the number of 
machines.  This holds even in the sparse setting~\citep{Braverman15}.

These averaging-based, one-shot communication approaches suffer from several 
drawbacks. First, they have generally been limited to point estimation; it 
is not straightforward to create confidence intervals/regions and hypothesis 
tests based on the averaging estimator.  Second, in order for the averaging 
estimator to achieve the minimax rate of convergence, each local machine must 
have access to at least $\Omega(\sqrt{N})$ samples, where $N$ is the total 
sample size.  In other words, the number of machines should be much smaller 
than $\sqrt{N}$; a highly restrictive assumption.  Third, when the statistical
is nonlinear, averaging estimators can perform poorly; for example, our empirical 
study shows that even for small $k$, of order $10^1$, the averaging estimator 
only exhibits a slight improvement over purely local estimators.

In the Bayesian setting, embarrassingly parallel approaches run Markov 
chain Monte Carlo (MCMC) algorithms in parallel  across local machines
and transmit the local posterior samples to a central node to produce an 
overall approximation to the global posterior distribution.  Unfortunately, 
when the dimension $d$ is high, the number of samples obtained locally must
be large due to the curse of dimensionality, incurring significant communication
costs.  Also, when combining local posterior samples in the central node, 
existing approaches that approximate the global posterior distribution 
by a weighted empirical distribution of ``averaging draws" \citep{Wang15,Xing15} 
tend to suffer from the weight-degeneracy issue (weights collapse to only a 
few samples) when $k$ is large.

In this paper, we formulate a unified framework for distributed statistical
inference.  We refer to our framework as the \emph{Communication-efficient 
Surrogate Likelihood} (CSL) framework.  From the frequentist perspective, 
CSL provides a communication-efficient surrogate to the global likelihood 
function that can play the role of the global likelihood function in forming 
the maximum likelihood estimator (MLE) in regular parametric models or 
the penalized MLE in high-dimensional models.  From a Bayesian perspective, 
CSL can be used to form a quasi-posterior distribution \citep{Chernozhukov03} 
as a surrogate for the full posterior.  The CSL approximation can be constructed 
efficiently by communicating $\mathcal{O}(dk)$ bits.  After its construction, 
CSL can be efficiently evaluated by using the $n$ samples in a single local 
machine.  Even in a non-distributed Bayesian setting, CSL can be used as a
computationally-efficient surrogate to the likelihood function by pre-dividing 
the dataset into $k$ subsamples---the computational complexity of one iteration 
of MCMC is then reduced by a factor of $k$. 

Our CSL-based distributed inference approach overcomes the aforementioned 
drawbacks associated with the one-shot and embarrassingly parallel approaches. 
In the frequentist framework, a CSL-based estimator can achieve the same 
rate of convergence as the global likelihood-based estimator while incurring
a communication complexity of only $\mathcal{O}(dk)$.  Moreover, the CSL
framework can readily be applied iteratively, with the resulting multi-round 
algorithm achieving a geometric convergence rate with contraction factor 
$\mathcal{O}(n^{-1/2})$, where $n$ is the number of samples in each local 
machine. This $\mathcal{O}(n^{-1/2})$ rate of convergence significantly 
improves on analyses based on condition-number contraction factors used
to analyze methods that form the global gradient in each iteration by 
combining local gradients.  As an implication, in order to achieve the 
same accuracy as the global likelihood-based estimator, $n$ can be independent 
of the total sample size $N$ as long as $\mathcal{O}(\frac{\log N}{\log n})$ 
iterations are applied, which is \textit{constant} for $n>k$.  In contrast, 
the averaging estimator requires $n\gg\sqrt{N}$.  Thus, due to the fast 
$\mathcal{O}(n^{-1/2})$ rate, usually two-three iterations suffice for 
our procedure to match the same accuracy of the global likelihood-based 
estimator even for relatively large $k$ (See Section~\ref{sec:lowdim} for 
more details).  Unlike bootstrap-based approaches \citep{Zhang13} for 
boosting accuracy, the additional complexity of the iterative version of 
our approach grows only linearly in the number of iterations.  Finally,
our empirical study suggests that a CSL-based estimator may exhibit 
significant improvement over the averaging estimator for nonlinear 
distributed statistical inference problems.

For high-dimensional $\ell_1$-regularized estimation, the CSL framework
yields an algorithm that communicates $O(dk)$ bits, attaining the optimal 
communication/risk trade off  \citep{garg2014communication}.  This improves 
over the averaging method of \cite{lee2015} because it requires $p$ times 
less computation, and allows for iterative refinement to obtain arbitrarily 
low optimization error in a logarithmic number of rounds\footnote{We note
that during the preparation of this manuscript, we became aware of concurrent 
work of \cite{wang2016} who also study a communication-efficient surrogate
likelihood.  Their focus is solely on the high-dimensional linear model
setting.}.

In the Bayesian framework, our method does not require transmitting local 
posterior samples and is thus free from the weight degeneracy issue. 
This makes the communication complexity of our approach considerably 
lower than competing embarrassingly parallel Bayesian computation approaches.

There is also relevant work in the distributed optimization literature, 
notably the distributed approximate Newton algorithm (DANE) proposed in 
\cite{shamir2014communication}.  Here the idea is to combine gradient 
descent with a local Newton method; as we will see, a similar algorithmic
structure emerges from the CSL framework.  DANE has been rigorously 
analyzed only for quadratic objectives, and indeed the analysis in 
\cite{shamir2014communication} does not imply that DANE can improve 
over the gradient method for non-quadratic objectives.  In contrast, 
our analysis demonstrates fast convergence rates for a broad class of 
regular parametric models and high-dimensional models and is not 
restricted to quadratic objectives.  Another related line of work 
is the iterated Hessian sketch (IHS) algorithm~\cite{pilanci2014iterative} 
for constrained least-squares optimization.  DANE applied to quadratic 
problems can be viewed as a special case of IHS by choosing the sketching 
to be a rescaled subsampling matrix; however, the analysis in 
\cite{pilanci2014iterative} only applies to a class of low-incoherence 
sketching matrices that excludes subsampling.

The remainder of this paper is organized as follows. In Section \ref{sec:background}, 
we informally present the motivation for CSL. Section \ref{sec:main} presents 
algorithms and theory for three different problems: parameter estimation in 
low-dimensional regular parametric models (Section~\ref{SectionMestForPara}), 
regularized parameter estimation in the high-dimensional problems 
(Section~\ref{SectionMestWithReg}), and Bayesian inference in regular 
parametric models (Section~\ref{SectionBayes}). Section \ref{sec:simulations} 
presents experimental results in these three settings. All proofs are provided 
in the Appendix.


\section{Background and problem formulation}
\label{sec:background}
We begin by setting up our general framework for distributed statistical 
inference.  We then turn to a description of the CSL methodology, demonstrating
its application to both frequentist and Bayesian inference.


\subsection{Statistical models with distributed data}
\label{SectionBackground}
Let $Z_1^N \defn \{Z_{ij}:\, i = 1,\ldots, n,\, j = 1,\ldots, k\}$ denote 
$N =nk$ identically distributed observations with marginal distribution 
$\Pp_{\thetas}$, where $\{\Pp_{\theta}: \theta\in\Theta\}$ is a family 
of statistical models parametrized by $\theta \in \Theta \subset \R^d$, 
$\Theta$ is the parameter space and $\thetas$ is the true data-generating 
parameter. Suppose that the data points are stored in a distributed manner 
in which each machine stores a subsample of $n$ observations. Let 
$Z_j \defn \{Z_{ij}:\, i = 1,\ldots, n\}$ denote the subsample that 
is stored in the $j$th machine $\cM_j$ for $j=1,\ldots,k$.  Our goal 
is to conduct statistical inference on the parameter $\theta$ while 
taking into consideration the communication cost among the machines. 
For example, we may want to find a point estimator $\htheta$ and an 
associated confidence interval (region). 

Let $\cL: \Theta \times \cZ \mapsto\R$ be a twice-differentiable loss 
function such that the true parameter is a minimizer of the population 
risk $\cLs(\theta) \defn \Ee_{\thetas} [\cL(\theta;\, Z)]$; that is,
 \begin{align}
 \label{EqnPR}
 \thetas \in \argmin_{\theta\in\Theta} \Ee_{\thetas} [\cL(\theta;\, Z)].
 \end{align}
Define the local and global loss functions as
\begin{align}
\cL_j (\theta) &= \frac1n \sum_{i=1}^n \cL(\theta; z_{ij}),\quad \mbox{for $j \in [k]$,}\\
\cL_N (\theta) &= \frac1N \sum_{i=1}^n \sum_{j=1}^k \cL(\theta; z_{ij}) = \frac1k \sum_{j=1}^k \cL_j (\theta).
\end{align}
Here $\cL_j (\theta)$ is the loss function evaluated at $\theta$ by 
using the local data stored in machine $\cM_j$. The negative log-likelihood 
function is a standard example of the loss function $\cL$.


\subsection{Distributed statistical inference}
\label{SectionDSL}
In this subsection, we motivate the CSL methodology by constructing a 
surrogate loss $\tcL:\Theta\times\cZ\mapsto \R$ that approximates the 
global loss function $\cL_N$ in a communication-efficient manner.  We
show that it can be constructed in any local machine $\cM_j$ by communicating 
at most $(k-1)$ $d$-dim vectors. After the construction, $\tcL$ can be 
used to replace the global loss function in various statistical inference
procedures by only using the data in a local machine (see 
Sections~\ref{SectionMestForPara}-\ref{SectionBayes}). We aim to show
that this distributed inference framework can simultaneously achieve 
high statistical accuracy and low communication cost. In this section
we motivate our construction using heuristic arguments; a rigorous analysis 
is provided in Section~\ref{sec:main} to follow.

Our motivation starts from the Taylor series expansion of $\cL_N$. 
Viewing $\cL_N(\theta)$ as an analytic function, we expand it into 
an infinite series:
\begin{align}
\label{EqnTay}
\cL_N (\theta) = \cL_N (\btheta) + \langle \nabla \cL_N (\btheta) , 
\theta- \btheta\rangle +\sum_{j= 2}^{\infty}  \frac{1}{j!}\,\nabla^j 
\cL_N (\btheta)\, ( \theta- \btheta)^{\otimes j}.
\end{align}
Here $\btheta$ is any initial estimator of $\theta$, for example, 
the local empirical loss minimizer $\argmin_\theta \cL_1(\theta)$ 
in the first machine $\cM_1$.  Because the data is split across 
machines, evaluating the derivatives $\nabla^j \cL_N (\btheta)$ 
$(j\geq 1)$ requires one communication round. However, unlike the 
$d$-dim gradient vector $\nabla \cL_N (\btheta)$, the higher-order 
derivatives require communicating more than $O(d^2)$ bits from each 
machine. This reasoning motivates us to replace the global higher-order 
derivatives $\nabla^j \cL_N (\btheta)$ $(j\geq 2)$ with local derivatives, 
leading to the following approximation of $\cL_N$:
\begin{equation}
\label{EqnTayApp}
\begin{aligned}
\tcL(\theta) = \cL_N (\btheta) 
+ \langle \nabla \cL_N (\btheta) , \theta- \btheta\rangle 
+ \sum_{j=2}^{\infty}  \frac{1}{j!}\,\nabla^j \cL_1 (\btheta) 
\,(  \theta- \btheta)^{\otimes j}.
 \end{aligned}
\end{equation}
Comparing expressions~\eqref{EqnTay} and \eqref{EqnTayApp}, we see 
that the approximation error is:
\begin{align}
\tcL(\theta) - \cL_N(\theta) &=
 \cL_N (\btheta) + \langle \nabla \cL_N (\btheta) , \theta- \btheta\rangle 
  +\sum_{j= 2}^{\infty}  \frac{1}{j!}\,\nabla^j \cL_1 (\btheta)\, 
   ( \theta- \btheta)^{\otimes j}\nonumber\\
 &-\left( \cL_N (\btheta) + \langle \nabla \cL_N (\btheta) , \theta- \btheta\rangle 
  +\sum_{j= 2}^{\infty}  \frac{1}{j!}\,\nabla^j \cL_N (\btheta)\, 
   ( \theta- \btheta)^{\otimes j} \right)\nonumber\\
&= \frac{1}{2} \, \big\langle \theta - \btheta, \, 
\big(\nabla^2 \cL_1(\btheta) -\nabla^2 \cL_N (\btheta)\big) \, 
(\theta - \btheta) \big\rangle +O( \|\theta-\btheta\|_2^3) \nonumber\\
&=O\big( \frac{1}{\sqrt{n}} \|\theta- \btheta\|_2^2 + \|\theta-\btheta\|_2^3 \big),
\label{eq:surrogate-loss-approx}
\end{align}
where the fact that $\norm{\nabla^2 \cL_N (\btheta)-\nabla^2 \cL_1 (\btheta)}_2 
= O\big(n^{-1/2}\big)$ is a consequence of matrix concentration.

We now use a Taylor expansion of $\cL_1(\theta)$ around $\btheta$ to replace 
the infinite sum of high-order derivatives in expression~\eqref{EqnTayApp} 
with $\cL_1(\theta) - \cL_1(\btheta) - \langle \nabla \cL_1 (\btheta) , 
\theta- \btheta\rangle$.  This yields:
\begin{equation}
\begin{aligned}
\tcL(\theta) = \cL_N (\btheta) 
+ \langle \nabla \cL_N (\btheta) , \theta- \btheta\rangle 
+ \cL_1(\theta) - \cL_1(\btheta) 
- \langle \nabla \cL_1 (\btheta) , \theta- \btheta\rangle.
 \label{EqnAppLossTmp}
 \end{aligned}
\end{equation}
Finally, we omit the additive constants in \eqref{EqnAppLossTmp} and redefine 
$\tcL(\theta)$ as follows:
 \begin{align}
 \label{EqnAppLoss}
\tcL(\theta) \defn \cL_1(\theta) - \big\langle \theta, \,\nabla \cL_1(\btheta) - \nabla \cL_N(\btheta) \big\rangle.
 \end{align}
Henceforth we refer to this expression for $\tcL(\theta)$ as a \emph{surrogate 
loss function}.  In the remainder of the section we present three examples 
of using this surrogate loss function for frequentist and Bayesian inference.  
A rigorous justification for $\tcL(\theta)$, which effectively provides conditions 
under which the terms in \eqref{eq:surrogate-loss-approx} are small, follows in 
Section~\ref{sec:main}.

\paragraph{Example ($M$-estimator):}
In the low-dimensional regime where the dimensionality $d$ of parameter space is 
$o(N)$, the global empirical loss minimizer,
\begin{align*}
\htheta:\,=  \argmin_{\theta\in\Theta} \cL_N(\theta),
\end{align*}
achieves a root-$N$ rate of convergence under mild conditions. One may construct 
confidence regions associated with $\htheta$ using the sandwiched covariance matrix 
(see, e.g.,~\eqref{EqnMLEAsymp}). In our distributed inference framework, we aim to
capture some of the desirable properties of $\htheta$ by replacing the global loss 
function $\cL_N(\theta)$ with the surrogate loss function $\tcL$ and defining the
following communication-efficient estimator:
\begin{align*}
\ttheta &= \arg\min_{\theta \in \Theta} \tcL(\theta).
\end{align*}
Indeed, in Section \ref{SectionMestForPara}, we show that $\ttheta$ is equivalent 
to $\htheta$ up to higher-order terms, and we provide two ways to construct confidence 
regions for $\ttheta$ using local observations stored in machine $\cM_1$.

In anticipation of that theoretical result, we give a heuristic argument for why 
$\ttheta$ is a good estimator.  For convenience, we assume that the empirical risk 
function $\cL_N(\theta)$ has a unique minimizer. First consider the global empirical 
loss minimizer $\htheta$. Under our assumption that the loss function is 
twice-differentiable, $\htheta$ is the unique solution of equation\footnote{There is no Taylor's theorem for vector-valued functions, but we formalize this heuristic in Section \ref{SectionMestForPara}.}
\begin{align*}
0= \nabla \cL_N(\htheta) \approx\nabla\cL_N(\thetas)+ \nabla^2 \cL_N(\thetas) \,(\htheta-\thetas).
\end{align*}
By solving this equation, we obtain $\|\htheta- \thetas\|_2 =O_p( \norm{\nabla 
\cL_N(\thetas)}_2) = O_p(N^{-1/2})$, as long as the Hessian matrix $\nabla^2 
\cL_N(\thetas)$ is non-singular. 
Now let is turn to the surrogate loss minimizer $\ttheta$. An analogous argument 
leads to $\|\ttheta- \thetas\|_2 =O_p( \|\nabla \tcL(\thetas)\|_2)$ and we only 
need to show that $\|\nabla \tcL(\thetas)\|_2$ is of order $O_p(N^{-1/2})$. 
In fact, by our construction, 
\begin{align*}
\nabla \tcL (\thetas) & = \big(\nabla \cL_1(\thetas) - \nabla \cL_1(\btheta)\big) - \big(\nabla \cL_N(\thetas) - \nabla \cL_N(\btheta)\big) + \nabla \cL_N(\thetas)
\\
& \approx \langle \nabla^2\cL_1(\thetas) - \nabla^2\cL_N(\thetas),\, \thetas - \btheta \rangle + O_p(N^{-1/2})\\
& = O_p\big(n^{-1/2}\, \|\thetas - \btheta\|_2\big) + O_p(N^{-1/2}),
\end{align*}
which is of order $O_p(N^{-1/2})$ as long as $\|\thetas - \btheta\|_2 = 
O_p\big( k^{-1/2}\big)$ where $k=N/n$ is the number of machines. For example, 
this requirement on initial estimator $\btheta$ is satisfied by the minimizer 
$\htheta_1$ of the subsample loss function $\cL_1(\theta)$ when $n >k$. 

\paragraph{Example (High-dimensional regularized estimator):}

In the high-dimensional regime, where the dimensionality $d$ can be much larger 
than the sample size $N$, we need to impose a low-dimensional structural assumption, 
such as sparsity, on the unknown true parameter $\thetas$. Under such an assumption, 
regularized estimators are known to be effective for estimating $\theta$. For 
concreteness, we focus on the sparsity assumption that most components of the 
$d$-dim vector $\thetas$ is zero, and consider the $\ell_1$-regularized estimator,
$$\htheta:\,=\argmin_{\theta\in\Theta}\big\{\cL_N(\theta) + \lambda \|\theta\|_1\big\}$$,
as the benchmark estimator that we want to approximate, where $\lambda$ is the 
regularization parameter. In the distributed inference framework, we consider the 
following estimator obtained from the surrogate loss function $\tcL$: 
\begin{align*}
\ttheta &= \arg\min_{\theta \in \Theta} \big\{\tcL(\theta)+ \lambda \|\theta\|_1\big\}.
\end{align*}
In Section \ref{SectionMestWithReg}, we show that $\ttheta$ achieves the same 
rate of convergence as the benchmark estimator $\htheta$ under a set of mild 
conditions. This idea of using the surrogate loss function to approximate the 
global regularized loss function is general and is applicable to other high-dimensional 
problems.

\paragraph{Example (Bayesian inference):}
In the Bayesian framework, viewing parameter $\theta$ as random, we place a prior 
distribution $\pi$ over parameter space $\Theta$.  For convenience, we also use 
the notation $\pi(\theta)$ to denote the pdf of the prior distribution at point 
$\theta$.  According to Bayes' theorem, the posterior distribution satisfies
$$\pi(\theta\,|\,Z_1^N) \propto \exp\{-N \cL_N(\theta)\} \, \pi(\theta).$$
The loss function $\cL$ corresponds to the negative log-likelihood function 
for the statistical model $\{\Pp_{\theta}: \theta\in\Theta\}$ and $\cL_N(\theta)$ 
is the global negative log-likelihood associated with the observations $Z_1^N$. 
The posterior distribution $\pi(\theta\,|\,Z_1^N)$ can be used to conduct 
statistical inference. For example, we may construct an estimator of $\theta$ 
as the posterior expectation and use the highest posterior region as a credible 
region for this estimator. Since the additive constant $C$ in 
expression~\eqref{eq:surrogate-loss-approx} can be absorbed into the 
normalizing constant, we may use the surrogate posterior distribution,
$$\apost (\theta\,|\,Z_1^N) \propto \exp \{ -N \tcL(\theta) \} \, \pi(\theta)$$,
to approximate the global posterior distribution $\pi(\theta\,|\,Z_1^N)$. 
In Section \ref{SectionBayes}, we formalize this argument and show that 
this surrogate posterior gives a good approximation the global posterior.

From now on, we will refer the methodology of using the surrogate loss 
function $\tcL(\cdot)$ to approximate the global loss function $\cL(\cdot)$ 
for distributed statistical inference as a Communication-efficient Surrogate 
Likelihood (CSL) method. Although our focus is on distributed inference,
we also wish to note that the idea of computing the global likelihood function 
using subsamples may be useful not only in the distributed inference framework, 
but also in a single-machine setting in which the sample size is so large 
that the evaluation of the likelihood function or its gradient is unduly
expensive. Using our surrogate loss function $\tcL(\theta)$, we only need 
one pass over the entire dataset to construct $\tcL(\theta)$.  After its 
construction, $\tcL(\theta)$ can be efficiently evaluated by using a small 
subset of the data.

\section{Main results and their consequences}
In this section, we delve into the three examples in Section~\ref{SectionDSL} 
of applying the CSL method. For each of the examples, we provide an explicit 
bound on either the estimation error $\|\ttheta-\thetas\|_2$ of the resulting 
estimator $\ttheta$ or the approximation error $\|\apost - \pi_N\|_1$ of the 
approximated posterior $\apost(\cdot)$.

\label{sec:main}
\subsection{Communication-efficient $M$-estimators in low dimensions}
\label{SectionMestForPara}
\label{sec:m-estimator}
In this subsection, we consider a low-dimensional parametric family 
$\{\mathbb{P}_\theta:\theta\in\Theta\}$, where the dimensionality $d$ 
of $\theta$ is much smaller than the sample size $n$. Under this setting, 
the minimizer of the population risk in optimization problem~\eqref{EqnPR} 
is unique under the set of regularity conditions to follow and $\thetas$ 
is identifiable. As a concrete example, we may consider the negative 
log-likelihood function $\ell(\theta;\, z) = -\log p(z;\,\theta)$ as 
the loss function, where $p(\cdot;\, \theta)$ is the pdf for $\Pp_{\theta}$. 
Note that the developments in this subsection can also be extended to 
misspecified families where the marginal distribution $\Pp$ of the observations 
is not contained in the model space $\{\Pp_{\theta}:\,\theta\in\Theta\}$. 
Under misspecification, we can view the parameter $\thetas$ associated 
with the projection $\Pp_{\thetas}$ of the true data generating model 
$\Pp$ onto the misspecified model space, $\{\Pp_{\theta}:\,\theta\in\Theta\}$, 
as the ``true parameter.'' The results under misspecification are similar 
to the well-specified case and are omitted in this paper.

For low-dimensional parametric models, we impose some regularity conditions on the parameter space, the loss function $\cL$ and the associated population risk function $\cLs$. These conditions are standard in classical statistical analysis of $M$-estimators. In the rest of the paper, we call a parametric model that satisfies this set of regularity conditions a regular parametric model.
Our first assumption describes the relationship of the parameter space $\Theta$ and the true parameter $\thetas$.

\paragraph{Assumption PA (Parameter space):}
The parameter space $\Theta$ is a compact and convex subset of $\R^d$. 
Moreover, $\thetas \in \mbox{int}(\Theta)$ and $R :\, = \sup_{\theta\in\Theta}\|\theta 
- \thetas\|_2 > 0$.

The second assumption is a local identifiability condition, ensuring that
$\thetas$ is a local minimum of $\cLs$.

\paragraph{Assumption PB (Local convexity):}
The Hessian matrix $I(\theta) = \nabla^2 \cLs(\theta)$ of the population 
risk function $\cLs(\theta)$ is invertible at $\thetas$: there exist two 
positive constants $(\mum,\, \mup)$, such that $\mum I_d \preceq\nabla^2 \cLs(\thetas) \preceq \mup I_d$.

When the loss function is the negative log-likelihood function, the corresponding 
Hessian matrix is an information matrix. 

Our next assumption is a global identifiability condition, which is a 
standard condition for proving estimation consistency.

\paragraph{Assumption PC (Identifiability):}
For any $\delta>0$, there exists $\epsilon>0$, such that
\begin{align*}
\liminf_{n\to\infty} \Pp\Big\{ \inf_{\|\theta - \thetas\|_2\geq \delta} 
( \cL(\theta) - \cL(\thetas)) \geq \epsilon \Big\} = 1.
\end{align*}

Our final assumption controls moments of higher-order derivatives of 
the loss function, and allows us to obtain high-probability bounds on 
the estimation error. Let $U(\rho) = \{\theta \in \R^d \, | \, \|\theta 
- \thetas\|_2 \leq \rho \} \subset \Theta$ be a ball around the truth 
$\thetas$ with radius $\rho>0$.

\paragraph{Assumption PD (Smoothness):}
There exist constants $(G, L)$ and a function $M(z)$ such that
\begin{align*}
\Ee\big[\|\nabla & \cL(\theta; Z)\|_2^{16} \big] \leq G^{16}, \quad
\Ee\big[\matnorm{\nabla^2 \cL(\theta; Z) - I(\theta)}{2}^{16} \big] \leq L^{16}, \quad \mbox{for all } \theta \in U, \\
&\matnorm{\nabla^2\cL(\theta; z) - \nabla^2 \cL(\theta'; z)}{2}  \leq M(z) \, \|\theta - \theta'\|_2, \quad \mbox{for all } \theta,\theta' \in U.
\end{align*}
Moreover, the function $M(z)$ satisfies $\Ee[ M^{16} (Z)] \leq M^{16}$ for some constant $M>0$.


Based on the heuristic argument in Section~\ref{SectionDSL}, we propose to
use the surrogate function $\tcL$ defined in~\eqref{EqnAppLoss} as the objective 
function for constructing an M-estimator in regular parametric models. Our 
first result shows that under Assumptions PA-PD, given any reasonably good 
initial estimator $\btheta$, any minimizer $\ttheta$ of $\tcL(\theta)$, i.e.,
\begin{align}
\label{EqnMintcL}
\ttheta \in \argmin_{\theta\in\Theta} \tcL(\theta),
\end{align}
significantly boosts the accuracy in terms of the approximation error $\|\ttheta - \htheta\|_2$ to the global empirical risk minimizer $\htheta = \argmin_{\theta \in \Theta} \cL_N(\theta)$.

\begin{theorem}\label{ThmMestNoReg}
Suppose that Assumptions PA-PD hold and the initial estimator $\btheta$ lies 
in the neighborhood $U(\rho)$ of $\thetas$. Then any minimizer $\ttheta$ of 
the surrogate loss function $\tcL(\theta)$ satisfies
\begin{align}
\label{EqnExactMinBound}
&\|\ttheta - \htheta\|_2 
\leq C_2\, \big(\|\btheta - \htheta\|_2 + \|\htheta - \thetas\|_2 
+ \matnorm{\nabla^2 \cL_1(\thetas) - \nabla^2 \cL_N(\thetas)}{2}\big)\, 
\|\btheta - \htheta\|_2,
\end{align}
with probability at least $1 - C_1\, kn^{-8}$, where the constants 
$C_1$ and $C_2$ are independent of $(k, n, N)$.
\end{theorem}
Under the conditions of Theorem~\ref{ThmMestNoReg}, it can be shown that $\|\htheta - \thetas\|_2 = O_p(N^{-1/2})$ and $\matnorm{\nabla^2 \cL_1(\thetas) - \nabla^2 \cL_N(\thetas)}{2} = O_p(n^{-1/2})$ (see Lemma~\ref{LemmaMomentBound} and inequality~\eqref{EqnMLEerr} in Appendix~\ref{AppProofLemSmallProbEvent}), and therefore
\begin{align*}
\|\ttheta - \htheta\|_2  = \big(O_p(n^{-1/2}) + \|\btheta - \htheta\|_2\big)\,\|\btheta - \htheta\|_2 = O_p(n^{-1/2}) \, \|\btheta - \htheta\|_2,
\end{align*}
as long as $\|\btheta - \htheta\|_2 = O_p(n^{-1/2})$, which is true for $\btheta= \htheta_1\defn \argmin_\theta \cL_1(\theta)$, the empirical risk minimizer in local machine $\cM_1$. To formalize this argument, we have the following corollary that provides an $\ell_2$ risk bound for $\ttheta$.

\begin{corollary}
\label{CoroMomentBound}
Under the conditions of Theorem~\ref{ThmMestNoReg}, we have
\begin{align*}
\Ee[\|\ttheta - \thetas\|_2^2] \leq 
\frac{A}{N} + \frac{C}{N\sqrt{N}} + \frac{C}{\sqrt{n\,N}} \, \min\Big\{\frac{1}{\sqrt{n}},\, \big(\Ee[\|\btheta - \htheta\|_2^4]\big)^{1/4}\Big\} + \frac{C}{n^4}\,\sqrt{\frac{k}{N}},
\end{align*}
where $A = \Ee[ \|I(\thetas)^{-1} \, \nabla \cL(\thetas; Z) \|_2^2]$ and $C$ is some constant independent of $(n,k,N)$.
\end{corollary}
Note that the H\'{a}jek-Le Cam minimax theorem guarantees that for any estimator 
$\htheta_N$ based on $N$ samples, we have
\begin{align*}
\lim_{c \to\infty} \liminf_{N\to\infty} \sup_{\theta\in U(c/\sqrt{N})} N\,  \Ee_{\theta} \big[\|\htheta_N - \theta\|_2^2\big] \geq A.
\end{align*}
Therefore, the estimator $\ttheta$ is (first-order) minimax-optimal and achieves the Cram\'{e}r-Rao lower bound when the loss function $\cL$ is the negative log-likelihood function.

\paragraph{One-step approximation:}
The computational complexity of exactly minimizing the surrogate loss $\tcL(\theta)$ 
in \eqref{EqnMintcL} can be further reduced by using a local quadratic approximation 
to $\cL$. In fact, we have by Taylor's theorem that
\begin{align*}
\cL_N (\theta) \approx \cL_N (\btheta)  +\langle \nabla \cL_N (\btheta), 
\theta- \btheta \rangle + \frac12 \langle \theta - \btheta , \nabla^2 \cL_N 
( \btheta-\theta) \rangle .
\end{align*}
As before, we replace the global gradient $\nabla \cL_N (\btheta)$ with the 
local gradient $\nabla \cL_1 (\btheta)$, which leads to the following quadratic 
surrogate loss:
\begin{align*}
\tcL^H (\theta) := \langle \nabla \cL_N (\btheta), \theta- \btheta \rangle 
+ \frac12 \langle \theta - \btheta , \nabla^2 \cL_1 ( \btheta-\theta) \rangle.
\end{align*}
Because the surrogate loss functions $\tcL^H$ and $\tcL$ agree up to the 
second-order Taylor expansion, they behave similarly when used as objective 
functions for constructing $M$-estimators. This motivates the closed-form
estimator
\begin{align*}
\ttheta^H \defn  \arg\min_{\theta \in \Theta} \tcL^H (\theta)=\, \btheta 
- \nabla^2 \cL_1(\btheta)^{-1}\, \nabla \cL_N(\btheta).
\end{align*}
The next theorem shows that $\ttheta^H$ satisfies the same estimation bound 
as $\ttheta$. Unlike the classical one-step MLE that requires the initial 
estimator to be within an $\mathcal{O}(N^{-1/2})$ neighborhood of the truth 
$\thetas$, we only require $\|\btheta-\thetas\|_2$ to be $\mathcal{O}(n^{-1/2})$.
\begin{theorem}
\label{ThmNewtonerr}
Suppose that Assumptions PA-PD hold and the initial estimator $\btheta$ satisfies $\|\btheta - \thetas\|_2 \leq \min\big\{\rho,\, (16M)^{-1}(1 - \rho)\,\mum\big\}$. Then the local one-step estimator $\ttheta^H$ satisfies
\begin{align*}
&\|\ttheta^H - \htheta\|_2 
\leq C_2'\, \big(\|\btheta - \htheta\|_2 + \|\htheta - \thetas\|_2 + \matnorm{\nabla^2 \cL_1(\thetas) - \nabla^2 \cL_N(\thetas)}{2} \big)\, \|\btheta - \htheta\|_2,
\end{align*}
with probability at least $1 - C_1'\, kn^{-8}$, where  $C_1'$ and $C_2'$ are 
independent of $(k, n, N)$.
\end{theorem}
\noindent
The analogue of Corollary \ref{CoroMomentBound} can also be stated for $\ttheta^H$. 

\paragraph{Iterative local estimation algorithm:}
Theorem~\ref{ThmMestNoReg} (Theorem~\ref{ThmNewtonerr}) suggests that an iterative 
algorithm may reduce the approximation error $\|\ttheta -\htheta\|_2$ by 
a factor of $n^{-1/2}$ in each iteration as long as the initial estimator 
satisfies $\|\btheta - \htheta\|_2 = O_p(n^{-1/2})$, or equivalently, 
$\|\btheta - \thetas\|_2 = O_p(n^{-1/2})$.  We refer to such an algorithm
as an Iterative Local Estimation Algorithm (ILEA, see Algorithm~\ref{AlgoIterative}).
In each iteration of ILEA, we set $\btheta$ as the current iterate $\theta^{(t)}$, 
construct the surrogate loss function $\tcL^{(t)}(\theta)$, and then solve for
the next iterate $\theta^{(t+1)}$ by either exactly minimizing the surrogate loss:
\begin{align*}
\theta^{(t+1)} \in \argmin_{\theta\in\Theta} \tcL^{(t)}(\theta),
\end{align*}
or by forming a local one-step quadratic approximation:
\begin{align*}
\theta^{(t+1)} = \theta^{(t)} - \nabla^2\cL_1(\theta^{(t)})^{-1}\nabla 
\cL_N(\theta^{(t)})=\argmin_{\theta \in \Theta} \tcL^{H,\,(t)} (\theta).
\end{align*}
Theorem~\ref{ThmMestNoReg} (or Theorem~\ref{ThmNewtonerr}) guarantees, 
with high probability, the error bound
\begin{align*}
\|\theta^{(t+1)} - \htheta\|_2 \leq \frac{C_3}{\sqrt{n}} 
\, \|\theta^{(t)} - \htheta\|_2,  \quad \mbox{for each $t\geq 0$},
\end{align*}
where $C_3$ is positive constant independent of $(n,k,N)$. If the desired 
accuracy is the statistical accuracy $\|\htheta-\thetas\|_2$ of the MLE 
and our initial estimator is $n^{-1/2}$-consistent, then we need to conduct 
at most $\lceil \frac{\log k}{\log n}\rceil$ iterations.  ILEA interpolates 
between the gradient method and Newton's algorithm. When $n$ is large 
relative to $k$, then ILEA behaves like Newton's algorithm, and we achieve 
the optimal statistical accuracy in one iteration. If $n$ is a fixed constant 
size, then ILEA reduces to a preconditioned gradient method. By appropriately 
choosing the subsample size $n$, ILEA achieves a trade-off among storage, 
communication, and computational complexities, depending on specific 
constraints of computing resources.

\begin{algorithm}[th]
Initialize $\theta^{(0)} = \btheta$\;
\For{$t = 0,1,\ldots, T -1 $}
{
Transmit the current iterate $\theta^{(t)}$ to local machines $\{\cM_j\}_{j=1}^k$\;
\For{$j = 1:k$}
{Compute the local gradient $\nabla \cL_j(\theta^{(t)})$ at machine $\cM_j$\;
Transmit the local gradient $\nabla \cL_j(\theta^{(t)})$ to machine $\cM_1$\;
}
Calculate the global gradient $\nabla \cL_N(\theta^{(t)}) = \frac 1k\sum_{j=1}^k \nabla\cL_j(\theta^{(t)})$ in Machine $\cM_1$\;
Form the surrogate function $\tcL^t(\theta) = \cL_1(\theta) - \langle \theta,\, \nabla \cL_1(\theta^{(t)}) - \nabla \cL_N(\theta^{(t)}) \rangle$\;
Do one of the following in Machine $\cM_1$:\\
(1).\, Update $\theta^{(t+1)} \in \argmin_{\theta\in\Theta} \tcL^t(\theta)$ \tcp*[l]{Exactly minimizing surrogate function $\tcL$}
(2).\, Update $\theta^{(t+1)} = \theta^{(t)} - \nabla^2\cL_1(\theta^{(t)})^{-1}\nabla \cL_N(\theta^{(t)})$ \tcp*[l]{One-step quadratic approximation}
}
\Return{$\theta^{(T)}$}
\caption{Iterative local estimation}\label{AlgoIterative}  \end{algorithm}

\paragraph{Confidence region construction:}
We now consider a natural class of local statistical inference procedures 
based on the surrogate function $\tcL(\theta)$ that only uses the subsample 
$\{z_{i1}\}_{ i=1}^n$ in Machine $\cM_1$.  It is a classical result that under 
Assumptions PA-PD, the global empirical risk minimizer $\htheta$ satisfies 
(see the proof of Corollary~\ref{CoroAsympExpan} in 
Section~\ref{SectionProofCoroAsympExpan})
\begin{align}
\label{EqnMLEAsymp}
\htheta - \thetas &= - I(\thetas)^{-1}\nabla \cL_N(\thetas) + O_p(N^{-1}),
\qquad \mbox{and}\\
\sqrt{N}\, (\htheta - \thetas)& \rightarrow \mathcal{N}\big(0,\, \Sigma \big) 
\quad \mbox{in distribution as $N\to\infty$}, \notag
\end{align}
where $\Sigma\defn I(\thetas)^{-1} \Ee[\nabla \cL(\thetas;\, Z) \nabla 
\cL(\thetas;\, Z)^T] I(\thetas)^{-1}$ is the so-called sandwich covariance 
matrix. For example, when $\cL$ corresponds to the negative log-likelihood 
function, $\Sigma = I(\thetas)^{-1}$ will be the inverse of the information matrix.
It is easy to see that the plug-in estimator,
\begin{align}
\label{EqnhSigma}
\hSigma \defn \nabla^2\cL_N(\htheta)^{-1} \Big(\frac{1}{N} \sum_{i=1}^n\sum_{j=1}^k 
\nabla \cL(\htheta;\, z_{ij})\, \nabla \cL(\htheta;\, z_{ij})^T\Big)
\nabla^2\cL_N(\htheta)^{-1},
\end{align}
is a consistent estimator of the asymptotic covariance matrix $\Sigma$; that 
is, $\hSigma \to \Sigma$ in probability as $N\to\infty$. Based on the limiting 
distribution of $\sqrt{N}\, (\htheta-\thetas)$ and the plug-in estimator 
$\hSigma$, we can conduct statistical inference, for example, constructing 
confidence intervals for $\theta^*$.

The following corollary shows that for any reasonably good initial estimator 
$\btheta$, the asymptotic distribution of either the minimizer $\ttheta$ of 
the surrogate function $\tcL(\theta)$ or the local one-step quadratic 
approximated estimator $\ttheta$ matches that of the global empirical risk 
minimizer $\htheta$. Moreover, we also have a consistent estimator $\tSigma$ 
of $\Sigma$ using only the local information in Machine $M_1$. Therefore, 
we can conduct statistical inference locally without access to the entire 
data while achieving the same asymptotic inferential accuracy as global 
statistical inference procedures.

\begin{corollary}
\label{CoroAsympExpan}
Under the same set of assumptions in Theorem~\ref{ThmMestNoReg}, if the 
initial estimator $\btheta$ satisfies $\|\btheta - \thetas\|_2 = O_p(n^{-1/2})$, 
then the surrogate minimizer $\ttheta$ satisfies
\begin{align*}
\ttheta - \thetas &= - I(\thetas)^{-1}\nabla \cL_N(\thetas) + O_p(N^{-1} + n^{-1/2}\, \|\btheta - \thetas\|_2),
\end{align*}
and if $\|\btheta - \thetas\|_2 =o_P\big(\sqrt{\frac{n}{N}}\big)$, then
\begin{align*}
\sqrt{N}\, (\ttheta - \thetas)& \rightarrow \mathcal{N}\big(0,\, \Sigma \big) \quad \mbox{in distribution as $N\to\infty$}.
\end{align*}
Moreover, the following plug-in estimator:
\begin{align}
\label{EqntSigma}
\tSigma&\defn \nabla^2\tcL(\ttheta)^{-1} \Big(\frac{1}{n}\,\sum_{i=1}^n 
\nabla \tcL(\ttheta;\, z_{i1})\, \nabla \tcL(\ttheta;\, z_{i1})^T\Big)
\nabla^2\tcL(\ttheta)^{-1},
\end{align}
is a consistent estimator for $\Sigma$ as $n\to\infty$. If we also have 
$k\to \infty$, then the plug-in estimator
\begin{align}
\label{EqntSigma'}
\tSigma'&\defn \nabla^2\tcL(\ttheta)^{-1} \Big(\frac{n}{k}\,\sum_{j=1}^k 
\nabla \cL_j(\ttheta)\, \nabla \cL_j(\ttheta)^T\Big)\nabla^2\tcL(\ttheta)^{-1}
\end{align}
is also a consistent estimator for $\Sigma$ as $(n,k)\to\infty$.
Similar results hold for the local one-step quadratic approximated estimator 
$\ttheta^H$ under the assumptions of Theorem~\ref{ThmNewtonerr}.
\end{corollary}

Corollary~\ref{CoroAsympExpan} illustrates that we may substitute $\tcL(\theta)$ 
as the global loss function and use it for statistical inference---$\tSigma$ 
is precisely the plug-in estimator of the sandwiched covariance matrix using 
the surrogate loss function $\tcL(\theta)$ (cf.~equation~\eqref{EqnhSigma}). 
In the special case when $\cL(\theta)$ is the negative log-likelihood function, 
we may instead use $\nabla^2\tcL(\ttheta)^{-1}$ as our plug-in estimator for 
$\Sigma = I(\thetas)^{-1}=\Ee[\nabla^2\cL(\thetas)^{-1}]$.  $\tSigma'$ tends 
to be a better estimator than $\tSigma$ when $k\gg n$, since the variance 
$\mathcal{O}(k^{-1})$ of the middle term in equation~\eqref{EqntSigma'} is 
much smaller than variance $\mathcal{O}(n^{-1})$ of the middle term in 
equation~\eqref{EqntSigma}. See Section~\ref{sec:simulowdim} for an 
empirical comparison of using $\hSigma$ and $\hSigma'$ for constructing 
confidence intervals.


\subsection{Communication-efficient regularized estimators with $\ell_1$-regularizer}
\label{SectionMestWithReg}
In this subsection, we consider high-dimensional estimation problems where 
the dimensionality $d$ of parameter $\theta$ can be much larger than the 
sample size $n$. Although the development here applies to a broader class 
of problems, we focus on $\ell_1$-regularized procedures. $\ell_1$-regularized 
estimators work well under the sparsity assumption that most components of 
the true parameter $\thetas$ is zero. Let $ S= \text{supp} (\thetas)$ be 
a subset of $\{1,\ldots,d\}$ that encodes the sparsity pattern of $\thetas$   
and let $s= |S| = \sum_{j=1}^d\mathbb{I}(\thetas_j\neq 0)$.  Using the 
surrogate loss function $\tcL(\theta)$ as a proxy to the global likelihood 
function in $\ell_1$-regularized estimation procedures, we obtain the 
following communication-efficient regularized estimator:
$$\ttheta \in \argmin_{\theta\in\Theta}\big\{\tcL(\theta) 
+\lambda \norm{\theta}_1\big\}.$$ 
We study the statistical precision of this estimator in the high-dimensional regime.

We first present a theorem on the statistical error bound $\|\ttheta-\thetas\|_2$ 
of the estimator $\ttheta$ for general loss function $\cL$.  We then illustrate
the use of the theorem in the settings of high-dimensional linear models and 
generalized linear models. We begin by stating our assumptions.

\paragraph{Assumption HA (Restricted strongly convexity):}
The local loss function  $\cL_1(\theta)$ at machine $\cL_1$ is restricted 
strongly convex over $S$: for all $\delta \in C(S):\,=\{v: \norm{v_S}_1 \le 3 
\norm{v_{S^c}}_1\}$,
$$
\cL_1(\thetas +\delta) - \cL_1(\thetas) - \nabla \cL_1 (\thetas)^T \delta 
\ge \mu \norm{\delta}_2 ^2,
$$
where $\delta$ is some positive constant independent of $n$.

As the name suggested, restricted strongly convexity requires the global loss 
function $\cL_n(\theta)$ to be a strongly convex function when restricted 
to the cone $C(S)$.

\paragraph{Assumption HB (Restricted Lipschitz Hessian):}
Both the local and global loss function $\cL_1(\theta)$ and $\cL_N(\theta)$ have restricted Lipschitz Hessian at radius $R$: for all $\delta \in C(S) \cap B_R ( \thetas)$,
\begin{align*}
\norm{(\nabla^2 \cL_1(\thetas+ \delta) -\nabla^2 \cL_1(\thetas)) \, \delta}_\infty &\le M \norm{\delta}_2 ^2,\qquad\mbox{and}\\
\norm{(\nabla^2 \cL_N(\thetas+ \delta) -\nabla^2 \cL_N(\thetas)) \, \delta}_\infty &\le M \norm{\delta}_2 ^2,
\end{align*}
where $M$ is some positive constant independent of $N$.

The restricted Lipschitz Hessian condition is always satisfied for linear 
models where the Hessian $\nabla^2 \cL_N(\theta)$ is a constant function of $\theta$. 

\begin{theorem}
Suppose that Assumption HA and Assumption HB at radius $R> \norm{\btheta-\thetas}_2$ 
are true. If regularization parameter $\lambda$ satisfies $\lambda \geq 2
\norm{\nabla \cL_N (\thetas)}_\infty + 2\norm{\nabla^2\cL_N 
(\thetas) - \nabla^2\cL_1(\thetas)}_\infty \|\btheta- \thetas\|_1 
+ 4M \|\btheta-\thetas\|^2 _2$, then
\[
\|\ttheta - \thetas\|_2 \le \frac{3\sqrt{s}\, \lambda}{\sqrt{\mu}}.
\]
\label{thm:l1-m-estimator}
\end{theorem}

The lower bound condition on the regularization parameter $\lambda$ for 
$\ttheta$ is slightly stronger than that for the estimator $\htheta$ based 
on the global loss function, which is $\lambda \geq 2\norm{\nabla \cL_N 
(\thetas)}_\infty$.  Since the estimation error upper bound provided by 
Theorem~\ref{thm:l1-m-estimator} is proportional to the regularization 
parameter, it is reasonable to expect that $\ttheta$ will yield a slightly 
larger error than $\htheta$, depending on how good the initial estimator 
$\btheta$ is.  For example, in generalized linear models, if small values
of the regularization parameters are chosen for $\ttheta$ and $\htheta$, 
then the estimation error of $\ttheta$ will be greater than that of 
$\htheta$ by an amount
\begin{align*}
 \frac{6\sqrt{s}}{\sqrt{\mu}}&\,\big(\norm{\nabla^2\cL_N (\thetas) - \nabla^2\cL_1(\thetas)}_\infty \|\btheta- \thetas\|_1 + 2M \|\btheta-\thetas\|_2^2\big)\\
 &\sim \sqrt{\frac{s \log d}{n}}\, \|\btheta- \thetas\|_1 + M\,\sqrt{s}\, \|\btheta- \thetas\|_2^2.
\end{align*}
As long as $\|\btheta - \thetas\|_1$ and $\|\btheta-\thetas\|_2$ are 
sufficiently small, this difference will be negligible with respect 
to the estimation error bound of $\htheta$, which is $ \sqrt{\frac{s\log d}{N}}$. 
For example, we may choose $\btheta$ to be the local $\ell_1$ regularized 
estimator $\htheta_1:\,=\argmin_{\theta}\big\{\cL_1(\theta) + \lambda_1\,
\|\theta\|\big\}$ with estimation error $\sqrt{\frac{s\log d}{n}}$, so that
\begin{align*}
\|\htheta_1 - \thetas\|_1 \leq Cs\,\sqrt{\frac{\log d}{n}}\quad\mbox{and}\quad
\|\htheta_1 - \thetas\|_2 \leq C\sqrt{\frac{s\,\log d}{n}}.
\end{align*}
We may also consider an iterative estimation procedure analogous to 
Algorithm~\ref{AlgoIterative} in order to provide higher-order estimation 
accuracy for the communication-efficient regularized estimator $\ttheta$.  
The convergence rate can be analyzed by inducting on Theorem~\ref{thm:l1-m-estimator}. 
We now apply Theorem~\ref{thm:l1-m-estimator} to two examples.

\paragraph{Example (Sparse linear regression):}
In sparse linear regression, observations $\{z_{ij}=(x_{ij},y_{ij}):\, 1\leq i \leq n,\, 1\leq j\leq k\}$ satisfy
\begin{align*}
y_{ij}= x_{ij}^T\beta + \epsilon,\qquad\epsilon\sim\mathcal{N}(0,\, \sigma^2),
\end{align*}
where $x_{ij}$ is a $d$-dimensional covariate vector, $y_{ij}$ is the response 
and $\beta\in\mathbb{R}^d$ is the unknown regression coefficient to be estimated. 
Recall the sparsity assumption that $s=\sum_{j=1}^d\mathbb{I}(\thetas_j\neq 0) =o(n)$. 
For linear regression, the global loss function takes the form
\begin{align*}
\cL_N(\theta) = \frac{1}{N} \sum_{i=1}^n\sum_{j=1}^k (y_{ij} - x_{ij}^T \beta)^2.
\end{align*}
We consider a random design where $x_{ij}$ is i.i.d.\! $A$-sub-Gaussian; 
that is, for all $\alpha\in\mathbb{R}^d$, 
\begin{align*}
\Ee[\exp(\alpha^Tx_{ij})] \leq \exp\big(A^2\,\|\alpha\|_2^2/2\big).
\end{align*}
Let $\Sigma = \Ee[x_{ij}x_{ij}^T]$ be the covariance matrix of the design.
For this class of design, it is known that Assumption HA is satisfied with 
high probability as long as $\Sigma$ is strictly positive definite and 
$n\geq C_0 s\log d$ for some constant $C_0>0$ depending on the minimal 
eigenvalue of $\Sigma$ \citep{Raskutti2010}.  For linear models, the 
Lipschitz constant $M$ in Assumption HB is zero and therefore HB is also 
satisfied.

\begin{theorem}
If $x_{ij} $ is $A$-sub-Gaussian, $\Sigma$ is strictly positive definite 
and $n\geq C_0s\log d$, then with probability at least $1-c_1\exp\{-c_2n\}$, 
it holds that
\[
\| \ttheta -\thetas \|_2 \le   C_1A\sqrt{\frac{s \log d}{N}} + 
C_1A\sqrt{\frac{s \log d}{n}}\, \|\btheta-\thetas\|_1.
\]
If the initial estimator satisfies $\|\btheta-\thetas\|_1 \leq C_2\, 
s \sqrt{\frac{\log d}{n}}$, then with the same probability, it holds that
\[
\| \ttheta -\thetas\|_2\sim C_1A\sqrt{ \frac{s \log d }{N}} + C_3 \frac{s^{3/2} \log d}{n}
\]
The constants $(c_1,c_2,C_0,C_1,C_2,C_3)$ are independent of $(n,k,d,s)$.
\label{thm:sparse-lr}
\end{theorem}

For linear regression under the sparsity condition, the minimax rate of 
estimating $\theta$ is $\sqrt{\frac{s\log \frac{d}{s}}{N}}$. Therefore, 
Theorem~\ref{thm:sparse-lr} shows that our approximated estimator 
$\ttheta$ is nearly minimax-optimal if $n\geq C s\sqrt{N\log d}$ 
for some constant $C>0$. When this lower bound on the local sample 
size $n$ fails, we may still apply the iterative estimation procedure 
(Algorithm~\ref{AlgoIterative}) to boost the estimation accuracy and 
obtain a minimax-optimal estimator as we remarked after 
Theorem~\ref{thm:l1-m-estimator}.


\paragraph{Example (Generalized linear models):}
In this section, we apply Theorem \ref{thm:l1-m-estimator} to generalized 
linear models with a $\ell_1$-regularizer. We begin with some background 
on generalized linear models. Recall that the data is $z_{ij} = (x_{ij},y_{ij})$, 
where $y_{ij}$ is the response and $x_{ij}$ is the $d$-dim covariate vector. 
A generalized linear model assumes the conditional distribution of $y_{ij}$ 
given $x_{ij}$ to be
\begin{align*}
\Pp(y_{ij}\,|\, x_{ij},\,\theta,\,\sigma) \propto \exp\Big\{\frac{y_{ij}x_{ij}^T\theta - \phi(x_{ij}^T\theta)}{c(\sigma)}\Big\},
\end{align*}
where $\sigma$ is a scalar parameter, $\theta$ is the unknown $d$-dim parameter 
to be estimated and $\phi$ is a link function.  For example, $\phi(x) = 
\log ( 1+e^x)$ in logistic regression, and $\phi(x) = e^x$ in Poisson 
regression. We still assume sparsity that $s=\sum_{j=1}^d\mathbb{I}(\thetas_j\neq 0) 
=o(n)$.  Now the global loss function and its gradient are given by
\begin{align*}
\cL_N (\theta) &= \frac1N \sum_{j=1}^k\sum_{i=1}^n -y_{ij} x_{ij}^T \theta + \phi ({x_{ij}^T \theta}), \qquad\mbox{and}\\
\nabla \cL_N (\theta) &= \frac1N \sum_{j=1}^k \sum_{i=1} ^n (\phi'(x_{ij}^T \theta) - y_{ij}) \,x_{ij}.
\end{align*}

Under a random design assumption, we verify Assumptions HA and HB, and obtain 
the following result.

\begin{theorem}
Assume that for some constants $(A,B,m,L)$, $x_{ij}$ is i.i.d.\! $A$-sub-Gaussian, 
$\norm{x_{ij}}_\infty \le B$, and $m I \preceq \Sigma=\Ee[x_{ij}x_{ij}^T] \preceq L I$. 
Then with probability at least $1-c_1\exp\{-c_2n\}$, it holds that
\begin{align*}
\|\ttheta - \thetas\|_2 \le C_1A\sqrt{\frac{s \log d}{N}} + C_1A\sqrt{\frac{s\log d}{n}} \|\btheta- \thetas\|_1   +C_1A \sqrt{s}\|\btheta-\thetas\|^2 _2.
\end{align*}
If $\|\bar \theta - \thetas\|_1 \leq C_2\, s \sqrt{\frac{\log d}{n}}$ 
and $\|\bar \theta -\thetas\|_2 \leq C_2\,\sqrt{\frac{s \log d}{n}}$, 
then with the same probability, we have
\[
\|\ttheta - \thetas\|_2 \leq  C_3\sqrt{\frac{s \log d}{N}} + C_3\frac{s^{3/2} \log d}{n}.
\]
The constants $(c_1,c_2,C_0,C_1,C_2,C_3)$ are independent of $(n,k,d,s)$.
\label{thm:l1-logistic}
\end{theorem}


\subsection{Communication-efficient Bayesian inference}
\label{SectionBayes}
In this subsection, we consider distributed Bayesian in the setting of 
regular parametric models.  We place a prior distribution $\pi$ on the 
parameter space $\Theta$ and form the global posterior distribution
\begin{align}
\label{EqnPost}
\pi(\theta\, |\, Z_{1}^N) =  D\, \exp\Big\{-\sum_{i=1}^n\sum_{j=1}^k \cL(\theta;\, z_{ij})\Big\}\,\pi(\theta),
\end{align}
where $D$ is the normalizing constant. In the rest of this subsection, 
we tacitly assume that the loss function $\cL$ is the negative log-likelihood 
function. Extensions to the Gibbs posterior~\citep{Walker2013} where 
$\cL$ is replaced with a generic loss function $\cL$ in posterior~\eqref{EqnPost} 
is straightforward.

Most existing literature \citep{Wang15,Xing15} in distributed Bayesian 
inference utilizes the decomposition
\begin{align}
\label{EqnProduct}
\pi(\theta\, |\, Z_{1}^N) =  D\, \prod_{j=1}^k \exp\big\{-n\cL_j(\theta)\big\},
\end{align}
such that the global posterior $\pi(\theta\, |\, Z_{1}^N)$ can be written 
as the product of subsample posteriors 
\begin{align*}
\pi(\theta\, |\, Z_j)=  D_j\, \exp\big\{-n\cL_j(\theta)\big\}\,\pi^{1/k}(\theta),\quad j = 1,\ldots,k,
\end{align*}
where the prior is raised to power $k^{-1}$ so that it  is appropriately 
weighted in product~\eqref{EqnProduct} and $D_j$ is the normalizing constant. 
This decomposition motivates a MapReduce computational framework in which 
separate Markov chains are run in machines $\{\cM_j\}_{j=1}^k$ based on 
the local data on that machine.  After running these Markov chains in
parallel, all local posterior draws are transmitted to a central node,
where  an approximation $\apost(\theta)$ to the global posterior 
$\pi_N (\theta)\defn\pi(\theta\,|\, Z_1^N)$ is formed.  A main drawback 
of these approaches is that the communication cost can be exorbitant---for 
example, exponentially large in the dimension $d$---since the number of 
draws from each local posterior must be large enough to be representative 
of the local posterior distribution.

Our approach to distributed Bayesian inference is based on using the 
surrogate function $\tcL(\theta)$. Our sampling scheme is communication 
efficient, requiring running one single Markov chain in a local machine.
Here is an outline of the algorithm:
\begin{enumerate}
\item Compute a good initial estimate $\btheta$, e.g. the one-step estimate $\ttheta^H$ in Section~\ref{SectionMestForPara}.
\item For $j=1,\ldots,k$, compute the local gradient $\nabla \cL_j(\btheta)$ in machine $\cM_j$.
\item Transmit all local gradients to Machine $\cM_1$ and form the global gradient $\nabla \cL_N(\btheta) = \frac 1k\,\sum_{j=1}^k\nabla \cL_j(\btheta)$.
\item Machine $\cM_1$ constructs the surrogate function $\tcL(\theta)$ as \eqref{EqnAppLoss}.
\item Machine $\cM_1$ runs a Markov chain to sample from the surrogate posterior $\apost(\theta) \propto \exp\big( - N \widetilde \cL(\theta)\big) \, \pi(\theta)$, and uses the draws to conduct statistical inference.
\end{enumerate}

The following result shows that the surrogate posterior $\apost(\cdot)$ is close to the global posterior $\pi(\cdot\,|\,Z_1^N)$ as long as the initial estimator $\btheta$ is reasonably close to $\thetas$.

\begin{theorem}
\label{ThmAPOST}
If Assumption PA-PD hold and $\|\btheta - \htheta\|_2 = o_p(N^{-1/2})$, then the approximate posterior $\apost(\theta)$ satisfies
\begin{align*}
\big\|\apost - \pi_N \big\|_1 = O_p\Big(\sqrt{N}\,\log N\, \|\btheta - \htheta\|_2 + \frac{(\log N)^2}{\sqrt{n}}\Big),
\end{align*}
where $\|P-Q\|_1 = \int |P(d\theta) - Q(d\theta)|$ is the variation distance 
between the distributions $P$ and $Q$.
\end{theorem}

If we use the local one-step estimator $\ttheta^H$ as the initial estimator 
$\btheta$, then the approximation error becomes
\begin{align*}
\big\|\apost - \pi_N \big\|_1 = O_p\Big( \frac{\sqrt{N}\log N}{n}\Big) + \Big(\frac{(\log N)^2}{\sqrt{n}}\Big).
\end{align*}
This illustrates that we may choose $k = N/n$ up to $o\big(N^{1/2} (\log N)^{-1}\big)$ 
while still maintaining $\big\|\apost - \pi_N \big\|_1 = o_p(1)$. The overall 
communication requirements of this procedure are two passes over the entire 
dataset (one for computing $\ttheta^H$ and one for constructing $\tcL(\theta)$). 
To allow larger $k$, we may apply the iterative algorithm in 
Section~\ref{sec:m-estimator} to improve the accuracy of the initial 
estimator $\btheta$.  Note that our theory only covers low-dimensional
regular parameter models; it is still an open problem to design theoretically-sound
communication-efficient Bayesian procedures for high-dimensional problems.


\section{Simulations}
\label{sec:simulations}
In this section, we present examples of simulation experiments using the CSL 
methodology developed in Section~\ref{SectionDSL}.

\subsection{Distributed $M$-estimation in logistic regression}
\label{sec:simulowdim}
In logistic regression, i.i.d.~observations $Z_1^N = \{Z_{ij} = (X_{ij}, Y_{ij}): \, i =1,\ldots,n. \, j=1, \ldots,k\}$ are generated from the model
\begin{align}
\label{EqnLogsitic}
Y_{ij} &\sim \mbox{Ber}(P_{ij}),\quad\mbox{with}\quad
\log\frac{P_{ij}}{1 - P_{ij}} =\langle X_{ij}, \thetas\rangle.
\end{align}
In our simulation, the true regression coefficient $\thetas$ is a $d$-dim 
vector with $d\in\{2, 10, 50\}$ and the $d$-dim covariate vector $X_{ij}$ 
is independently generated from $\mathcal{N}(0, I_d)$. For each replicate 
of the simulation, we uniformly sample the parameter $\thetas$ from the 
$d$-dim unit cube $[0,1]^d$. 

\label{sec:lowdim}
\begin{figure}[htp]
\begin{center}
\begin{tabular}{ccc}
\widgraph{.45\textwidth}{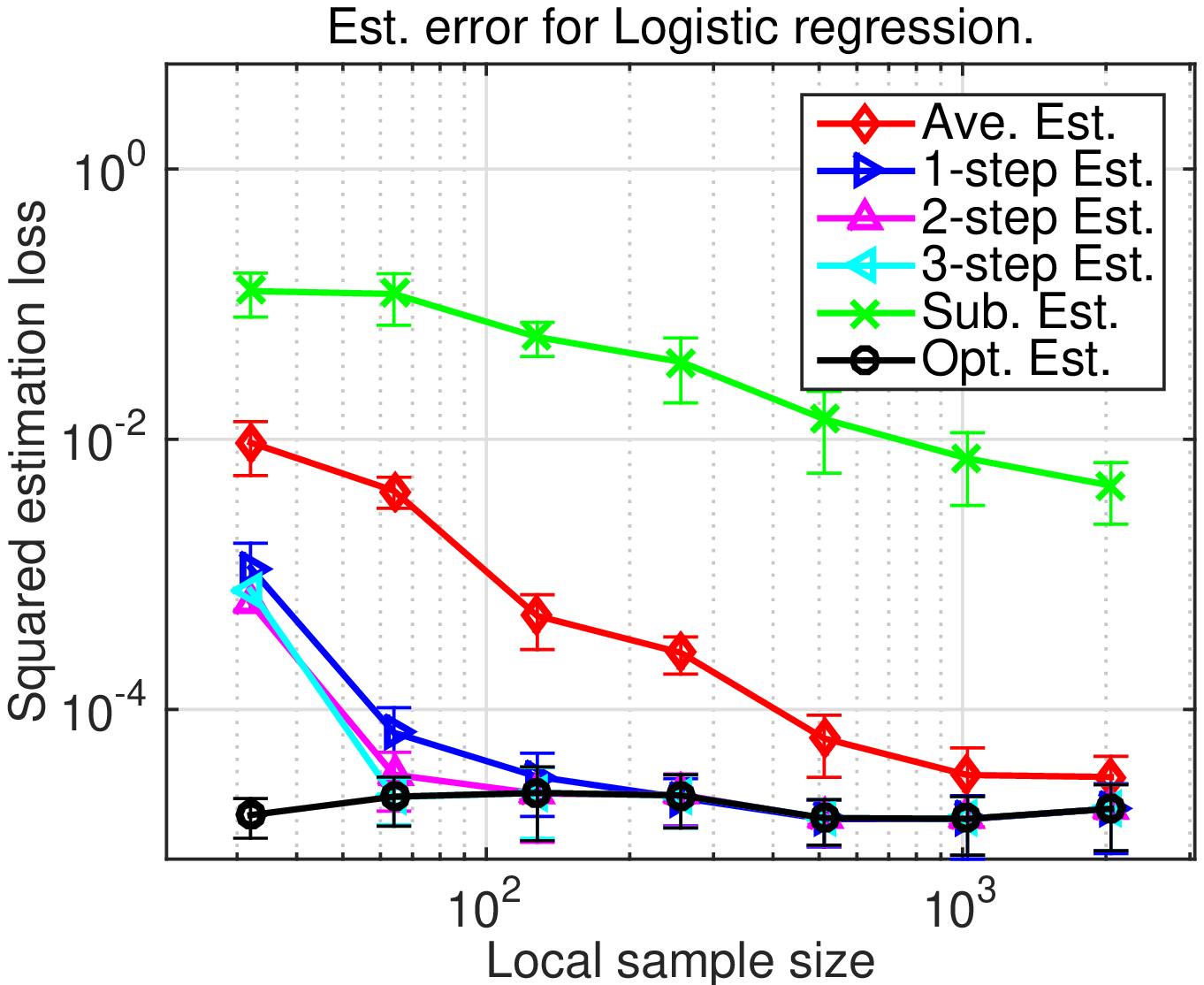} & &
\widgraph{.45\textwidth}{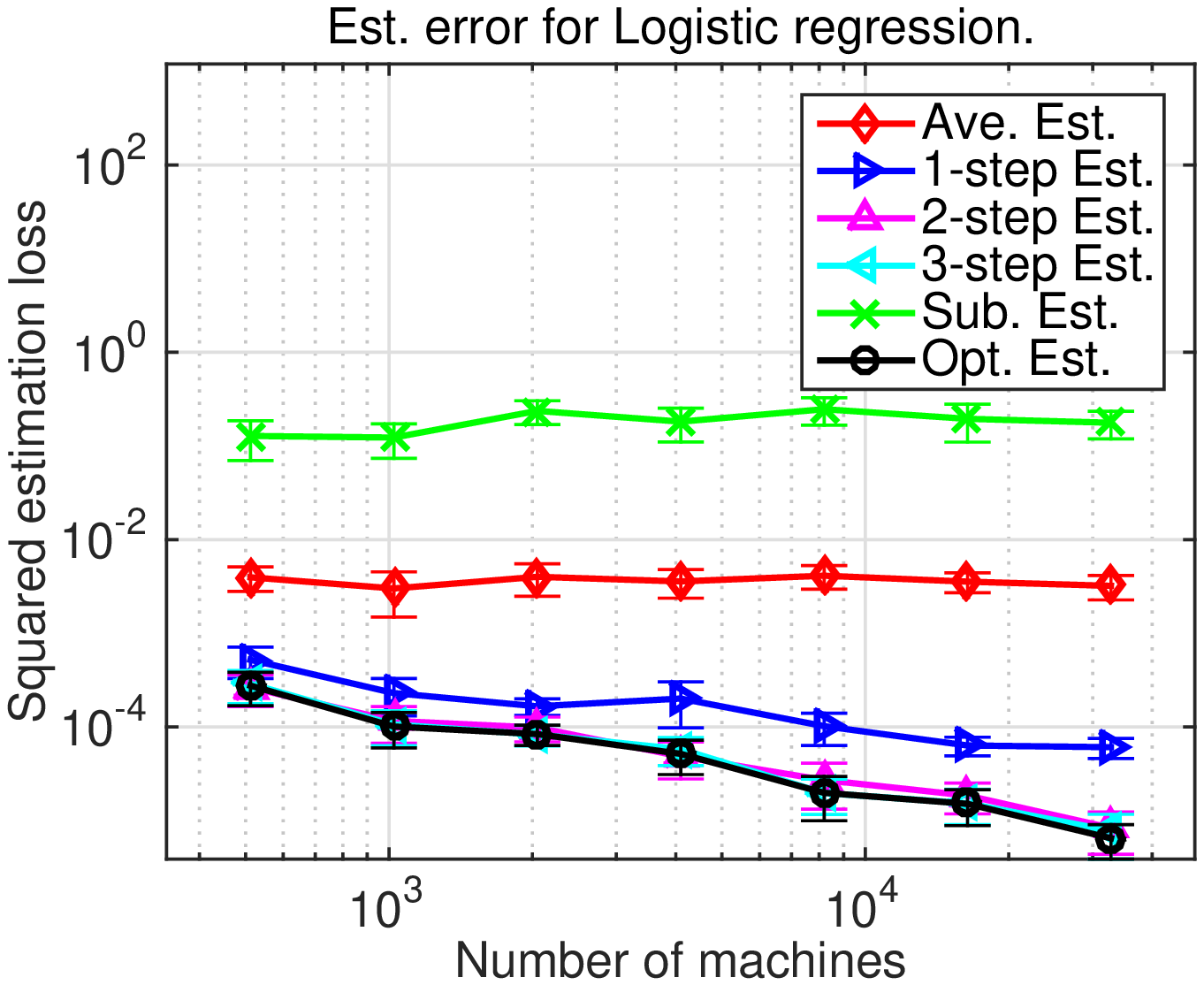} \\
(a) $d=2$ and $N = 524288$. & & (b) $d=2$ and $n = 64$.\\
\widgraph{.45\textwidth}{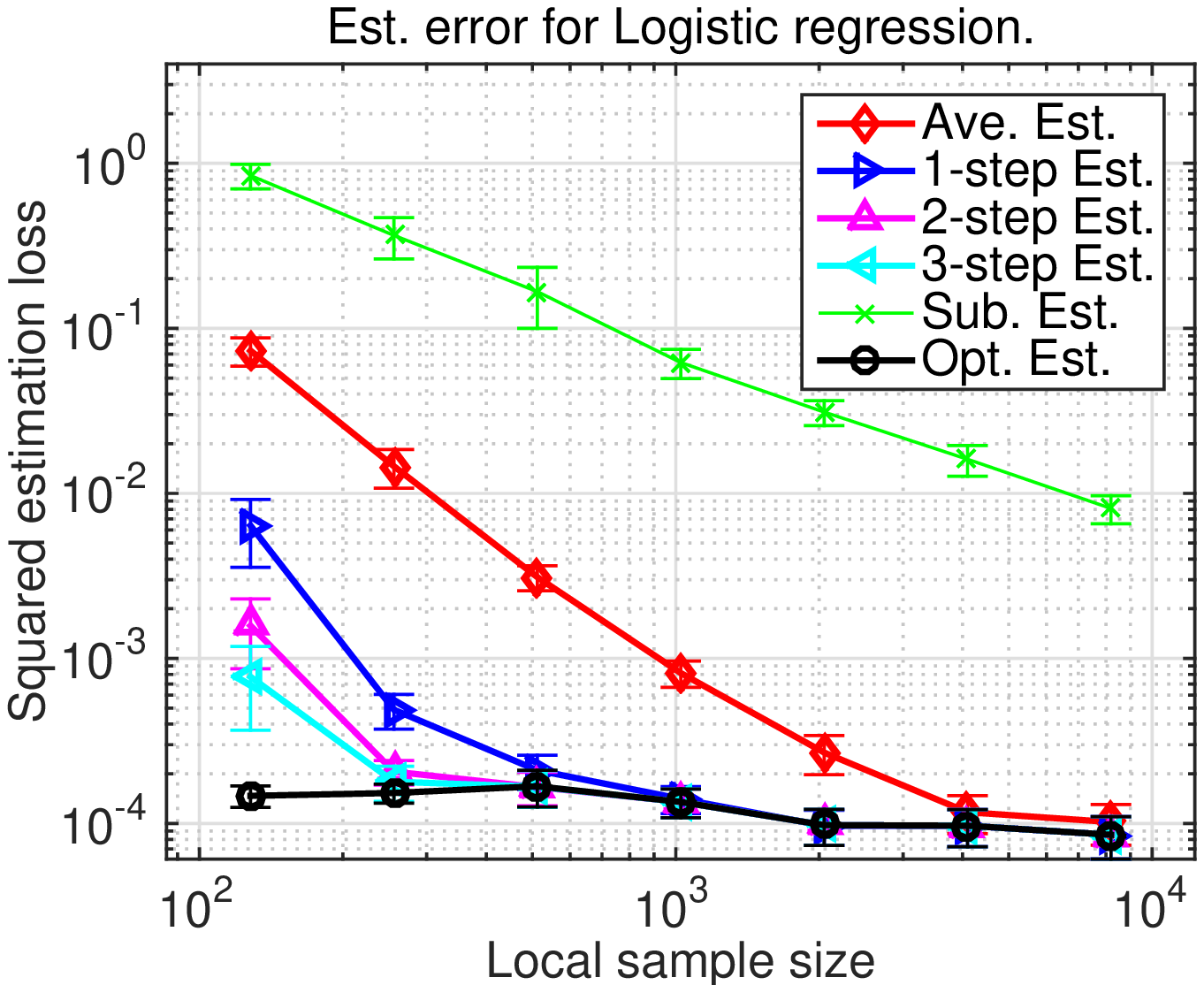} & & 
\widgraph{.45\textwidth}{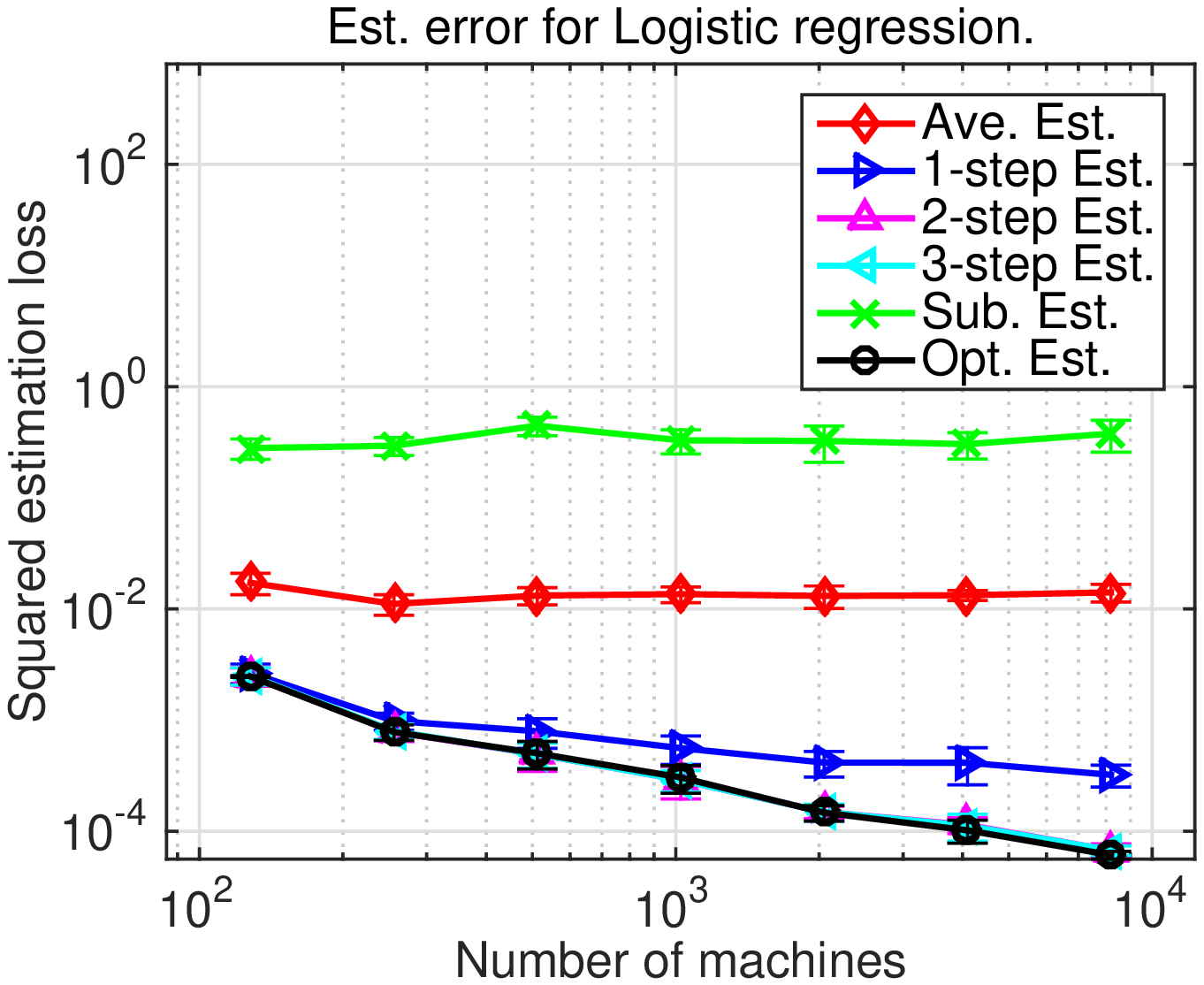} \\
(c) $d=10$ and $N = 524288$. & & (d) $d=10$ and $n = 256$.\\
\widgraph{.45\textwidth}{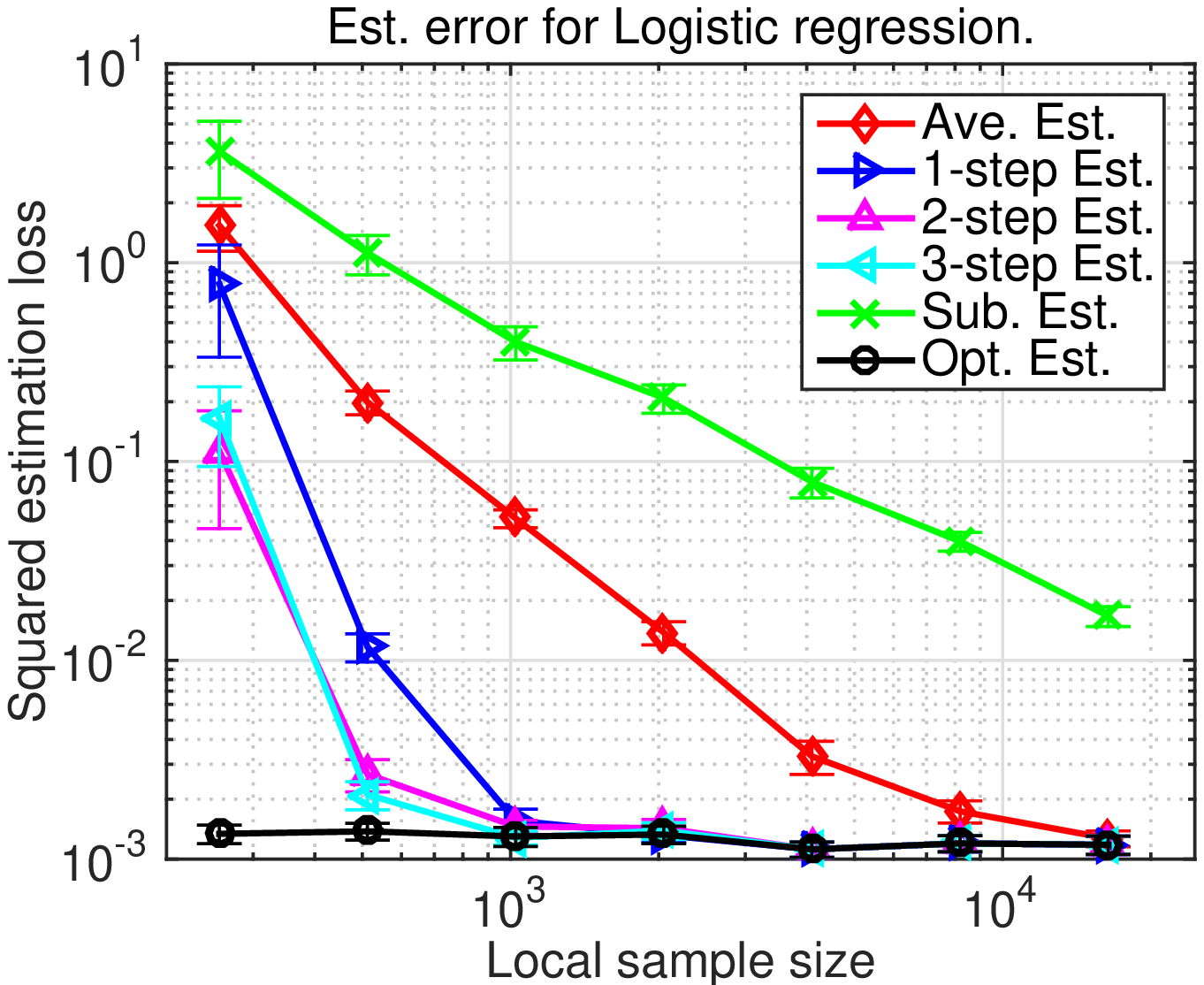} & & 
\widgraph{.45\textwidth}{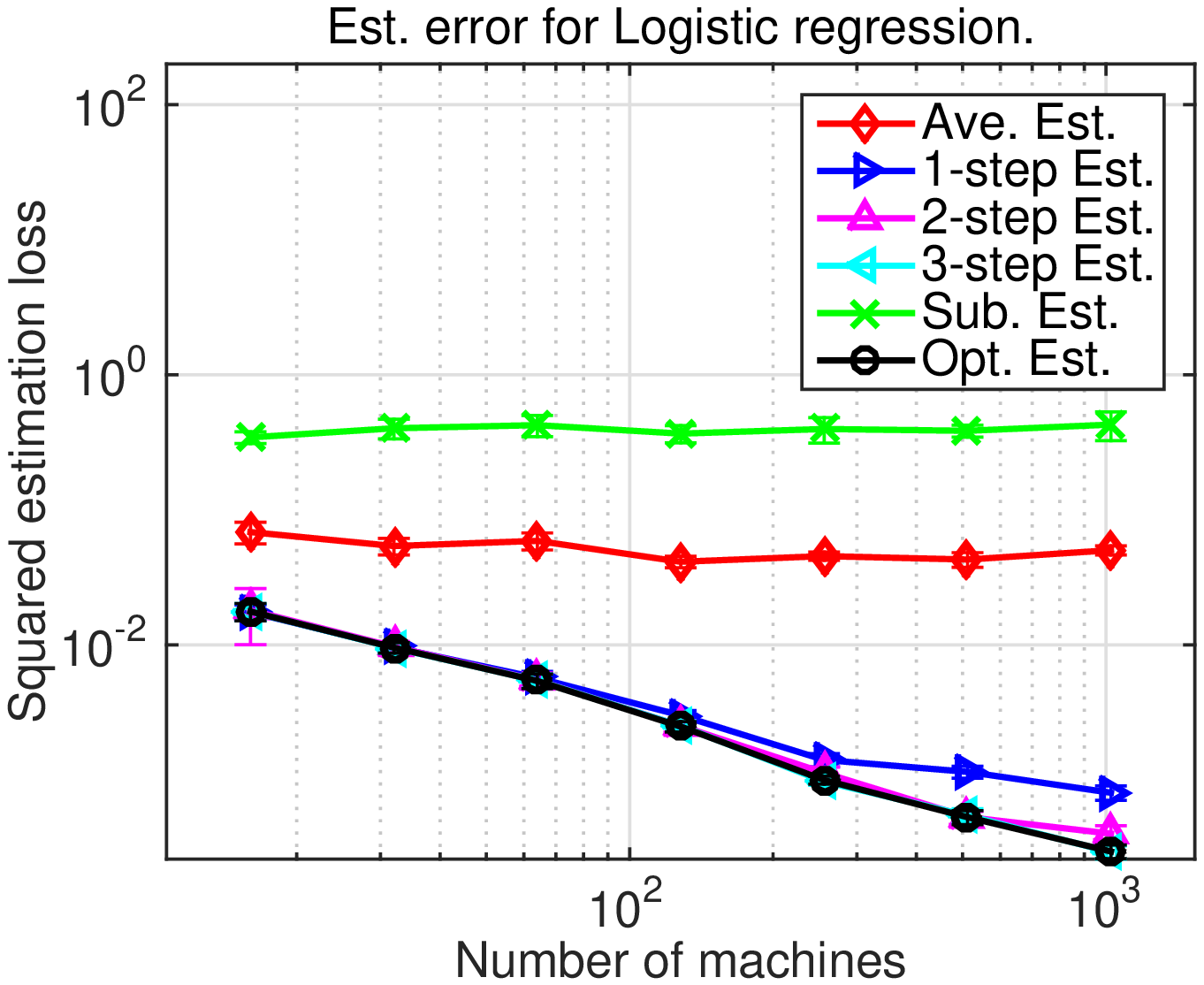} \\
(e) $d=50$ and $N = 524288$. & & (f) $d=50$ and $n = 2048$.
\end{tabular}
\end{center}
\caption{Squared estimation error $\|\htheta - \thetas\|_2^2$ versus local 
sample size $n$ and number of machines $k$ for logistic regression. 
In all cases, each point corresponds to the average of $100$ trials, 
with standard errors also shown.  In plots (a), (c) and (e), we change 
the local sample size $n$ while fixing the total sample size $N$ 
(number of machines $k = N / n$) for dimension $d\in\{2,10,50\}$. 
In plots (b), (d) and (f), we change the number of machines $k$ while 
fixing the local sample size $n$ (total sample size $N=nk$) under 
dimension $d\in\{2,10,50\}$.}
\label{FigLogEst}
\end{figure}

We implement the one-step CSL estimator $\theta^{(1)}$ with the averaging 
estimator $\htheta^{A}$ (based on simply averaging the local estimators)
as our initial estimator $\btheta$. We also implement the iterative local 
estimation algorithm to produce 2-step and 3-step estimators $\theta^{(2)}$ 
and  $\theta^{(3)}$ by iteratively applying the one-step estimation procedure. 
We compare our communication-efficient estimators with the (optimal) global 
$M$-estimator $\theta^{global}$ and the subsample estimator $\theta^{sub}$ 
that only uses the local data in Machine $\cM_1$. Two different regimes 
are considered: (1) the total sample size $N$ is fixed at $N=2^{19} \approx 10^6$, 
and the local sample size $n$ varies from $10^2$ to $10^4$; (2) the local 
sample size $n$ is fixed at $64$ ($d=2$), $256$ ($d=10$) or $2048$ ($d=50$), 
and the number of machines $k$ varies from $10^2$ to $10^4$. 

Figure~\ref{FigLogEst} reports the results. In plots (a), (c) and (d), 
the total sample size $N$ is fixed and therefore the estimation error 
associated with the global estimate $\theta^{global}$ remains approximately
fixed as $n$ varies. As expected, the remaining estimators exhibit a 
rapid decay in the estimation error as the local sample size $n$ grows. 
Our communication-efficient estimators yield the best performance among
the distributed estimators.  When $n$ is sufficiently large, the 1-step, 2-step 
and 3-step estimators have almost the same performance as $\theta^{global}$. 
However, as $n$ becomes small, further application of the iterative local 
estimation procedure in Algorithm~\ref{AlgoIterative} does not improve 
the statistical accuracy.  This is in fact consistent with 
Theorem~\ref{ThmNewtonerr}---the contraction coefficient 
$\|\theta^{(t+1)} - \theta^{global}\|_2 / \|\theta^{(t)} - \theta^{global}\|_2$ 
is dominated by the sum of two terms: the initial estimation error 
$\|\theta^{(t)} - \theta^{global}\|_2$ and the local Hessian approximation 
error $\matnorm{\nabla \cL_1(\thetas) - \nabla \cL_N(\thetas)}{2}$. 
Even though the initial estimation error can be reduced to a small level, 
the local Hessian approximation error still persists for small $n$ and 
prevents further improvement from application of the iterative procedure. 
We remark that the condition that the local size $n$ should exceed a 
$d$-dependent threshold is a mild requirement in practice. Indeed, 
the local machine storage limit in reality is often large enough to 
ensure $n\gg d$. Even under the scenario (small $n$) where our theory 
fails to predict, the 1-step, 2-step and 3-step estimators still have 
better performance than $\htheta^A$ and $\theta^{sub}$.  In plots (b), 
(d) and (e), we fix the local sample size $n$ under different $d$ such 
that $n$ exceeds the $d$-dependent threshold, and gradually increase 
the number of machines $k$. In our regime, $k$ is comparable or even 
much larger than $n$, and therefore the averaging estimator $\htheta^{A}$ 
does not improve as more data is available. This is consistent with 
theoretical results in \cite{Zhang13} that require $k \gg n$ for 
$\htheta^{A}$ to have comparable performance as $\theta^{global}$. 
By using our approach, even a single step of Algorithm~\ref{AlgoIterative} 
significantly improves the accuracy of $\htheta^{A}$. Moreover, 
$\theta^{(2)}$ and $\theta^{(3)}$ achieve almost the same accuracy 
as $\theta^{global}$. Consistent with our theory, for a fixed number 
of steps $t$, the $t$-step estimate $\theta^{(t)}$ tends to have larger 
estimation error than $\theta^{global}$ as $k$ grows. In plot (d), 
even for $k$ as large as $10^4$ (much larger than the local sample 
size $n\sim 10^2$), the 2-step estimate $\theta^{(2)}$ already 
achieves the same level of estimation accuracy as the global estimator 
$\theta^{global}$. 

\begin{figure}[htp]
\begin{center}
\begin{tabular}{ccc}
\widgraph{.45\textwidth}{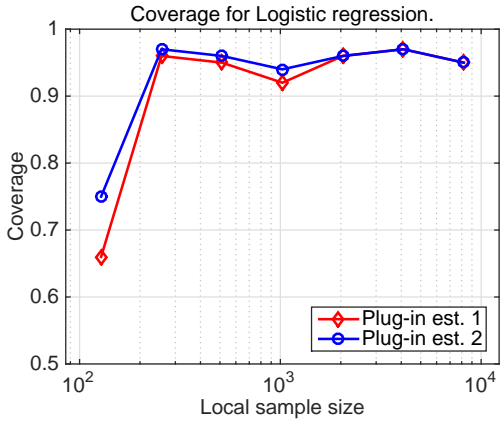} & &
\widgraph{.45\textwidth}{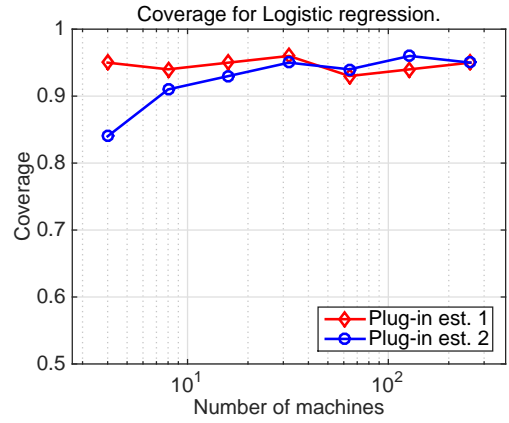}\\
(a) $d=10$ and $N = 524288$. & & (b) $d=10$ and $n = 256$.\
\end{tabular}
\end{center}
\caption{Coverage of the confidence interval for the first component 
of $\beta$ versus local sample size $n$ and number of machines $k$ for 
logistic regression under $d=10$. In all cases, the coverage probability 
is computed based on $100$ trials. Here, ``plug-in est. 1'' corresponds 
to the confidence interval constructed based on the plug-in estimator 
$\tSigma$ and the 3-step estimator $\theta^{(3)}$, whereas ``plug-in 
est. 2'' is based on $\tSigma'$.  In plots (a), we change the local 
sample size $n$ while fixing the total sample size $N$ (number of 
machines $k = N / n$). In plots (b), we change the number of machines 
$k$ while fixing the local sample size $n$ (total sample size $N=nk$).}
\label{FigLogCov}
\end{figure}

We now assess the performance of the inference procedures based on the 
plug-in estimators $\tSigma$ and $\tSigma'$ under the logistic 
model~\eqref{EqnLogsitic}. We use $\tSigma$ or $\tSigma'$ and the 
3-step estimator $\theta^{(3)}$ to construct a $95\%$ confidence 
interval (CI) for the first component $\theta_1$ of $\theta$ as 
\begin{align*}
\big[\theta^{(3)}_1 - 1.96\, \tSigma_{11} / \sqrt{N},\,  \theta^{(3)}_1 + 1.96\, \tSigma_{11} / \sqrt{N}\,\big] \mbox{\quad or \quad} \big[\theta^{(3)}_1 - 1.96 \,\tSigma'_{11} / \sqrt{N},\,  \theta^{(3)} + 1.96\, \tSigma'_{11} / \sqrt{N}\,\big].
\end{align*}
The coverage of the CI based on $100$ trials is calculated. 
Figure~\ref{FigLogCov} shows the results. In plot (a), coverage 
based on both plug-in estimators is low at $n=2^7$ because the 
sample size is so small that the center $\theta^{(3)}$ of the 
CI has a large bias (see Figure~\ref{FigLogEst}\,(c)).  In plot (b), 
the CI based on $\tSigma'$ has low coverage when the number $k$ 
of machines is small, which is consistent with our theory. 
In all other regimes of $(n, k)$, both CI's have coverage that 
is close to the nominal level $95\%$. Moreover, the CI based on 
$\tSigma'$ is slightly better than the one based on $\tSigma$ 
for large $k$, which empirically supports our intuition in the 
discussion after Corollary~\ref{CoroAsympExpan}.

\subsection{Distributed sparse linear regression}
We evaluate the CSL estimator on the sparse linear regression problem. 
The data is generated as $y_{ij}=X_{ij} ^T \thetas +\epsilon_{ij}$, 
where $i \in [n]$ and $ j \in[k]$. The covariates $X_{ij}$ are 
i.i.d.~$\cN(0,1)$, the noise $\epsilon_{ij}$ is i.i.d.~$\cN(0,1)$,  
and $\thetas$ is $s$-sparse with signal-to-noise ratio 
$\frac{|\theta_i|}{\sigma} =5$.

In the first experiment, we keep the total data size $N$ fixed, and 
increase the number of machines $k$. This corresponds to each machine 
having a smaller local sample size $n$ as $k$ increases. We observe that 
the one-step CSL estimator has nearly constant error, even though each 
machine has less local data. In fact at $k=30$, the local data size is 
$n=720$, which is much smaller than $d$, yet the CSL estimator achieves 
the same mean-square error as lasso on all $N$ points. The error of the 
averaging estimator increases dramatically as $n$ decreases, since the 
mean-squared error is $\frac{s \log d}{n}$, showing that the averaging 
algorithm is not suitable in this setting. 

In the second experiment, we keep $n$ fixed and increase $k$ and $N$. 
As predicted by our theory, the one-step CSL estimator has error that 
is linear on the log-log scale because the mean-squared error scales 
as $ \frac{s \log d}{nk}$. The averaging estimator has error that slowly 
decreases with the increased sample size, due to the bias induced by 
regularization. The averaging estimator does not attain mean-square 
error of $\frac{s \log d}{nk}$.

\begin{figure}[htp]
\begin{center}
\begin{tabular}{ccc}
\widgraph{.45\textwidth}
{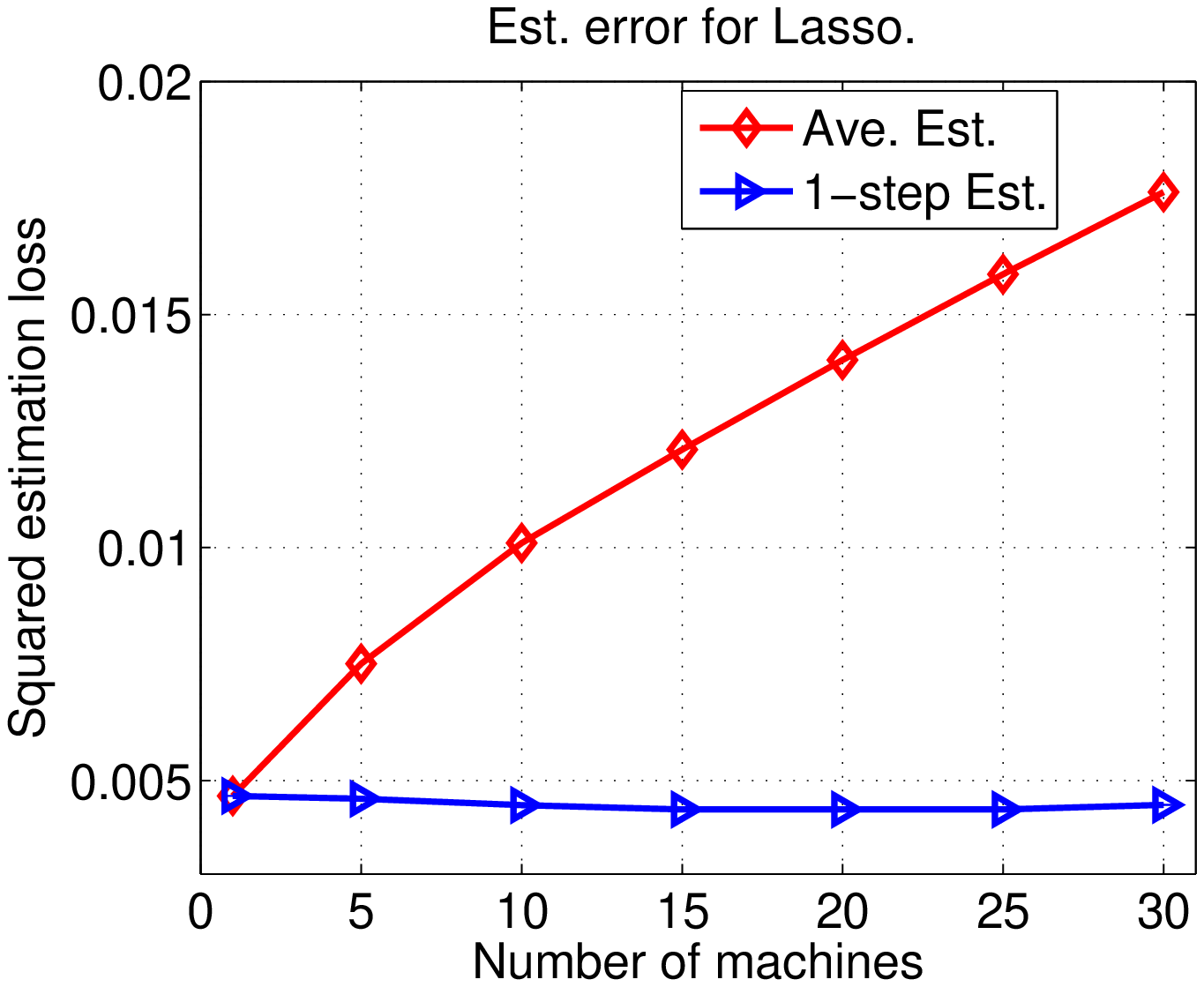}
& &
\widgraph{.45\textwidth}{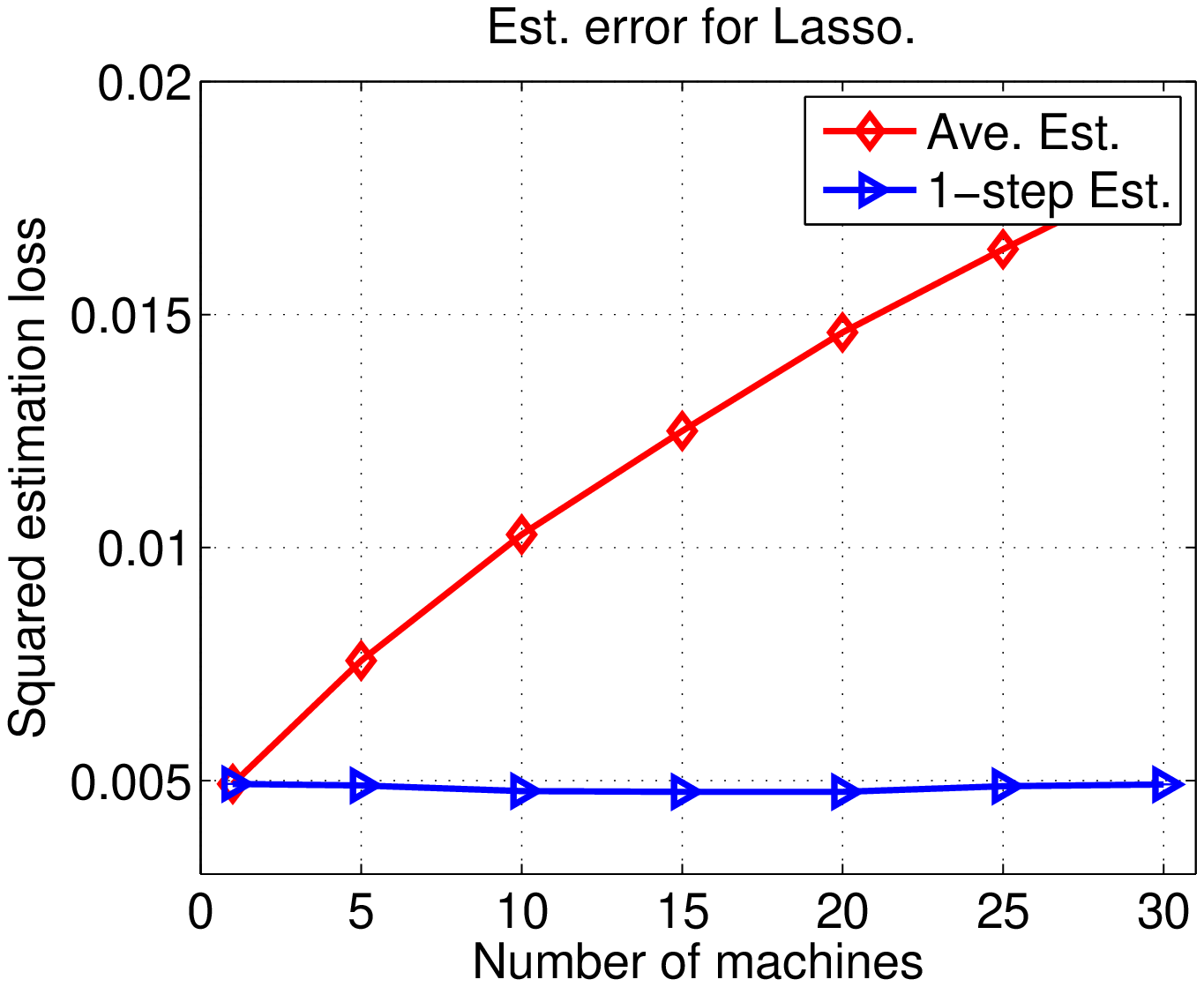}\\
$(N,d,s) = (21600, 5000,20)$
& &
$(N,d,s) = (21600, 10000, 20)$
\end{tabular}
\end{center}
\caption{As $k \in \{1,5,10,15,20,25,30\}$ increases, the local data size $n$ decreases, but the one-step CSL estimator has constant error. The averaging estimator error increases, since $k$ decreases.  }
\label{fig:lasso-fixN}
\end{figure}

\begin{figure}[htp]
\begin{center}
\begin{tabular}{ccc}
\widgraph{.45\textwidth}
{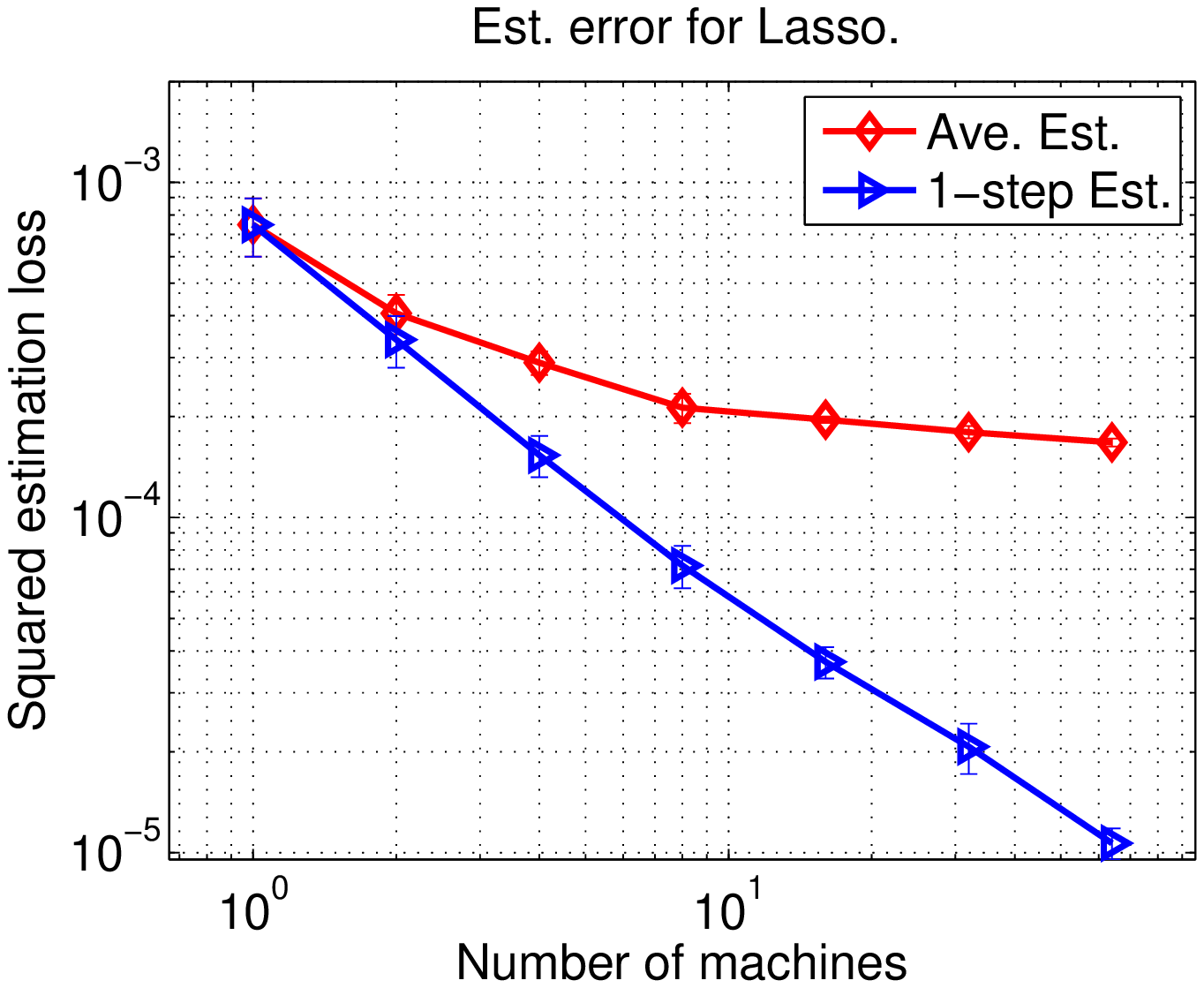}
& &
\widgraph{.45\textwidth}{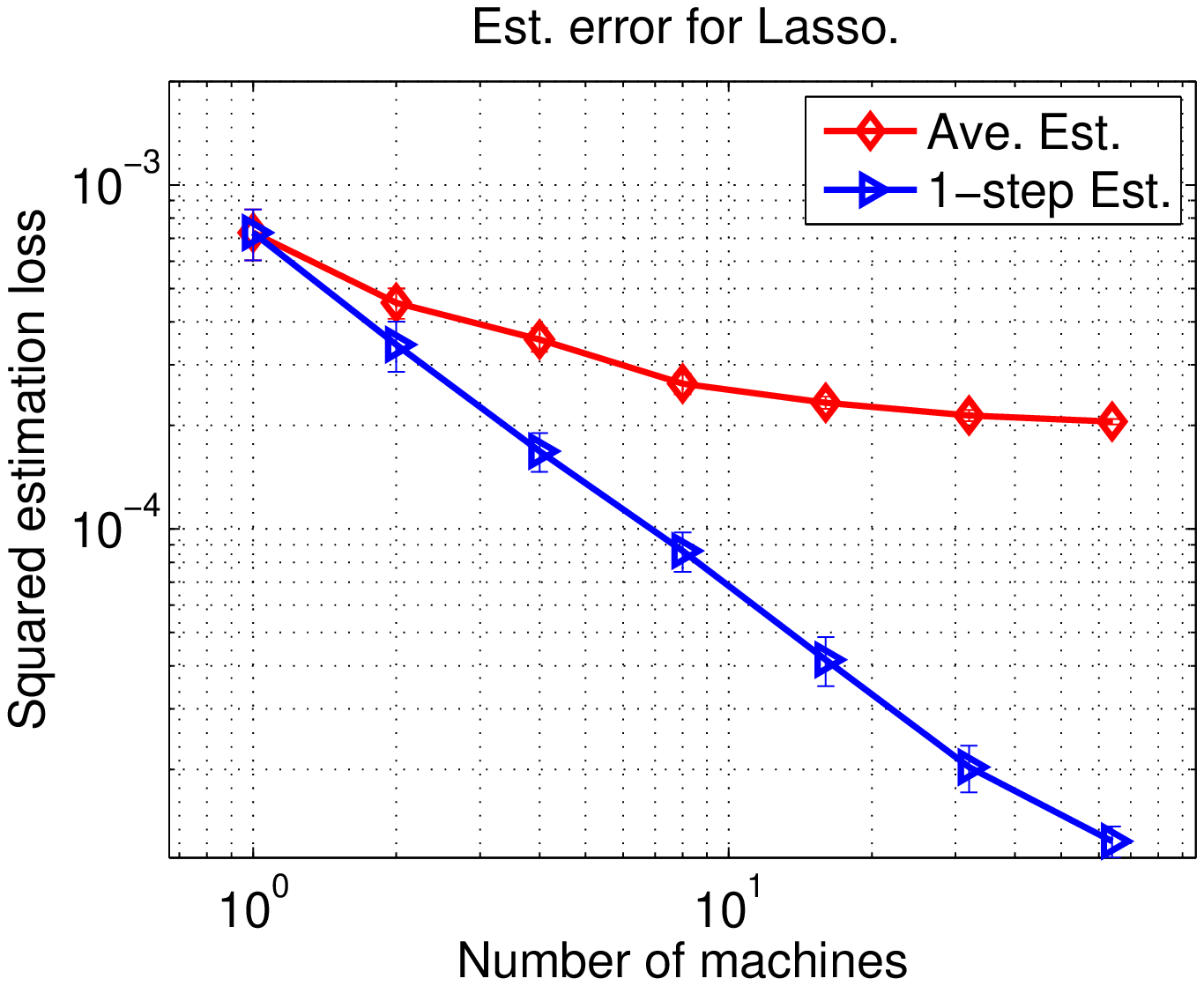}\\
$(n,d,s) = (800,1000,20)$ 
& &
$(n,d,s) = (800,2000,20)$
\end{tabular}
\end{center}
\caption{(a) As $k \in \{1,2, 4,8,16,32,64\}$ increases, the mean-squared error of the one-step CSL estimator decreases. For the averaging estimator, the mean-squared error does not decrease significantly for large values of $k$.}
\label{fig:lasso-fixn}
\end{figure}

%
%
%

\subsection{Distributed Bayesian inference}

\begin{figure}[htp]
\begin{center}
\begin{tabular}{ccc}
\widgraph{.45\textwidth}{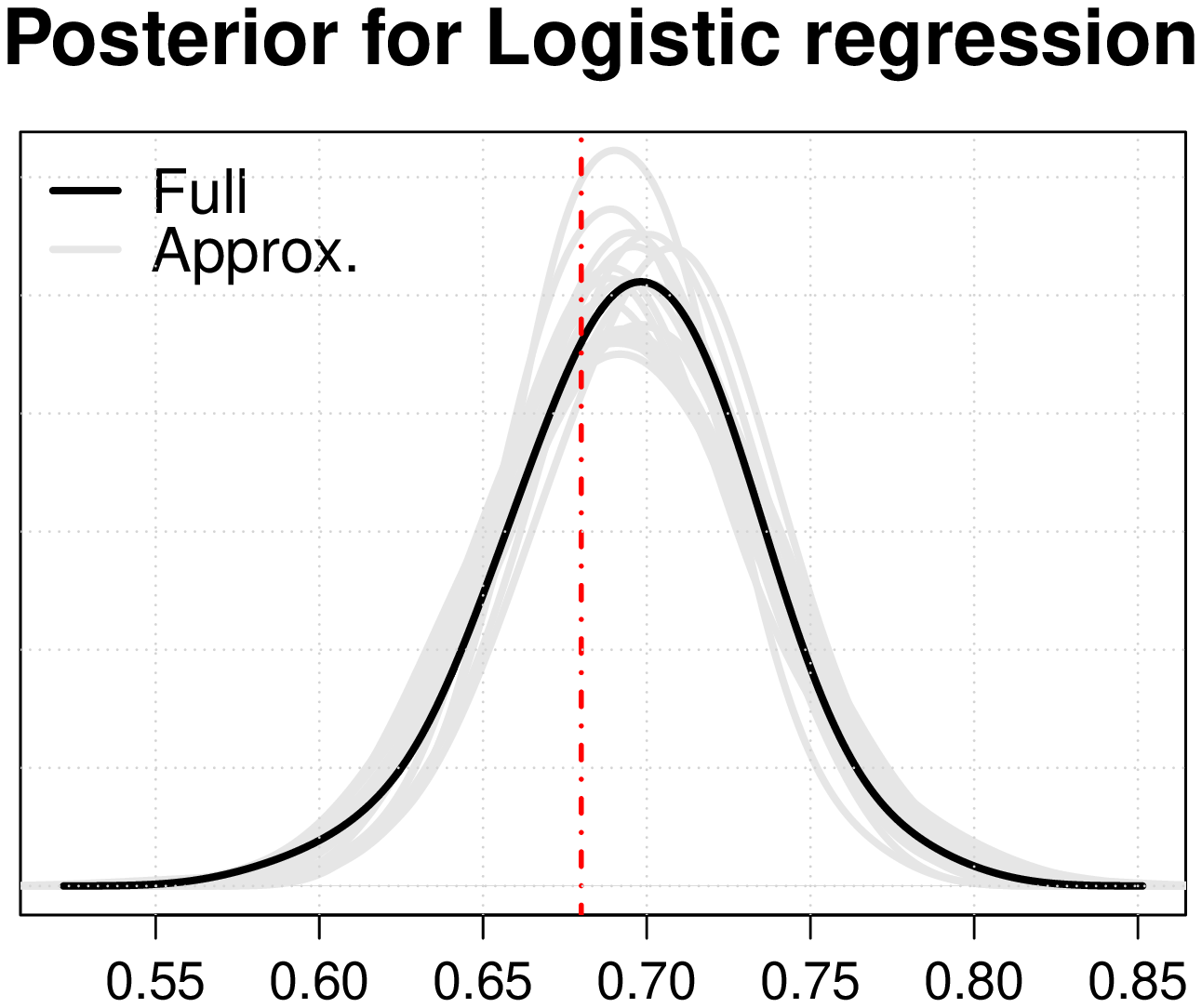} & &
\widgraph{.45\textwidth}{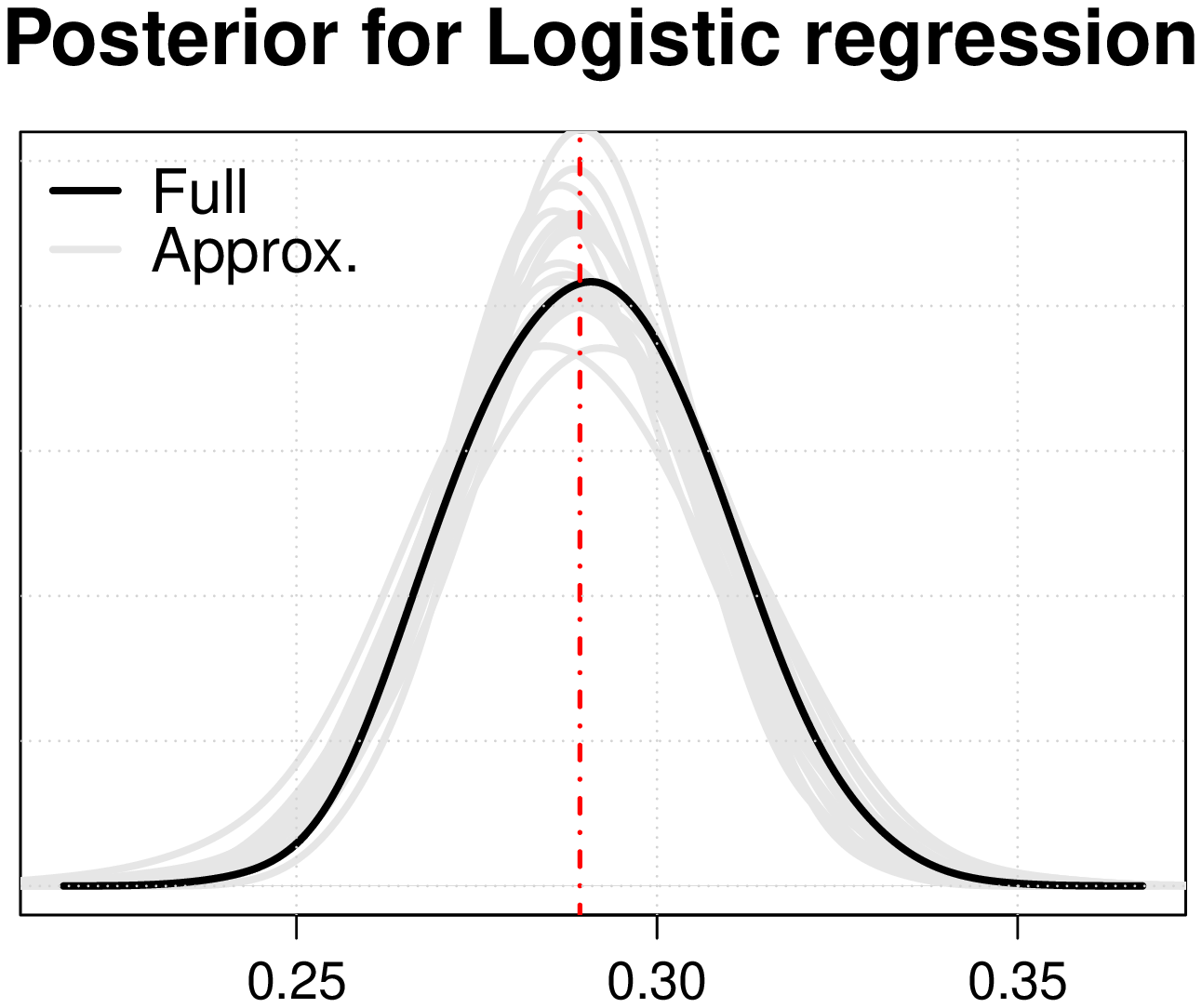} \\
(a) $(d,\,n,\,k)=(2,\, 64,\, 64)$. & & (b) $(d,\,n,\,k)=(2,\, 64,\, 256)$.\\
\widgraph{.45\textwidth}{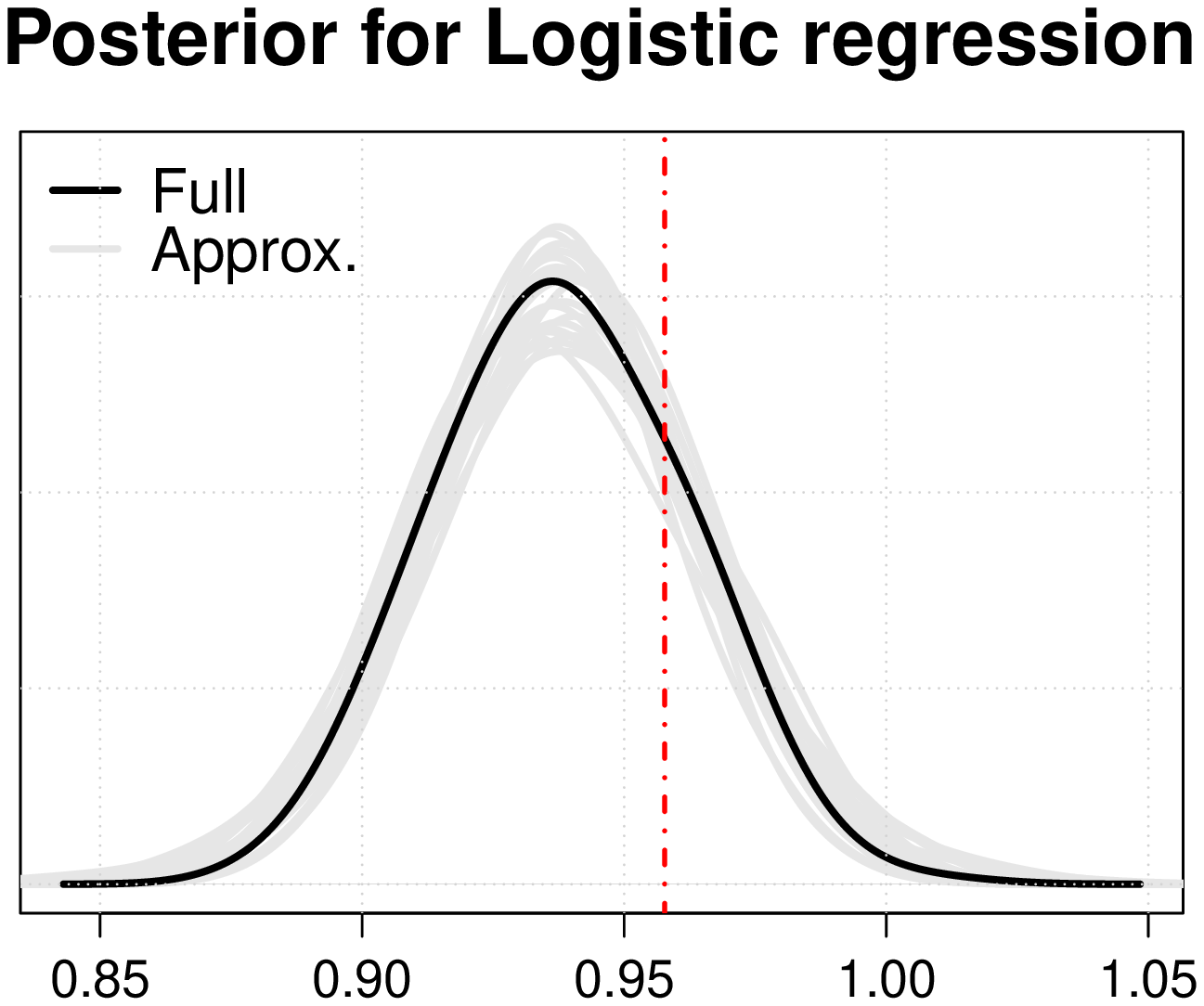} & & 
\widgraph{.45\textwidth}{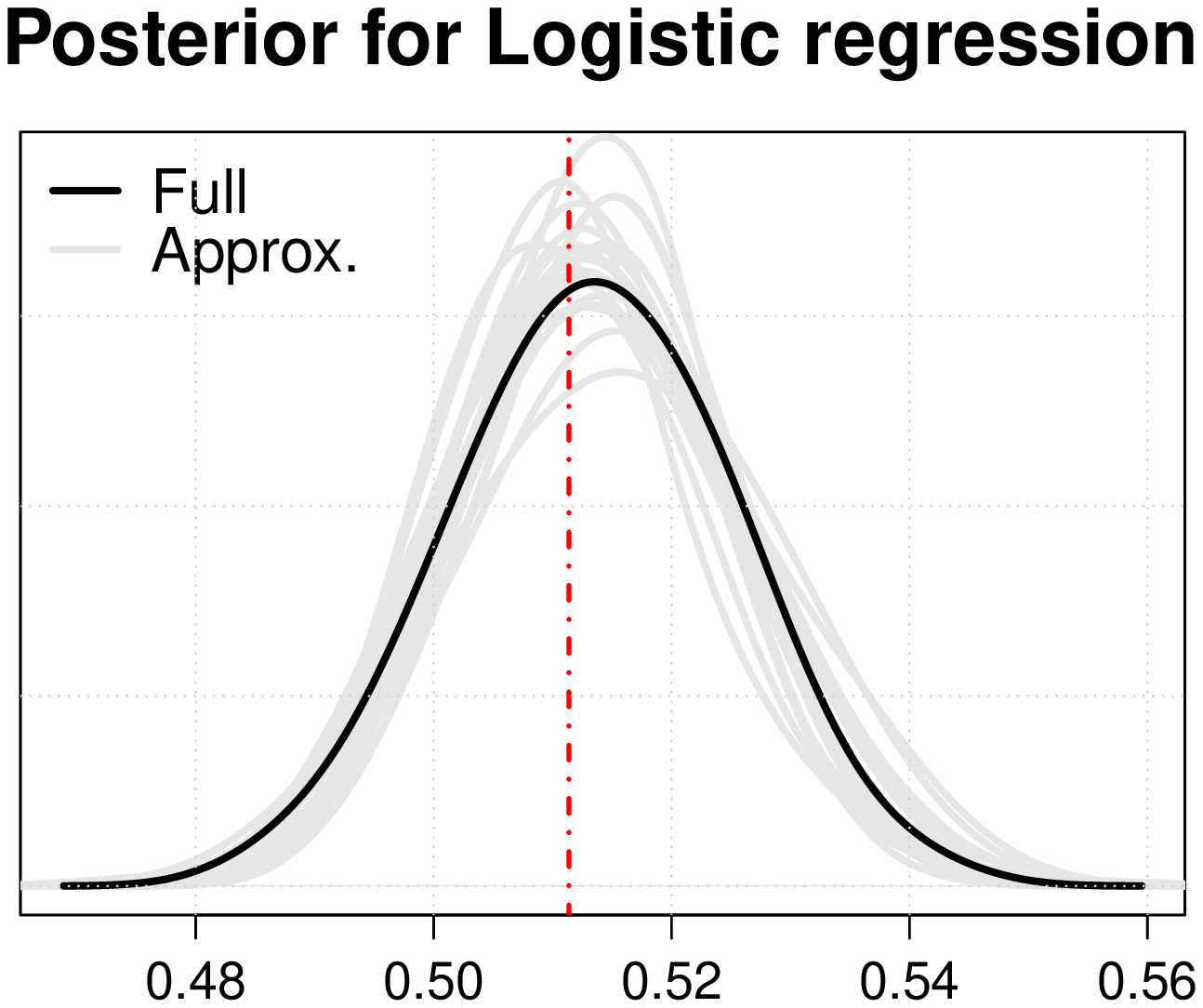} \\
(c) $(d,\,n,\,k)=(10,\, 256,\,64)$. & & (d) $(d,\,n,\,k)=(10,\, 256,\, 256)$.\\
\widgraph{.45\textwidth}{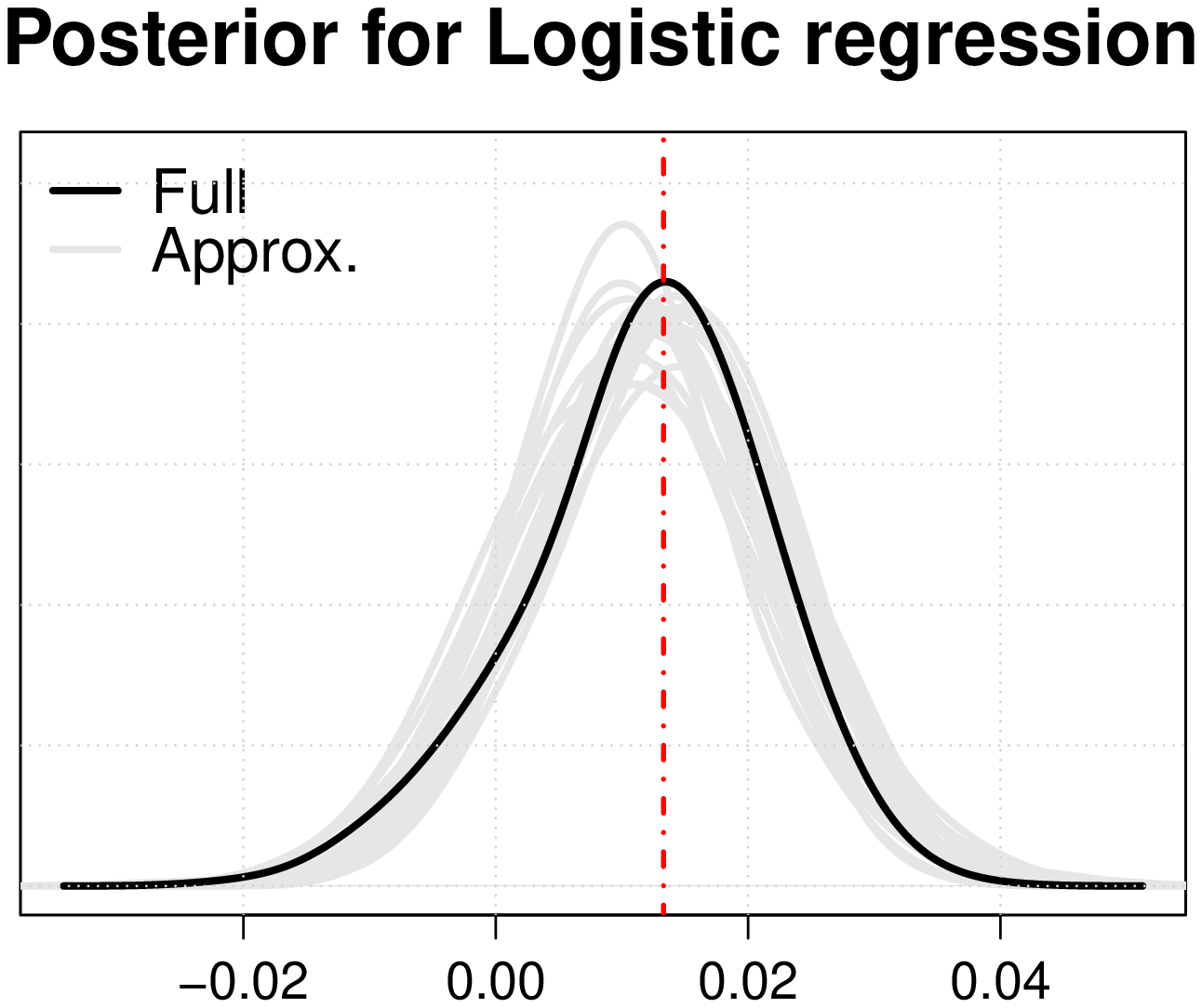} & & 
\widgraph{.45\textwidth}{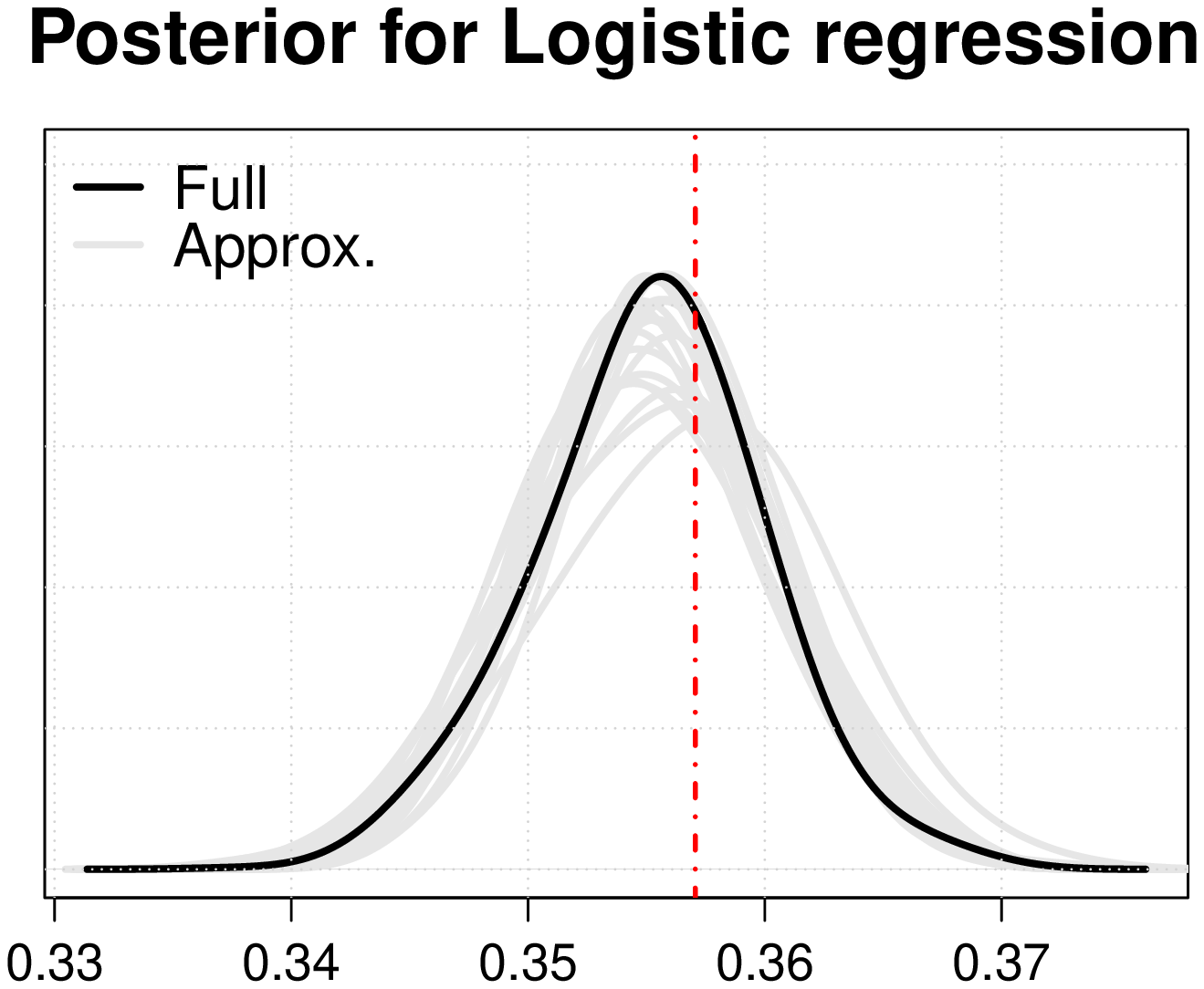} \\
(e) $(d,\,n,\,k)=(50,\, 2048,\,64)$. & & (f) $(d,\,n,\,k)=(50,\, 2048,\,256)$.
\end{tabular}
\end{center}
\caption{Marginal posterior distribution of the first component $\theta_1$ 
of $\theta$ for logistic regression for dimension $d\in\{2,10,50\}$ are shown. 
In each plot, $20$ approximations (grey curves) to the full posterior 
(black curve) are shown based on random splits of the data into $k$ 
subsamples. The vertical dotted line indicates the location of the truth 
$\thetas_1$. }
\label{FigBayes}
\end{figure}

Our synthetic dataset is generated from the logistic model~\eqref{EqnLogsitic} 
for dimension $d\in\{2, 10, 50\}$. We use the 3-step estimator $\theta^{(3)}$ 
in Section~\ref{SectionMestForPara} as the initial estimator $\btheta$ and 
implement the Bayesian procedures based on the (approximated) posterior 
distribution $\pi_n(\theta)$ and $\apost(\theta)$ by sampling a 
Markov Chain Monte Carlo algorithm. We use the Metropolis algorithm, 
where at each iteration the proposal distribution for $\theta$ is a $d$-dim 
Gaussian distribution centered at the current iterate $\theta^{(t)}$. 
In each case, we run the Markov chain for $20000$ iterations and treat 
the first half as burn-in. Figure~\ref{FigBayes} plots the (approximated) 
marginal posterior distributions of the first component $\theta_1$ of 
$\theta$ under different $(d,n,k)$ combinations ($n$ is chosen so that 
$\theta^{(3)}$ is a good approximation to the global estimator $\htheta$, 
see Figure~\ref{FigLogEst}). Consistent with our theoretical prediction, 
$\apost(\theta)$ provides a good approximation to $\pi_N(\theta)$ as 
long as the initial estimator $\btheta$ is sufficiently close to $\htheta$, 
even when $k$ is much larger than $n$ (see plot (b)). Since the computation 
of the approximate posterior distribution $\apost(\theta)$ only uses the local 
data in Machine $\cM_1$, the computation of the acceptance ratio using 
$\apost(\theta)$ is $k$ times as fast as that using the full data posterior 
$\pi_N(\theta)$ in each iteration of the Metropolis algorithm.


\section{Real data application}

We apply distributed logistic regression to a computer-vision dataset~\citep{Bhatt2012}. 
The goal is to predict whether a given color sample described by its B, G, R 
values (each ranges from $0-255$) corresponds to a skin sample or non-skin 
sample. The dataset is generated using skin textures from face images of 
diverse of age, gender, and race. The total sample size is $245,057$, out 
of which $50,859$ images are skin samples and $194,198$ are non-skin samples. 
The dataset contains three features---B, G, R values of the color, and a 
$0$-$1$ response variable indicates whether the sample is non-skin ($0$) 
or skin ($1$). We randomly split the dataset into a training set of size 
$N=200,000$ and a testing set $N_0 = 45057$, and use B-spline transforms 
(df$=15$) for each feature as predictors to allow a nonlinear dependence 
between the response and features. Therefore, the dimension of the covariate 
$X$ is $d=45$.

We randomly split the training set into $k_0=100$ subsets, each of size 
$n=2000$. We apply our distributed $M$-estimation method for logistic 
regression to a training set with $k\in \{20,40,60,80,100\}$ subsets, 
and test the fitted model to the testing set. We then exactly minimize 
the surrogate function to form the 1-step estimator $\theta^{(1)}$, using
the averaging estimator $\htheta^{A}$ \citep{Zhang13} as our initial 
estimator. We also implement the iterative local estimation algorithm 
to produce 2-step and 3-step estimators $\theta^{(2)}$ and  $\theta^{(3)}$ 
by iteratively applying the one-step estimation procedure.
Figure~\ref{FigSkin} plots the misclassification rate versus 
the number of subsets used. As we can see, the 1-step estimator yields
significant gains in prediction performance over the initial averaging 
estimator, and both the 2-step and 3-step estimators have similar 
prediction performance as the 1-step estimator. This suggests that 
for the skin dataset and our split setting, the one-step approximation 
of the likelihood function is already adequate.

\begin{figure}[htp]
\begin{center}
\widgraph{.75\textwidth}{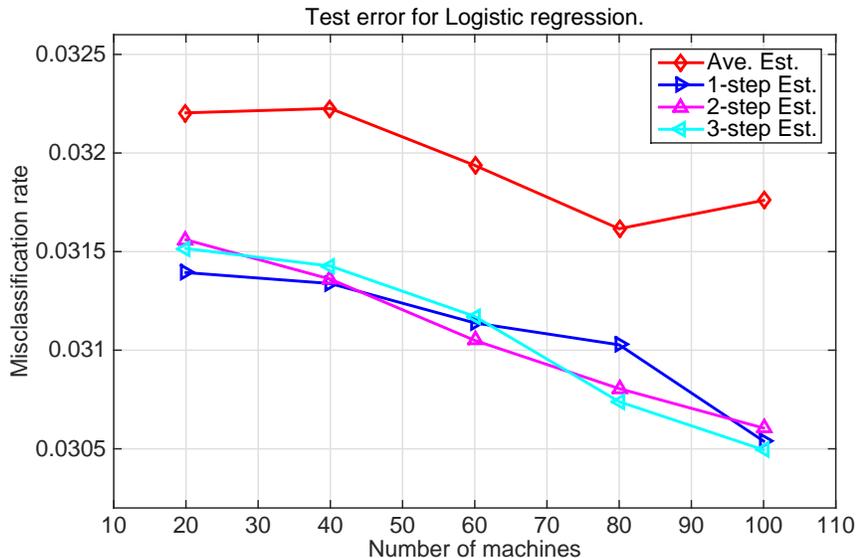} 
\end{center}
\caption{Distributed logistic regression for the skin dataset. }
\label{FigSkin}
\end{figure}


\section{Discussion}
We have presented a Communication-efficient Surrogate Likelihood (CSL) 
framework for solving distributed statistical inference problems. 
We applied this methodology to three problem domains: low-dimensional 
$M$-estimation, high-dimensional regularized estimation and low-dimensional 
Bayesian inference. Our results demonstrate that the general idea of 
constructing a surrogate to the negative log-likelihood function 
(or general loss function) is viable for communication-limited statistical
inference.  We also believe that the approach can prove useful for 
``big-data'' problems on a single machine, when the sample size is
large and the calculation of the likelihood function is expensive. 

There are several directions for future research in this area.  We would
like to find methods that permit high-dimensional distributed Bayesian 
inference. We would also like to find a sharp theoretical lower bound on 
the local sample size $n$ needed for the final estimator to remain optimal
(for example, in the minimax sense).  It is also worthwhile to consider
ensemble and hierarchical versions of the CSL method, in which multiple 
machines aggregate local results.

\bibliographystyle{plainnat}
\bibliography{one_step}

\begin{thebibliography}{29}
\providecommand{\natexlab}[1]{#1}
\providecommand{\url}[1]{\texttt{#1}}
\expandafter\ifx\csname urlstyle\endcsname\relax
  \providecommand{\doi}[1]{doi: #1}\else
  \providecommand{\doi}{doi: \begingroup \urlstyle{rm}\Url}\fi

\bibitem[Bissiri et~al.(2013)Bissiri, Holmes, and Walker]{Walker2013}
P.~Bissiri, C.~Holmes, and S.~Walker.
\newblock A general framework for updating belief distributions.
\newblock \emph{arXiv:1306.6430}, 2013.

\bibitem[Braverman et~al.(2015)Braverman, Garg, Ma, Nguyen, and
  Woodruff]{Braverman15}
Mark Braverman, Ankit Garg, Tengyu Ma, Huy Nguyen, and David Woodruff.
\newblock Communication lower bounds for statistical estimation problems via a
  distributed data processing inequality.
\newblock \emph{arXiv:1506.07216}, 2015.

\bibitem[Bubeck(2015)]{bubeck2014theory}
S.~Bubeck.
\newblock Theory of convex optimization for machine learning.
\newblock \emph{Foundations and Trends in Machine Learning}, 8, 2015.

\bibitem[Chernozhukov and Hong(2003)]{Chernozhukov03}
V.~Chernozhukov and H.~Hong.
\newblock An {MCMC} approach to classical estimation.
\newblock \emph{Journal of Econometrics}, 115\penalty0 (2):\penalty0 293 --
  346, 2003.

\bibitem[Cleveland and Hafen(2014)]{Cleveland2014}
W.~Cleveland and R.~Hafen.
\newblock Divide and recombine ({D}\&{R}): Data science for large complex data.
\newblock \emph{Statistical Analysis and Data Mining}, 7:\penalty0 425--433,
  2014.

\bibitem[Demmel et~al.(2012)Demmel, Grigori, Hoemmen, and
  Langou]{DemmelEtAl2012}
J.~Demmel, L.~Grigori, M.~Hoemmen, and J.~Langou.
\newblock Communication-optimal parallel and sequential {QR} and {LU}
  factorizations.
\newblock \emph{SIAM Journal on Scientific Computing}, 34:\penalty0 206--239,
  2012.

\bibitem[Duchi et~al.(2015)Duchi, Jordan, Wainwright, and Zhang]{Duchi15}
J.~Duchi, M.~Jordan, M.~Wainwright, and Y.~Zhang.
\newblock Optimality guarantees for distributed statistical estimation.
\newblock \emph{arXiv:1405.0782}, 2015.

\bibitem[Duchi et~al.(2012)Duchi, Agarwal, and Wainwright]{DuchiAW12}
John~C. Duchi, Alekh Agarwal, and Martin~J. Wainwright.
\newblock Dual averaging for distributed optimization: Convergence analysis and
  network scaling.
\newblock \emph{IEEE Transactions on Automatic Control}, 57:\penalty0 592--606,
  2012.

\bibitem[Garg et~al.(2014)Garg, Ma, and Nguyen]{garg2014communication}
A.~Garg, T.~Ma, and H.~Nguyen.
\newblock On communication cost of distributed statistical estimation and
  dimensionality.
\newblock In \emph{Advances in Neural Information Processing Systems}, pages
  2726--2734, 2014.

\bibitem[Kannan et~al.(2014)Kannan, Vempala, and Woodruff]{Kannan14}
R.~Kannan, S.~Vempala, and D.~Woodruff.
\newblock Principal component analysis and higher correlations for distributed
  data.
\newblock In \emph{the 27th Conference on Learning Theory}, pages 1040--1057,
  2014.

\bibitem[Kleiner et~al.(2014)Kleiner, Talwalkar, Sarkar, and
  Jordan]{KleinerEtAl2014}
A.~Kleiner, A.~Talwalkar, P.~Sarkar, and M.~I. Jordan.
\newblock A scalable bootstrap for massive data.
\newblock \emph{Journal of the Royal Statistical Society, Series B},
  76:\penalty0 795--816, 2014.

\bibitem[Lee et~al.(2015)Lee, Sun, Liu, and Taylor]{lee2015}
J.~Lee, Y.~Sun, Q.~Liu, and J.~Taylor.
\newblock Communication-efficient sparse regression: a one-shot approach.
\newblock \emph{arXiv:1503.04337}, 2015.

\bibitem[Mackey et~al.(2015)Mackey, Talwalkar, and Jordan]{MackeyEtAl2015}
L.~Mackey, A.~Talwalkar, and M.~I. Jordan.
\newblock Distributed matrix completion and robust factorization.
\newblock \emph{Journal of Machine Learning Research}, 16:\penalty0 913--960,
  2015.

\bibitem[Maclaurin and Adams(2014)]{MaclaurinAdams2014}
D.~Maclaurin and R.~Adams.
\newblock Firefly {M}onte {C}arlo: Exact {MCMC} with subsets of data.
\newblock \emph{arXiv:1403.5693}, 2014.

\bibitem[Negahban et~al.(2012)Negahban, Ravikumar, Wainwright, and
  Yu]{negahban2012unified}
S.~Negahban, P.~Ravikumar, M.~Wainwright, and B.~Yu.
\newblock A unified framework for high-dimensional analysis of {M}-estimators
  with decomposable regularizers.
\newblock \emph{Statistical Science}, 27\penalty0 (4):\penalty0 538--557, 2012.

\bibitem[Neiswanger et~al.(2015)Neiswanger, Wang, and Xing]{Xing15}
W.~Neiswanger, C.~Wang, and E.~Xing.
\newblock Asymptotically exact, embarrassingly parallel {MCMC}.
\newblock \emph{arXiv:1311.4780}, 2015.

\bibitem[Pilanci and Wainwright(2014)]{pilanci2014iterative}
M.~Pilanci and M.~Wainwright.
\newblock Iterative {H}essian sketch: {F}ast and accurate solution
  approximation for constrained least-squares.
\newblock \emph{arXiv:1411.0347}, 2014.

\bibitem[Rabinovich et~al.(2016)Rabinovich, Angelino, and
  Jordan]{RabinovichEtAl2016}
M.~Rabinovich, E.~Angelino, and M.~Jordan.
\newblock Variational consensus {M}onte {C}arlo.
\newblock In \emph{Advances in Neural Information Processing Systems}, Red
  Hook, NY, 2016. Curran Associates.

\bibitem[Rajen and Abhinav(2012)]{Bhatt2012}
G.~Rajen and D.~Abhinav.
\newblock Skin segmentation dataset.
\newblock \emph{UCI Machine Learning Repository}, 2012.

\bibitem[Raskutti et~al.(2010)Raskutti, Wainwright, and Yu]{Raskutti2010}
G.~Raskutti, M.~Wainwright, and B.~Yu.
\newblock Restricted eigenvalue properties for correlated {G}aussian designs.
\newblock \emph{Journal of Machine Learning Research}, 11:\penalty0 2241--2259,
  2010.

\bibitem[Scott et~al.(2016)Scott, Blocker, Bonassi, Chipman, George, and
  McCulloch]{ScottEtAl2016}
S.~Scott, A.~Blocker, F.~Bonassi, H.~Chipman, E.~George, and R.~McCulloch.
\newblock Bayes and big data: the consensus {M}onte {C}arlo algorithm.
\newblock \emph{International Journal of Management Science and Engineering
  Management}, 11:\penalty0 78--88, 2016.

\bibitem[Shamir et~al.(2014)Shamir, Srebro, and Zhang]{shamir2014communication}
O.~Shamir, N.~Srebro, and T.~Zhang.
\newblock Communication-efficient distributed optimization using an approximate
  {N}ewton-type method.
\newblock In \emph{Proceedings of the 31st International Conference on Machine
  Learning (ICML-14)}, pages 1000--1008, 2014.

\bibitem[Suchard et~al.(2010)Suchard, Wang, Chan, Frelinger, Cron, and
  West]{SuchardEtAl2010}
M.~Suchard, Q.~Wang, C.~Chan, J.~Frelinger, M.~Cron, and M.~West.
\newblock Understanding gpu programming for statistical computation: Studies in
  massively parallel massive mixtures.
\newblock \emph{Journal of Computational and Graphical Statistics},
  19:\penalty0 419--438, 2010.

\bibitem[Terenin et~al.(2016)Terenin, Simpson, and Draper]{TereninEtAl2016}
A.~Terenin, D.~Simpson, and D.~Draper.
\newblock Asynchronous {G}ibbs sampling.
\newblock \emph{arXiv:1509.08999}, 2016.

\bibitem[van~de Geer et~al.(2014)van~de Geer, B\"{u}hlmann, Ritov, and
  Dezeure]{vandegeer2014}
S.~van~de Geer, P.~B\"{u}hlmann, Y.~Ritov, and R.~Dezeure.
\newblock On asymptotically optimal confidence regions and tests for
  high-dimensional models.
\newblock \emph{Annals of Statistics}, 42:\penalty0 1166--1202, 2014.

\bibitem[Wang et~al.(2016)Wang, Kolar, Srebro, and Zhang]{wang2016}
J.~Wang, M.~Kolar, N.~Srebro, and T.~Zhang.
\newblock Efficient distributed learning with sparsity.
\newblock \emph{arXiv:1605.07991}, 2016.

\bibitem[Wang and Dunson(2015)]{Wang15}
X.~Wang and D.~Dunson.
\newblock Parallelizing {MCMC} via {W}eierstrass sampler.
\newblock \emph{arXiv:1312.4605}, 2015.

\bibitem[Zhang and Lin(2015)]{Zhang15}
Y.~Zhang and X.~Lin.
\newblock Communication-efficient distributed optimization of self-concordant
  empirical loss.
\newblock \emph{arXiv:1501.00263}, 2015.

\bibitem[Zhang et~al.(2013)Zhang, Duchi, and Wainwright]{Zhang13}
Y.~Zhang, J.~Duchi, and M.~Wainwright.
\newblock Communication-efficient algorithms for statistical optimization.
\newblock \emph{Journal of Machine Learning Research}, 14:\penalty0 3321--3363,
  2013.

\end{thebibliography}

\newpage
\appendix
\makeatletter   
 \renewcommand{\@seccntformat}[1]{APPENDIX~{\csname the#1\endcsname}.\hspace*{1em}}
 \makeatother

\section{Proofs of main results}


\subsection{Proof of Theorem~\ref{ThmMestNoReg}}
\label{SectionProofThmMestNoReg}
For $j=1,\ldots,k$, let $M_j = \frac1n \sum_{i=1}^n M(z_{ij})$ and $\delta_{\rho} = \min\{\rho,\, \rho\mum/4M\}$.
Consider the following ``good events":
\begin{align*}
\cE_0& \defn \bigg\{ \|\htheta - \thetas\|_2 \leq \min\Big\{\frac{\rho\mum}{8M},\,  \frac{(1-\rho)\mum \delta_{\rho}}{8\mup},\,  \sqrt{\frac{(1 - \rho)\mum \delta_{\rho}}{16M}}\,\Big\} \bigg\}, \qquad\mbox{and}\\
\cE_{j} & \defn \Big\{M_j \leq 2 M, \, \matnorm{\nabla^2 \cL_j(\thetas) - I(\thetas)}{2} \leq \frac{\rho \mum}{4}, \, \| \nabla \cL_j(\thetas) \|_2 \leq \frac{(1 - \rho)\mum \delta_\rho}{4} \Big\}.
\end{align*}
Before proving the claimed error bound for $\ttheta$, we state two auxiliary results that are used in the proof. 
The first result provides control on the probability of a bad event $\bigcup_{j=0}^k \cE_j^c$, which is proved in Appendix~\ref{AppProofLemSmallProbEvent}.
\begin{lemma}
\label{LemSmallProbEvent}
Under Assumptions PA-PD, we have
\begin{align*}
\Pp\Big(\bigcup_{j=0}^k\cE_{j}^c\Big) \leq \big(c_1 + c_2 \, (\log 2d)^{16} \, L^{16} + c_3 \,  G^{16}\big)\, \frac{k}{n^8}.
\end{align*}
Here $c_j$ ($j=1,2,3$) are constants independent of $(n, k, N, d, G, L)$.
\end{lemma}
The second result characterizes the error bound $\| \ttheta - \htheta\|_2$ in terms of the gradient norm $\|\nabla \tcL(\htheta)\|_2$ at $\htheta$, which formalizes the heuristic argument in Section~\ref{SectionDSL}. Its proof is provided in Appendix~\ref{AppProofLemErrToGrad}.
\begin{lemma}
\label{LemmaErrToGrad}
Suppose that Assumptions PA-PD hold. Then under event $\cE_0\cap \cE_1$ we have
\begin{align*}
\| \ttheta - \htheta\|_2 \leq \frac{2\, \|\nabla \tcL(\htheta)\|_2}{(1-\rho) \mum}.
\end{align*}
\end{lemma}

Therefore, it remains to prove a high-probability upper bound for the gradient 
norm $\|\nabla \tcL(\htheta)\|_2$.  A simple calculation yields
\begin{align}
\label{EqnZeroGrad}
\nabla \tcL(\htheta) = \nabla \cL_1(\htheta) - \nabla \cL_1(\btheta) + \nabla \cL_N(\btheta).
\end{align}
By the optimality of the global empirical risk minimizer $\htheta$, we have
\begin{align*}
\nabla \cL_N(\htheta) = 0.
\end{align*}
By adding and subtracting $\nabla_N(\htheta)$ in equation~\eqref{EqnZeroGrad}, we obtain
\begin{align}
\label{EqnScore}
\nabla \tcL(\htheta) = \big(\nabla \cL_1(\htheta) - \nabla \cL_1(\btheta)\big) -  \big(\nabla \cL_N(\htheta) - \nabla \cL_N(\btheta)\big).
\end{align}

By the integral form of Taylor's expansion, we have that for any $j\in\{1,\ldots,k\}$,
\begin{align*}
\nabla \cL_j(\htheta) - \nabla \cL_j(\btheta) = H_j \, (\htheta - \btheta),
\end{align*}
where $H_j = \int_{0}^1 \nabla^2 \cL_j \big (\btheta + t(\htheta - \btheta)\big) \, dt$ satisfies
\begin{align*}
\matnorm{H_j - \nabla^2 \cL_j(\thetas)}{2} \leq 2M\, (\|\btheta - \htheta\|_2 + \|\htheta - \thetas\|_2)
\end{align*}
under event $\cE_j$.
Combining the three preceding displays, we obtain that under event $\bigcap_{j=0}^k\cE_k$,
\begin{align*}
\|\nabla \tcL(\htheta)\|_2 &\leq \matnorm{H_1 - \nabla^2 \cL_1(\thetas)}{2}\, \|\htheta - \btheta\|_2 + \frac{1}{k}\, \sum_{j=1}^k \matnorm{H_j - \nabla^2 \cL_j(\thetas)}{2}\, \|\htheta - \btheta\|_2 \\
& \quad
+ \matnorm{\nabla^2 \cL_1(\thetas) - \nabla^2 \cL_N(\thetas)}{2}\, \|\htheta - \btheta\|_2\\
&\leq \big(2M\, \|\htheta - \btheta\|_2 + 2M\, \|\htheta - \thetas\|_2 + \matnorm{\nabla^2 \cL_1(\thetas) - \nabla^2 \cL_N(\thetas)}{2}\big)\, \|\htheta - \btheta\|_2.
\end{align*}
Combining Lemma~\ref{EqnZeroGrad} and the above display yields the claimed error bound on $\|\ttheta - \htheta\|_2$.


\subsection{Proof of Corollary~\ref{CoroMomentBound}}
\label{SectionProofCoroMomentBound}
Recall the definitions of the events $\{\cE_j\}_{j=0}^k$ in 
Section~\ref{SectionProofThmMestNoReg}. In the remaining of this 
proof, we use $C$ to denote some constant independent of $(n,k,N)$, 
whose magnitude may change from line to line.  We need the following 
auxiliary result, whose proof is provided in Appendix~\ref{AppProofLemSecondOrder}.
\begin{lemma}
\label{LemSecondOrder}
Under event $\bigcap_{j=0}^k \cE_j$, we have
\begin{align*}
&\Big\| \htheta - \thetas -  I(\thetas)^{-1} \, \nabla \cL_N(\thetas) \Big\|_2
\\
&\ \
\leq \frac{2}{(1-\rho) \mum} \, \matnorm{\nabla^2 \cL_N(\thetas) - I(\thetas)}{2} \, \|\nabla \cL_N(\thetas)\|_2 + \frac{8M}{(1-\rho)^2 \mum^2\, } \, \|\nabla \cL_N(\thetas)\|_2^2.
\end{align*}
\end{lemma}
Combining this Lemma, inequality~\eqref{EqnMLEerr} in Appendix~\ref{AppProofLemSmallProbEvent} and Theorem~\ref{ThmMestNoReg}, we obtain that under event $\bigcap_{j=0}^k \cE_j$,
\begin{align*}
\Big\| \ttheta - \thetas -  I(\thetas)^{-1} \, \nabla \cL_N(\thetas) \Big\|_2
\leq &\,\Big\| \htheta - \thetas -  I(\thetas)^{-1} \, \nabla \cL_N(\thetas) \Big\|_2 + \|\ttheta - \htheta\|_2\\
\leq &\,
 C \, \matnorm{\nabla^2 \cL_N(\thetas) - I(\thetas)}{2} \, \|\nabla \cL_N(\thetas)\|_2 + C \, \|\nabla \cL_N(\thetas)\|_2^2\\
 & + C\, \big(\|\btheta - \htheta\|_2 + \|\nabla \cL_N(\thetas)\|_2 + \matnorm{\nabla^2 \cL_1(\thetas) - \nabla^2 \cL_N(\thetas)}{2}\big)\, \|\btheta - \htheta\|_2.
\end{align*}
Now by applying H\"{o}lder's inequality and Lemma~\ref{LemmaMomentBound} in Appendix~\ref{AppProofLemSmallProbEvent}, we obtain
\begin{align*}
&\Ee\Big[ \Big\| \Big(\ttheta - \thetas -  I(\thetas)^{-1} \, \nabla \cL_N(\thetas)\Big)\, I\Big(\bigcap_{j=0}^k \cE_j\Big) \Big\|_2^2\Big] \\
 \leq &\, C\,  \sqrt{\Ee [\matnorm{\nabla^2 \cL_N(\thetas) - I(\thetas)}{2}^4] \ \Ee[\|\nabla \cL_N(\thetas)\|_2^4]} + C\,  \Ee[\|\nabla \cL_N(\thetas)\|_2^4]\\
&  + C\, \Ee[\|\btheta - \htheta\|_2^4] + C\, \sqrt{\Ee [\matnorm{\nabla^2 \cL_1(\thetas) - \nabla^2 \cL_N(\thetas)}{2}^4] +  \Ee [\|\nabla \cL_N(\thetas)\|_2^4]} \, \sqrt{\Ee[\|\btheta - \htheta\|_2^4]}\\
\leq &\, \frac{C}{N^2} + \frac{C}{n} \, \min\Big\{\frac{1}{n},\, \sqrt{\Ee[\|\btheta - \htheta\|_2^4]} \Big\}.
\end{align*}
Combining this with bound~\eqref{EqnSmallProbEvent} in Appendix~\ref{AppProofLemSmallProbEvent} on $\Pp(\bigcup_{j=0}^k \cE_j^c)$, we obtain that under Assumption PA,
\begin{align*}
\Ee\Big[ \Big\| \,\ttheta - \thetas -  I(\thetas)^{-1} \, \nabla \cL_N(\thetas) \Big\|_2^2\Big] \leq  \frac{C}{N^2} + \frac{C}{n} \, \min\Big\{\frac{1}{n},\, \sqrt{\Ee[\|\btheta - \htheta\|_2^4]} \Big\} + \frac{C\,k}{n^8},
\end{align*}
which implies the claimed bound on $\Ee[\|\ttheta - \thetas\|_2^2]$.


\subsection{Proof of Theorem~\ref{ThmNewtonerr}}
Before analyzing the one-step Newton-Raphson estimator $\theta^H$, 
we establish some auxiliary results. Recall that for $j=1,\ldots,k$, let 
$M_j = \frac1n \sum_{i=1}^n M(z_{ij})$ and $\delta_{\rho} = \min\{\rho,\, \rho\mum/4M\}$.
Analogously to the definition of the events $\cE_j$ ($j=0,1\ldots,k$) in 
Section~\ref{SectionProofThmMestNoReg}, we define the following ``good events":
\begin{align*}
\cE_0'& \defn \Big\{ \|\htheta - \thetas\|_2 \leq \frac{\mum}{4M}\Big\}, \qquad\mbox{and}\\
\cE_{j}' & \defn \Big\{ M_j \leq 2 M, \, \matnorm{\nabla^2 \cL_j(\thetas) - I(\thetas)}{2} \leq \frac{\rho \mum}{4}, \, \| \nabla \cL_j(\thetas) \|_2 \leq \frac{(1 - \rho)\mum \delta_\rho}{4}\Big\}.
\end{align*}
We then have that under Assumptions PA-PD,
\begin{align*}
\Pp\Big(\bigcup_{j=0}^k\cE_{j}'^c\Big) \leq \big(c_1' + c_2' \, (\log 2d)^{16} \, L^{16} + c_3' \,  G^{16}\big)\, \frac{k}{n^8},
\end{align*}
where $c_1', c_2'$ and $c_3'j$ are constants independent of $(n, k, N, d, G, L)$.

Use $\lambdamin(A)$ to denote the minimal eigenvalue of a symmetric matrix $A$.
\begin{lemma}
\label{LemNRauxiliary}
Assume that the conditions in Theorem~\ref{ThmNewtonerr} are true. Then under event $\bigcap_{j=0}^k \cE_j'$ we have
\begin{align*}
&\lambdamin[\nabla^2 \cL_N(\htheta)] \geq \frac{1}{2}\, (1 - \rho)\mum,\quad \|\btheta - \htheta\|_2 \leq \Delta\defn \frac{(1 - \rho)\mu}{8M},\\
&U_N \defn \max_{\theta \in (\htheta -\Delta, \, \htheta + \Delta)} \matnorm{\nabla^2 \cL_N (\theta)}{2} \leq U \defn 2M\Delta + \frac{\rho \mu }{4} + \mup, \quad\mbox{and}\\
&\matnorm{\nabla ^2 \cL_N (\btheta)^{-1}-\nabla ^2 \cL_1 (\btheta)^{-1}}{2} 
\leq   \Big(\frac{2M\rho}{\mum^{2}} + \frac{\rho + 4}{4\mum}\Big)\, \Big(\matnorm{\nabla ^2 \cL_N (\thetas)-\nabla ^2 \cL_1 (\thetas)}{2} + 4M \, \|\btheta - \thetas\|_2\Big).
\end{align*}
\end{lemma}
\noindent The proof of this lemma is provided in Appendix~\ref{AppProofLemNRauxiliary}.

Now we proceed to prove Theorem~\ref{ThmNewtonerr}. 
For the purpose of analysis, we define the global one-step Newton-Raphson estimator $\theta^N \defn \btheta - \nabla ^2 \cL_N (\btheta)$.

The error can be decomposed as
\begin{align*}
\theta^H - \htheta = (\theta^H - \theta^N) + (\theta^N - \htheta).
\end{align*}
We analyze the two terms respectively. The first term can be expressed as
\begin{align*}
\theta^H - \theta^N &= \big(\btheta - \nabla ^2 \cL_1 (\btheta)^{-1} \nabla \cL_N (\btheta) \big) - \big( \btheta - \nabla ^2 \cL_N (\btheta)^{-1} \nabla \cL_N (\btheta) \big) \\
&= \big( \nabla ^2 \cL_N (\btheta)^{-1}-\nabla ^2 \cL_1 (\btheta)^{-1}\big)\, \nabla \cL_N(\btheta)\\
&= \big( \nabla ^2 \cL_N (\btheta)^{-1}-\nabla ^2 \cL_1 (\btheta)^{-1}\big)\, \big(\nabla \cL_N(\btheta)- \nabla \cL_N (\htheta) \big),
\end{align*}
which yields the bound
\begin{align*}
\| \theta^H - \theta^N \|_2 &\le U_N \, \matnorm{\nabla ^2 \cL_N (\btheta)^{-1}-\nabla ^2 \cL_1 (\btheta)^{-1}}{2} \, \|\btheta- \htheta\|_2. 
\end{align*}
The second term can be analyzed  by using Theorem 5.3 in \cite{bubeck2014theory}, which guarantees that under the assumption $\|\btheta - \htheta\|_2 \le \frac{\mu_N}{2M_N}$, it holds that
\begin{align*}
\|\theta^N - \htheta \|_2 \le \frac{M_N }{ \mu_N } \|\btheta - \htheta\|_2^2, 
\end{align*}
where $\mu_N\defn \lambdamin[\nabla^2 \cL_N (\htheta)]$ and $M_N$ is the Lipschitz constant of the Hessian $\nabla^2 \cL_N (\theta)$, that is $\norm{ \nabla^2 \cL_N (\theta_1 ) - \nabla^2 \cL_N (\theta_2)}_2 \le M_N \norm{\theta_1-\theta_2}_2 $ for all $\theta_1,\theta_2 \in U(\rho)$.
Putting the pieces together, we obtain
\begin{align*}
\|\theta^H - \htheta\|_2 \leq U_N \, \matnorm{\nabla ^2 \cL_N (\btheta)^{-1}-\nabla ^2 \cL_1 (\btheta)^{-1}}{2}\, \|\btheta- \htheta\|_2 + \frac{M_N }{ \mu_N } \, \|\btheta - \htheta\|_2^2.
\end{align*}

Now the claimed bound on $\|\theta^H - \htheta\|_2$ is a direct consequence of the preceding display and Lemma~\ref{LemNRauxiliary}.


\subsection{Proof of Corollary~\ref{CoroAsympExpan}}
\label{SectionProofCoroAsympExpan}
The proof of the second part on the consistency of plug-in estimators for $\Sigma$ is standard by using the consistency of $\ttheta$ implied by the first part, the central limit theorem and Slutsky's theorem. Therefore we only prove the first part on the asymptotic expansion of $\ttheta$. Based on Theorem~\ref{ThmMestNoReg}, we only need to establish the asymptotic expansion~\eqref{EqnMLEAsymp} of the global empirical risk minimizer $\htheta$.
By the integral form of Taylor's expansion, we have
\begin{align*}
0 = \nabla \cL_N(\htheta) = \nabla \cL_N(\thetas) + H_N \, (\htheta - \thetas),
\end{align*}
where $H_N = \int_{0}^1 \nabla^2 \cL_N\big (\thetas + t(\htheta - \thetas)\big) \, dt$.
Then simple linear algebra yields
\begin{align}
\label{EqnMLEa}
\htheta - \thetas = - I(\thetas)^{-1} \, \nabla \cL_N(\thetas) 
- U_N \, (\htheta - \thetas) - V_N\, (\htheta - \thetas),
\end{align}
where $U_N = H_N - \nabla^2 \cL_N(\thetas)$ and $V_N = \nabla^2 \cL_N(\thetas) - I(\thetas)$. Then, the claimed expansion is an easy consequence of inequality~\eqref{EqnMLEerr} and Assumption D.


\subsection{Proofs for regularized M-estimators}
\begin{proof}[Proof of Theorem \ref{thm:l1-m-estimator}]
This theorem follows from applying Corollary 1 of \cite{negahban2012unified} to the objective $F(\theta)$. We check that $\tcL(\theta)$ satisfies the restricted strong convexity condition.

The restricted strong convexity of $\tcL$ is implied by the same property of $\cL_1$, since
\begin{align*}
\tcL(\thetas+\delta) - \tcL(\thetas) - \nabla \tcL(\thetas)^T \delta =\cL_1(\thetas +\delta) -\cL_1(\thetas) - \nabla \cL_1(\thetas)^T \delta.
\end{align*}
Thus by Corollary 1 of \cite{negahban2012unified}, we have established\[
\norm{\ttheta - \thetas} \le \frac{3\sqrt{s} \lambda}{\sqrt{\mu}},
\]
for $\lambda>  2 \norm{ \nabla \tcL(\thetas)}_\infty$. We can upper bound $\norm{\nabla \tcL(\thetas)}_\infty$ as follows:
\begin{align*}
\nabla \bar \cL(\thetas)& = \nabla \cL_1 (\thetas) - \nabla \cL_1 (\btheta) +  \cL_N (\btheta) \nonumber\\
&=  ( \nabla \cL_N (\btheta) - \nabla \cL_N (\thetas)) - ( \nabla \cL_1 (\btheta) - \nabla \cL_1 (\thetas) ) + \nabla \cL_N (\thetas) \nonumber\\
&=\nabla^2 \cL_N (\thetas) (\btheta- \thetas)-\nabla^2 \cL_1 (\thetas) (\btheta- \thetas)\nonumber\\
& + \int_{s=0}^{s=1} ds (\nabla^2 \cL_N(\thetas +s(\btheta-\thetas)) - \nabla^2 \cL_N (\thetas))(\btheta -\thetas) \nonumber
\\
&- \int_{s=0}^{s=1} ds (\nabla^2 \cL_1(\thetas +s(\btheta-\thetas)) - \nabla^2 \cL_1 (\thetas))(\btheta -\thetas)+\nabla \cL_N (\thetas)\nonumber
\\
\end{align*}
Using Assumption HB,
\begin{align*}
\norm{\nabla \tcL (\thetas)}_\infty &\le \norm{\nabla^2 \cL_N (\thetas) - \nabla^2\cL_1(\thetas)}_\infty \norm{\btheta- \thetas}_1 + \norm{\nabla \cL_N (\thetas)}_\infty \\
&+ 2M \norm{\btheta-\theta}^2 _2
\end{align*}
\end{proof}

\begin{proof}[Proof of Theorem \ref{thm:sparse-lr}]

To apply Theorem \ref{thm:l1-m-estimator}, we have to compute $\norm{\nabla^2 \cL_N(\thetas ) - \nabla^2 \cL_1(\thetas) }_\infty$. Let $\Sigma = E[xx^T]$.

\begin{align*}
\norm{\nabla^2 \cL_N(\thetas ) - \nabla^2 \cL_1(\thetas) }_\infty
&=\norm{ (\Sigma - \frac1N X^TX }_\infty+\norm{ \frac1n X_1 ^T X_1- \Sigma}_\infty 
+\sigma\sqrt{ \frac{2\log d}{N}}.
\end{align*}
By applying the sub-exponential concentration inequality, we have
$\Pr( \frac1N \sum_{i=1}^N ( |x_{ij} x_{ik} - \Sigma_{jk}|>t ) 
\le \exp( - c_{\Sigma} \min( t^2, t) N)$, where $c_{\Sigma}$ is a 
constant that depends on $\Sigma$. By a union bound over all $(j,k)$ pairs,
\[
\Pr(  \norm{\frac1N X^TX - \Sigma}_{\max} >t ) \le \exp(2 \log d - c_{\Sigma} \min( t^2, t) N).
\]
Thus, letting $t= C \sqrt{\frac{\log d}{N}}$, we have 
$\norm{\frac1N X^TX -\Sigma}_{\max} < C \sqrt{\frac{\log d}{N}}$ 
with probability greater than $1-1/p^{C'}$. By a similar argument, 
$\norm{\frac1n X_1 ^TX_1  -\Sigma}_{\max} < C \sqrt{\frac{\log d}{n}}$.

Since $\nabla^2 \cL$ is a constant in linear regression, $M=0$.
Thus 
\begin{align*}
\norm{\ttheta -\thetas}_2 \le \sqrt{\frac{s \log d}{n}} \norm{\btheta-\thetas}_1 + \sqrt{\frac{s \log d}{N}}.
\end{align*}
\end{proof}

\begin{proof}[Proof of Theorem \ref{thm:l1-logistic}]
To apply Theorem \ref{thm:l1-m-estimator}, we need to verify Assumptions 
HA and HB. The restricted strong convexity of $\cL$ is established in 
Proposition 1 of \cite{negahban2012unified}.  Next we verify Assumption HB:
\begin{align*}
\nabla^2\cL_1(\thetas + \delta ) - \nabla^2 \cL_1(\thetas) &= \frac1n \sum_{i=1}^n (\phi''(x_{ij}^T \thetas +x_{ij}^T  \delta) -\phi''(x_{ij}^T\thetas))x_{ij} x_{ij}^T\\
&=\frac1n \sum_{i=1}^n \phi'''(x_{ij}^T \thetas+ s_{ij}x_{ij}^T \delta )  x_{ij} (x_{ij}^T\delta)^2 .
\end{align*}
Thus,
\begin{align*}
\norm{(\nabla^2\cL_1(\thetas +s \delta ) - \nabla^2 \cL_1(\thetas) ) (\delta) }_\infty &\le \norm{\frac1n \sum_{i=1}^n \phi'''(x_{ij}^T\thetas+s_{ij} x_{ij}^T \delta )   x_{ij} (x_{ij}^T\delta)^2}\\
&\le  L_{\phi}B|\frac1n \sum_{i=1}^n (x_{ij}^T \delta)^2|\\
&\le L_{\phi} B L \norm{\delta}_2 ^2,
\end{align*}
Thus $M= L_{\phi} B L$, where $L_{\phi} $ is a local upper bound on $\phi'''$, $L$ is the upper restricted eigenvalue of $X$, and $B =\max \norm{x}_\infty$. 

We also need to compute an upper bound on $\norm{\nabla^2 \cL_N (\thetas) - \nabla^2 \cL_1 (\thetas) }_\infty $. Define $A= \E[ \phi''(x^T \thetas ) xx^T]$. 
\begin{align*}
\norm{\nabla^2\cL_N (\thetas) - \nabla^2 \cL_1 (\thetas) }_\infty &=\left(\frac1N \sum_{j=1}^k \sum_{i=1}^n \phi''(x_{ij}^T \thetas) x_{ij}x_{ij}^T - A \right)
+\left(A - \frac1n \sum_{i=1}^n  \phi''(x_{i1}^T \thetas) x_{i1}x_{i1}^T \right) \\
& \le C\sqrt{\frac{\log d}{N}} +  C\sqrt{\frac{\log d}{n}} ,
\end{align*}
where we used the same argument as in the proof of Theorem \ref{thm:sparse-lr}.

By Lemma 6 of \cite{negahban2012unified}, we know $\norm{\nabla \cL_N (\thetas)}_\infty \le C \sqrt{\frac{\log d}{N}}$.

Thus by Theorem \ref{thm:l1-m-estimator}, we have shown
\begin{align*}
\norm{\ttheta - \thetas}_2 \le C\sqrt{\frac{s\log d}{n}} \norm{\btheta- \thetas}_1  + \sqrt{\frac{s \log d}{N}} +C \norm{\btheta-\theta}^2 _2).
\end{align*}

\end{proof}


\subsection{Proof of Theorem~\ref{ThmAPOST}}
\label{SectionProofThmAPOST}
Recall the definition of the ``good events" $\cE_j$ for $j=1,\ldots,k$ in Section~\ref{SectionProofThmMestNoReg} as
\begin{align*}
\cE_{j} & \defn \Big\{M_j \leq 2 M, \, \matnorm{\nabla^2 \cL_j(\thetas) - I(\thetas)}{2} \leq \frac{\rho \mum}{4}, \, \| \nabla \cL_j(\thetas) \|_2 \leq \frac{(1 - \rho)\mum \delta_\rho}{4} \Big\}.
\end{align*}
Moreover, we define events
\begin{align*}
\cA_n &= \big\{\inf_{\|\theta - \thetas\|_2\geq \delta} \frac{1}{n}\,( \cL_1(\theta) - \cL_1(\thetas)) \geq 3\epsilon \big\},\\
\cB_1 &=\Big\{\|\btheta - \thetas\|_2 \leq \frac{\epsilon}{4R}\, \min\big\{ \frac{1}{2\sqrt{M}},\, \frac{1}{\rho\mum + 2\mup}\big\} \Big\},\quad\mbox{and}\\
\cB_2 &=\big\{ \sqrt{N}\, \mup\, \|\btheta - \htheta\|_2 + 2M\, \sqrt{N}\, \|\btheta - \htheta\|_2^2 + 2M\,\sqrt{N} \, \|\btheta - \htheta\|_2\,\|\htheta - \thetas\|_2 + M\,  \|\htheta - \thetas\|_2 \leq \mum/16\big\},
\end{align*}
where $\delta = \min\{\rho/2, (4M)^{-1}\mum\}$ and 
$$
\epsilon = 4R\, \min\Big\{\frac{(1 - \rho)\mum \delta_\rho}{2},\,  \frac{(1-\rho)^3\mum^2}{8M}, \, \frac{(1-\rho)^{3/2}\mum^{3/2}}{8M}  \Big\}.
$$
Then under the assumptions of the theorem and our previous developments in Section~\ref{SectionProofThmMestNoReg}, we have 
\begin{align}
\label{EqnHighprob}
\Pp\Big(\cA_n^c\cup \cB_1^c\cup \cB_2^c \cup\bigcup_{j=1}^k \cE_j^c\Big) \to 1,\quad\mbox{as }n\to\infty.
\end{align}

To prove the claimed result, we need three auxiliary lemmas.
The first lemma provides the local expansions of global loss function $\cL_N(\theta)$ and surrogate function $\tcL(\theta)$ around the global empirical loss minimizer $\htheta$. The proof is provided in Appendix~\ref{AppProofLemLAN}.
\begin{lemma}
\label{LemLAN}
Under event $\bigcap_{j=1}^k \cE_j$, we have that for all $\theta\in U(\rho)$,
\begin{align*}
\Big| \cL_N(\theta) &- \cL_N(\htheta) -  \frac{1}{2}\, \langle \theta - \htheta, I(\thetas)\, (\theta - \htheta) \rangle\Big| \\
&
\leq \Big(M\,  \|\htheta - \thetas\|_2  + \frac{1}{2k}\, \sum_{j=1}^k\matnorm{\nabla^2 \cL_j\big(\thetas) - I(\thetas)}{2}\Big)\, \|\theta - \htheta\|_2^2+M\,  \|\theta - \htheta\|_2^3,\quad\mbox{and}\\
\Big| \widetilde \cL\, (\theta) & - \widetilde \cL\, (\htheta) - \frac{1}{2}\, \langle \theta - \htheta, I(\thetas)\, (\theta - \htheta) \rangle \Big| \\
&\leq  A_n\, \|\theta - \htheta\|_2 +
 B_n\, \|\theta - \htheta\|_2^2+M\,  \|\theta - \htheta\|_2^3,
\end{align*}
where $A_n :\, =  \mup\, \|\btheta - \htheta\|_2 + 2M\, \|\btheta - \htheta\|_2^2 + 2M\, \|\btheta - \htheta\|_2\,\|\htheta - \thetas\|_2$ and $B_n :\,=M\,  \|\htheta - \thetas\|_2  + \frac{1}{2k}\, \sum_{j=1}^k\matnorm{\cL_j\big(\thetas) - I(\thetas)}{2}$.
\end{lemma}

Our second lemma shows that the global identifiability assumption PC for $\cL_1(\theta)$ implies the identifiability for the surrogate loss $\tcL(\theta)$. The proof is provided in Appendix~\ref{AppProofLemConsistency}.
\begin{lemma}
\label{LemConsistency}
Under the joint event $\cA_n\cap\cB_1\cap\bigcap_{j=1}^k\cE_j$, we have
\begin{align*}
\inf_{\|\theta - \thetas\|_2\geq \delta} ( \tcL(\theta) - \tcL(\thetas)) \geq 2\epsilon.
\end{align*}
\end{lemma}

Our final lemma shows that if the results in the previous two lemma hold, 
then we obtain a Bernstein-von Mises result for the approximated posterior 
$\apost$. The proof is provided in Appendix~\ref{AppProofLemBvmForAPOST}.
\begin{lemma}
\label{LemBvmForAPOST}
Suppose that the conclusions of Lemma~\ref{LemLAN} and Lemma~\ref{LemConsistency} are true. Then under the event $\cB_2$, we have
\begin{align}
\label{EqnBvM}
\Big\|\apost(\theta) - \mathcal{N}_d\big( \htheta,\, I(\thetas)^{-1}\big)\, (\theta) \Big\|_1 \leq C\, R,
\end{align}
where $\mathcal{N}_d(\mu, \Sigma)\, (\cdot)$ is the pdf of a $d$-dim Gaussian distribution with mean vector $\mu$ and covariance matrix $\Sigma$, and the remainder term 
\begin{align*}
R :\,=& A_n\, \sqrt{N}\, \log N + B_n \,(\log N)^2+ M\, N^{-1/2} \, (\log N)^3.
\end{align*}
Here $C$ is a constant independent of $(n,k,N)$.
\end{lemma}

Combining the three lemmas and the high-probability bound~\eqref{EqnHighprob}, we obtain that with probability tending to one, bound~\eqref{EqnBvM} holds.
Similarly, by considering the global posterior $\pi_N(\theta)$ as the approximated posterior $\apost(\theta)$ with $n=N$ and $k=1$, we obtain that
\begin{align}
\label{EqnBvMb}
\Big\|\pi_N(\theta) - \mathcal{N}_d\big( \htheta,\, I(\thetas)^{-1}\big)\, (\theta) \Big\|_1 \leq C\, R.
\end{align}
Combining \eqref{EqnBvM} and ~\eqref{EqnBvMb} yields a proof of the claimed result.

\section{Proof of the auxiliary results}

\subsection{Proof of Lemma~\ref{LemSmallProbEvent}}
\label{AppProofLemSmallProbEvent}
Apply Lemma 6 in \cite{Zhang13}, we obtain that under the event $\bigcap_{j=1}^k \cE_j$,
\begin{align}
\label{EqnMLEerr}
\|\htheta - \thetas\|_2 \leq \frac{2 \|\nabla \cL_N(\thetas) \|_2}{(1 - \rho) \mum},
\end{align}
where $\nabla \cL_N(\thetas)  = \frac 1k\,\sum_{j=1}^k\nabla \cL_j(\thetas)$.
In order to obtain high-probability bounds for $\nabla \cL_j(\thetas)$ and $\matnorm{\nabla^2 \cL_j(\thetas) - I(\thetas)}{2}$ for $j=1,\ldots,k$, we apply the following result.
\begin{lemma}[\cite{Zhang13}, Lemma 7]
\label{LemmaMomentBound}
Under Assumption PB and PD, there exist universal constants $c, c'$ such that for
$\nu \in \{1,\ldots,8\}$,
\begin{align*}
\Ee[ \|\nabla \cL_j(\thetas)\|_2^{2\nu} ] &\leq \frac{c\, G^{2\nu}}{n^\nu},\\
\Ee[\matnorm{\nabla^2 \cL_j(\thetas) - I(\thetas)}{2}^{2\nu} ] &\leq 
\frac{c' \, (\log 2d)^\nu \, L^{2\nu}}{n^{\nu}}.
\end{align*}
\end{lemma}
\noindent Now we apply Markov's inequality, Jensen's inequality and the union bound to obtain that there exist constants $c_1,c_2,c_3$ independent of $(n,k,N,d, G, L)$ such that
\begin{align}
\label{EqnSmallProbEvent}
\Pp\Big(\bigcup_{j=0}^k\cE_{j}^c\Big) \leq \big(c_1 + c_2 \, (\log 2d)^{16} \, L^{16} + c_3 \,  G^{16}\big)\, \frac{k}{n^8}.
\end{align}


\subsection{Proof of Lemma~\ref{LemmaErrToGrad}}
\label{AppProofLemErrToGrad}
We will apply Lemma 6 in \cite{Zhang13} with $\thetas = \htheta$ and $F_1 = \tcL$ in the notation therein. Since the Hessian of $\tcL$ is the same as that of $\cL_1$, in order to apply their result, we only need to verify that under event $\cE_0\cap \cE_1$, it holds that
\begin{align*}
\matnorm{\nabla^2 \cL_1(\htheta) - I(\thetas)}{2} \leq \frac{\rho \mu}{2}, \quad\mbox{and}\quad \| \nabla \cL_1(\htheta) \|_2 \leq \frac{(1 - \rho)\mum \delta_\rho}{2}.
\end{align*}
The first inequality is true since under event $\cE_0\cap \cE_1$ and Assumption D, we have
\begin{align*}
\matnorm{\nabla^2 \cL_1(\htheta) - I(\thetas)}{2} &\leq 2M\, \|\htheta - \thetas\|_2 + \matnorm{\nabla^2 \cL_1(\thetas) - I(\thetas)}{2} \leq \frac{\rho\mum}{4} + \frac{\rho\mum}{4} = \frac{\rho\mum}{2}.
\end{align*}
To prove the second inequality, we apply the integral form of Taylor's expansion to obtain that
\begin{align*}
\nabla \cL_1(\htheta) - \nabla \cL_1(\thetas) = H_1 \, (\htheta - \thetas),
\end{align*}
where matrix $H_1 = \int_{0}^1 \nabla^2 \cL_1 \big (\thetas + t(\htheta - \thetas)\big) \, dt$ satisfies 
$$
\matnorm{H_1 - I(\thetas)}{2} \leq 2M\,  \|\htheta - \thetas\|_2
$$
under event $\cE_1$.
Therefore, the triangle inequality yields that under event $\cE_0\cap\cE_1$,
\begin{align*}
\| \nabla \cL_1(\htheta) \|_2 & \leq \| \nabla \cL_1(\thetas) \|_2 + \matnorm{H_1 - I(\thetas)}{2} \, \|\htheta - \thetas\|_2 + \matnorm{ I(\thetas)}{2} \, \|\htheta - \thetas\|_2\\
&\leq \frac{(1 - \rho)\mum \delta_\rho}{4} + 2M\, \|\htheta - \thetas\|_2^2 + \mup \, \|\htheta - \thetas\|_2\\
&\leq \frac{(1 - \rho)\mum \delta_\rho}{2}.
\end{align*}
This proves the second inequality and therefore the claimed result.


\subsection{Proof of Lemma~\ref{AppProofLemSecondOrder}}
\label{AppProofLemSecondOrder}
The claimed inequality is a immediate consequence of equation~\eqref{EqnMLEa} in Section~\ref{SectionProofCoroAsympExpan} and inequality~\eqref{EqnMLEerr} in Appendix~\ref{AppProofLemSmallProbEvent}.


\subsection{Proof of Lemma~\ref{LemNRauxiliary}}
\label{AppProofLemNRauxiliary}
Under Assumption D and event $\bigcap_{j=0}^k\cE_j'$, we can bound $\nabla^2 \cL_N(\htheta)$ as
\begin{align*}
\lambdamin[\nabla^2 \cL_N(\htheta)] &\geq \lambdamin[I(\thetas)] - \matnorm{\nabla^2 \cL_N(\thetas) - I(\thetas)}{2} - \matnorm{\nabla^2 \cL_N(\htheta) - \nabla^2 \cL_N(\thetas)}{2}\\
& \geq \mum - \frac{\rho \mum}{2} - 2M \, \|\htheta - \thetas\|_2  \geq \frac{1}{2}\, (1 - \rho) \mum.
\end{align*}
This proves the first claimed inequality.

The claimed inequality $\|\btheta - \htheta\|_2 \leq \frac{(1 - \rho)\mu}{8M}$ is immediate under the definition of $\cE_0$ and the condition $\|\htheta - \thetas\|_2 \leq \frac{(1 - \rho)\mu}{16M}$.

Under event $\bigcap_{j=0}^k\cE_j'$, the third inequality can be proved as
\begin{align*}
U_N &\leq \max_{\theta \in (\htheta-\Delta, \htheta+ \Delta)} \matnorm{\nabla^2 \cL_N(\theta) - \nabla^2 \cL_N(\thetas)}{2} +  \matnorm{\nabla^2 \cL_N(\thetas) - I(\thetas)}{2} + \matnorm{I(\thetas)}{2}\\
&\leq 2M\Delta + \frac{\rho \mu }{4} + \mup.
\end{align*}

To bound the term $\matnorm{\nabla ^2 \cL_N (\btheta)^{-1}-\nabla ^2 \cL_1 (\btheta)^{-1}}{2}$, we make use of the following inequality: for any matrix $A\in \R^{d\times d}$,
\begin{align}
\label{EqnNormInverse}
\matnorm{(A + \Delta A)^{-1} - A^{-1}}{2}
\leq  \matnorm{A^{-1}}{2}^2 \, \matnorm{\Delta A}{2}.
\end{align}
First, choose $A = I(\thetas)$ and $\Delta A = \nabla ^2 \cL_N (\btheta) - I(\thetas)$ in ~\eqref{EqnNormInverse}. Note that under the event $\bigcap_{j=0}^k \cE_j'$, we have
\begin{align*}
\matnorm{\nabla ^2 \cL_N (\btheta) - I(\thetas)}{2} &\leq 
\matnorm{\nabla ^2 \cL_N (\btheta) - \nabla ^2 \cL_N (\thetas)}{2} +
\matnorm{\nabla ^2 \cL_N (\thetas) - I(\thetas)}{2}\\
&\leq 2M \, \|\btheta - \thetas\|_2 + \matnorm{\nabla ^2 \cL_N (\thetas) - I(\thetas)}{2}.
\end{align*}
Therefore, we have that under the event $\bigcap_{j=0}^k \cE_j'$
\begin{align*}
\matnorm{\nabla ^2 \cL_N (\btheta)^{-1}}{2} &\leq \matnorm{\nabla ^2 \cL_N (\btheta)^{-1} - I(\thetas)^{-1}}{2} + \mum^{-1}\\
& \leq  2M \, \mum^{-2}\,\|\btheta - \thetas\|_2 + \mum^{-2}\,\matnorm{\nabla ^2 \cL_N (\thetas) - I(\thetas)}{2}+ \mum^{-1}\\
&\leq 2M \, \mum^{-2}\, \rho + \mum^{-1} + \mum^{-1}\,\rho/4,
\end{align*}
where in the last step we used the assumption that $\|\btheta - \thetas\|_2 \leq \rho$ and the definition of events $\cE_j'$'s.
Now choosing $A = \nabla ^2 \cL_N (\btheta)$ and $\Delta A = \nabla ^2 \cL_1 (\btheta) - \nabla ^2 \cL_N (\btheta)$ in inequality~\eqref{EqnNormInverse}, we obtain
\begin{align*}
&\matnorm{\nabla ^2 \cL_N (\btheta)^{-1}-\nabla ^2 \cL_1 (\btheta)^{-1}}{2} \notag\\
\leq &\, \matnorm{\nabla^2 \cL_N(\btheta)^{-1}}{2}^2\, \matnorm{\nabla ^2 \cL_N (\btheta)-\nabla ^2 \cL_1 (\btheta)}{2}\\
\leq &\, \matnorm{\nabla^2 \cL_N(\btheta)^{-1}}{2}^2\, \Big(\matnorm{\nabla ^2 \cL_N (\thetas)-\nabla ^2 \cL_1 (\thetas)}{2} + \matnorm{\nabla ^2 \cL_N (\btheta) - \nabla ^2 \cL_N (\thetas)}{2} + \matnorm{\nabla ^2 \cL_1 (\btheta) - \nabla ^2 \cL_1 (\thetas)}{2} \Big)\\
\leq &\, \matnorm{\nabla^2 \cL_N(\btheta)^{-1}}{2}^2\, \Big(\matnorm{\nabla ^2 \cL_N (\thetas)-\nabla ^2 \cL_1 (\thetas)}{2} + 4M \, \|\btheta - \thetas\|_2\Big).
\end{align*}
Putting the pieces together, we have proved the final claimed inequality.


\subsection{Proof of Lemma~\ref{LemLAN}}
\label{AppProofLemLAN}

To prove the first expansion for $\cL_N(\theta)$, it suffices to prove the following inequality by using the fact that $\nabla \cL_N(\htheta) = 0$:
\begin{align}
\label{EqnL_j}
\Big| \cL_j(\theta)& - \cL_j(\htheta) - \langle \nabla \cL_j(\htheta), \theta - \htheta \rangle - \frac{1}{2}\, \langle \theta - \htheta, I(\thetas)\, (\theta - \htheta) \rangle\Big| \notag\\
&
\leq \Big(M\,  \|\htheta - \thetas\|_2  + \frac{1}{2}\, \matnorm{\nabla^2\cL_j\big(\thetas) - I(\thetas)}{2}\Big)\, \|\theta - \htheta\|_2^2+M\,  \|\theta - \htheta\|_2^3
\end{align}
for $j=1,\ldots,k$.
In fact, by Taylor's theorem, we have
\begin{align*}
\cL_j(\theta) - \cL_j(\htheta) = \langle \nabla \cL_j(\htheta), \theta - \htheta \rangle & + \frac{1}{2}\, \langle \theta - \htheta, I(\thetas)\, (\theta - \htheta) \rangle\\
&
+ \frac{1}{2}\, \langle \theta - \htheta, \big(\widetilde{H}_j -  I(\thetas)\big)
\, (\theta - \htheta) \rangle, 
\end{align*}
where $\widetilde{H}_j = \nabla^2 \cL_j\big(\htheta + t_j(\theta - \htheta)\big)$ for some $t_j\in[0,1]$.
Under event $\cE_j$, we can bound the last remainder term for $\theta \in U(\rho)$ by
\begin{align*}
\frac{1}{2}\, \|\theta - \htheta\|_2^2&\, \big( \matnorm{\widetilde{H}_j - \nabla^2\cL_j\big(\thetas)}{2} + \matnorm{\nabla^2 \cL_j\big(\thetas) - I(\thetas)}{2}\big)  \\
\leq &\, 
M\,  \|\theta - \htheta\|_2^3 + \Big(M\,  \|\htheta - \thetas\|_2  + \frac{1}{2}\, \matnorm{\nabla^2 \cL_j\big(\thetas) - I(\thetas)}{2}\Big)\, \|\theta - \htheta\|_2^2,
\end{align*}
which yields the expansion~\eqref{EqnL_j}.

To prove the expansion for $\tcL(\theta)$, we note that simple calculation yields
\begin{align*}
\widetilde \cL(\theta) - \widetilde \cL(\htheta) = \cL_1(\theta) - \cL_1(\htheta) + \langle \nabla \cL_N(\btheta) - \nabla\cL_1(\btheta), \theta - \htheta \rangle.
\end{align*}
Given the first expansion for $\cL_N(\theta)$ and the expansion~\eqref{EqnL_j} for $\cL_1(\theta)$, we only need to show that under the joint event $\bigcap_{j=1}^k \cE_j$,
\begin{align*}
&\, \Big|\langle \nabla \cL_N(\btheta) - \nabla \cL_N(\htheta), \theta - \htheta \rangle\Big| + \Big|\langle \nabla \cL_1(\btheta) - \nabla \cL_1(\htheta), \theta - \htheta \rangle\Big| \leq  A_n \, \|\theta - \htheta\|_2,
\end{align*}
because $\nabla \cL_N(\htheta) = 0$.
This is true since by using the integral form of Taylor's expansion and Cauchy-Schwarz inequality, the left-hand side in the preceding display can be bounded by
\begin{align*}
\matnorm{\widehat{H}_N}{2}&\, \|\btheta - \htheta\|_2 \,\|\theta - \htheta\|_2 +
\matnorm{\widehat{H}_j}{2}\, \|\btheta - \htheta\|_2 \,\|\theta - \htheta\|_2  \\
\leq &\, 2\big(\mup + 2M\, \|\btheta - \htheta\|_2 + 2M\, \|\htheta - \thetas\|_2\big)  \, \|\btheta - \htheta\|_2 \,\|\theta - \htheta\|_2 = A_n \,\|\theta - \htheta\|_2,
\end{align*}
where $\widehat{H}_N = \int_0^1 \nabla^2 \cL_N\big(\htheta + t(\btheta - \htheta)\big)\, dt$, $\widehat{H}_j = \int_0^1 \nabla^2 \cL_j\big(\htheta + t(\btheta - \htheta)\big)\, dt$ and in the last step we used the fact that under the joint event $\bigcap_{j=1}^k \cE_j$, $\matnorm{\widehat{H}_j - I(\thetas)}{2} \leq 2M\, (\|\btheta - \htheta\|_2 + \|\htheta - \thetas\|_2)$ for each $j$ and $\matnorm{\widehat{H}_N - I(\thetas)}{2}\leq k^{-1}\, \sum_{j=1}^k \matnorm{\widehat{H}_j - I(\thetas)}{2}$.


\subsection{Proof of Lemma~\ref{LemConsistency}}
\label{AppProofLemConsistency}
By Assumption PA, $\|\theta - \thetas\|_2 \leq R$ for all $\theta \in \Theta$. Therefore, by the definition of events $\cE_j$ and $\cA_n$, we obtain
\begin{align*}
&\inf_{\|\theta - \thetas\|_2\geq \delta} ( \tilde\cL(\theta) - \tilde\cL(\thetas)) \\
\geq &\, \inf_{\|\theta - \thetas\|_2\geq \delta} ( \cL_1(\theta) - \cL_1(\thetas)) - \sup_{\theta\in\Theta}\,  \langle \theta - \thetas, \nabla \cL_1 (\btheta) - \nabla \cL_N(\btheta)\rangle\\
\geq &\, 3\epsilon - R \, \big( \|\nabla \cL_1 (\btheta) - \nabla \cL_1 (\thetas)\|_2 + \|\nabla \cL_1 (\thetas)\|_2 + \|\nabla \cL_N (\thetas)\|_2 + \|\nabla \cL_N (\btheta) - \nabla \cL_N (\thetas)\|_2\big).
\end{align*}
Now we bound the four terms inside the brackets under the event $\cB_n$, respectively, as
\begin{align*}
&\|\nabla \cL_1 (\btheta) - \nabla \cL_1 (\thetas)\|_2 \leq \max_{\theta\in U(\rho)} \matnorm{\nabla^2 \cL_1 (\theta)}{2} \, \|\btheta - \thetas\|_2 \\
&\qquad\qquad\qquad\qquad
\leq  (M\, \|\btheta - \thetas\|_2 + \frac{\rho\mum}{2} + \mup)\,  \|\btheta - \thetas\|_2 \leq \frac{\epsilon}{4R},\\
&\|\nabla \cL_1 (\thetas)\|_2 \leq \min\Big\{\frac{(1 - \rho)\mum \delta_\rho}{2},\,  \frac{(1-\rho)^3\mum^2}{8M}, \, \frac{(1-\rho)^{3/2}\mum^{3/2}}{8M}  \Big\} \leq 
\frac{\epsilon}{4R},\\
&\|\nabla \cL_N (\thetas)\|_2 \leq \frac{1}{k}\,\sum_{j=1}^k\|\nabla \cL_j (\thetas)\|_2 \leq \frac{\epsilon}{4R}, \quad\mbox{and}\\
&\|\nabla \cL_N (\btheta) - \nabla \cL_N (\thetas)\|_2 \leq \frac{1}{k}\,\sum_{j=1}^k\|\nabla \cL_j (\btheta) - \nabla \cL_j (\thetas)\|_2 \leq \frac{\epsilon}{4R}.
\end{align*}

Putting the pieces together, we obtain that under the joint event $\cA_n\cap\cB_1\cap\bigcap_{j=1}^k\cE_j$,
\begin{align*}
\inf_{\|\theta - \thetas\|_2\geq \delta} \frac{1}{n}\,( \tilde\cL(\theta) - \tilde\cL(\thetas))  \geq 3\epsilon - R \cdot 4\cdot \frac{\epsilon}{4R} =2\epsilon,
\end{align*}
which completes the proof.


\subsection{Proof of Lemma~\ref{LemBvmForAPOST}}
\label{AppProofLemBvmForAPOST}

The approximated posterior can be expressed by
\begin{align*}
\apost(\theta) = \frac{\pi(\theta)\, e^{-N(\tilde\cL(\theta) - \tilde\cL(\htheta))}}
{\int_\Theta  \pi(\theta)\, e^{-N(\tilde\cL(\theta) - \tilde\cL(\htheta))}\, d\theta}.
\end{align*}
We claim that it suffices to prove 
\begin{equation}
\label{EqnGoal}
\begin{aligned}
&\Big| \pi(\theta)\, e^{-N(\tilde\cL(\theta) - \tilde\cL(\htheta))} - \pi(\htheta\, ) \, e^{-\frac{N}{2}\,  \langle \theta - \htheta, I(\thetas)\, (\theta - \htheta) \rangle}\Big| \\
&\qquad\qquad\leq C\, R\,  \pi(\ttheta\, ) \, e^{-\frac{N}{4}\,  \langle \theta - \ttheta, I(\thetas)\, (\theta - \ttheta) \rangle} + \pi(\theta)\, e^{-N\epsilon}.
\end{aligned}
\end{equation}
In fact, if \eqref{EqnGoal} holds, then by integrating $\theta$ over $\R^d$, we obtain
\begin{align*}
\Big| N^{d/2} \int_\Theta  \pi(\theta)\, e^{-N(\tilde\cL(\theta) - \tilde\cL(\htheta))}\, d\theta - \pi(\htheta\, ) \, \frac{(2\pi)^{d/2}}{\sqrt{\det I(\thetas)}} \Big| \leq C\, R\, + C\, N^{(d-1) /2} e^{-N\epsilon} \leq C' R.
\end{align*}
Then, by combining all three preceding displays, we obtain
\begin{align*}
\int_{\R^d}\Big| \apost(\theta) - \frac{N^{d/2}\,\sqrt{\det I(\thetas)}}{(2\pi)^{d/2}} \,
e^{-\frac{N}{2}\,  \langle \theta - \htheta, I(\thetas)\, (\theta - \htheta) \rangle}\Big| \, d\theta \leq C''\,R,
\end{align*}
which is the claimed Bernstein-von Mises result for $\apost$.
The remainder of the proof focuses on proving \eqref{EqnGoal}.

Let $s = \sqrt{N}\, (\theta - \htheta)$ be the localized parameter. Then \eqref{EqnGoal} is equivalent to
\begin{equation}
\label{EqnGoalb}
\begin{aligned}
&\Big| \pi(\htheta + s/\sqrt{N})\, e^{-N(\tilde\cL(\htheta + s/\sqrt{N}) - \tilde\cL(\htheta))} - \pi(\htheta\, ) \, e^{-\frac{1}{2}\,  \langle s, I(\thetas)\, s \rangle}\Big| \\
&\qquad\qquad\leq C\, R\,  \pi(\htheta\, ) \, e^{-\frac{1}{4}\,  \langle s, I(\thetas)\, s \rangle} +  \pi(\htheta + s/\sqrt{N})\, e^{-N\epsilon}.
\end{aligned}
\end{equation}

\noindent Corollary~\ref{LemLAN} guarantees that for all $\|s\|\leq \delta\,\sqrt{N}$,
\begin{equation}
\label{EqnLANForS}
\begin{aligned}
&\Big| N(\widetilde \cL(\htheta + s/\sqrt{N}) - \widetilde \cL(\thetas)) - \frac{1}{2}\, \langle s, I(\thetas)\, s \rangle \Big| \\
&\qquad\qquad\qquad\qquad\leq  A_n\, \sqrt{N}\, \|s\|_2 +
B_n \, \|s\|_2^2+\frac{M}{\sqrt{N}}\,  \|s\|_2^3.
\end{aligned}
\end{equation}
We prove \eqref{EqnGoalb} by considering $s$ in the following three subsets separately:
\begin{align*}
S_1 & :\, = \{s:\, \|s\|_2 \leq c\,\log N\}\\
S_2 & :\, = \{s:\, c\,\log N \leq\|s\|_2 \leq \delta\, \sqrt{N}\}\\
S_3 & :\, = \{s:\, \|s\|_2 > \delta\, \sqrt{N}\}.
\end{align*}
We begin with $s\in S_1$. Using \eqref{EqnLANForS}, we obtain that
\begin{align*}
&\Big| \pi(\htheta + s/\sqrt{N})\, e^{-N(\tilde\cL(\htheta + s/\sqrt{N}) - \tilde\cL(\thetas))} - \pi(\htheta\, ) \, e^{-\frac{1}{2}\,  \langle s, I(\thetas)\, s \rangle}\Big| \\
\leq &\, \Big| \pi(\htheta + s/\sqrt{N})\, e^{-N(\tilde\cL(\htheta + s/\sqrt{N}) - \tilde\cL(\thetas))} - \pi(\htheta + s/\sqrt{N}) \, e^{-\frac{1}{2}\,  \langle s, I(\thetas)\, s \rangle}\Big| \\
& \qquad\qquad\qquad\qquad+ \big| \pi(\htheta + s/\sqrt{N}) - \pi(\htheta)\big| \, e^{-\frac{1}{2}\,  \langle s, I(\thetas)\, s \rangle}\\
\leq &\, C\,\pi(\htheta)\, e^{-\frac{1}{2}\,  \langle s, I(\thetas)\, s \rangle} \, \big(A_n\, \sqrt{N}\, \log N + B_n (\log N)^2 + M\, N^{-1/2} \, (\log N)^3\big)\\
& \qquad\qquad\qquad\qquad+  C\, \frac{\log N}{\sqrt{N}} \, e^{-\frac{1}{2}\,  \langle s, I(\thetas)\, s \rangle}\\
\leq &\, C\, R\,  \pi(\ttheta\, ) e^{-\frac{1}{4}\,  \langle s, I(\thetas)\, s \rangle}.
\end{align*}

Next consider $s\in S_2$. Then $\|s\|_2 \leq \|s\|_2^2$ for sufficiently small constant $c$. Under the event $\cB_2$, we have $A_n\, \sqrt{N} \, \|s\|_2^2 + B_n \,\|s\|_2^2 + M\, N^{-1/2}\, \|s\|_2^3 \leq \langle s, I(\thetas)\, s \rangle /4$.
Then using \eqref{EqnLANForS}, we obtain
\begin{align*}
\Big| N(\widetilde \cL(\htheta + s/\sqrt{N}) - \widetilde \cL(\thetas)) - \frac{1}{2}\, \langle s, I(\thetas)\, s \rangle \Big| \leq \frac{1}{4}\, \langle s, I(\thetas)\, s \rangle.
\end{align*}
Therefore, we have
\begin{align*}
&\Big| \pi(\htheta + s/\sqrt{N})\, e^{-N(\tilde\cL(\htheta + s/\sqrt{N}) - \tilde\cL(\thetas))} - \pi(\htheta\, ) \, e^{-\frac{1}{2}\,  \langle s, I(\thetas)\, s \rangle}\Big| \\
\leq &\, \pi(\htheta + s/\sqrt{N})\, e^{-N(\tilde\cL(\ttheta + s/\sqrt{N}) - \tilde\cL(\thetas))}  +\pi(\htheta\, ) \, e^{-\frac{1}{2}\,  \langle s, I(\thetas)\, s \rangle}\\
\leq&\, C\, (R_1 + 1)\, \pi(\htheta) e^{-\frac{1}{4}\,  \langle s, I(\thetas)\, s \rangle}\\
\leq &\, C\, R \,\pi(\htheta\,) \,e^{-\frac{1}{4}\,  \langle s, I(\thetas)\, s \rangle}.
\end{align*}

For $s\in S_3$, we have $\|\htheta + s/\sqrt{N} - \thetas\|_2 = \|\theta - \thetas\|_2 \geq \|\theta - \htheta\|_2 - \|\htheta - \thetas\|_2 \geq \delta$. The proof of Lemma~\ref{LemConsistency} shows that under the joint event $\cA_n\cap\cB_n\cap\bigcap_{j=1}^k\cE_j$, we have
$$
\sup_{\theta,\theta'\in\Theta} \Big| \tilde\cL(\theta) - \tilde\cL(\theta') - \cL_1(\theta) - \cL_1(\theta')\Big| \leq \epsilon.
$$
Then we obtain
\begin{align*}
e^{-N(\tilde\cL(\htheta + s/\sqrt{N}) - \tilde\cL(\htheta\, ))}  \leq&\,  e^{N\epsilon}\, e^{-N(\cL(\htheta + s/\sqrt{N}) - \cL(\htheta\, ))}\\
= &\, 
 e^{N\epsilon} \,  e^{-N(\cL(\htheta + s/\sqrt{N}) - \cL(\thetas)}\,   e^{-N( \cL(\thetas) - \cL(\htheta\, ))}
\overset{(i)}{\leq} e^{-N\epsilon},
\end{align*}
where in step (i) we used Lemma~\ref{LemConsistency} and the optimality of $\htheta$ that implies $\cL(\thetas) - \cL(\htheta\, )\geq 0$.
Therefore, we have
\begin{align*}
&\Big| \pi(\htheta + s/\sqrt{N})\, e^{-N(\tilde\cL(\htheta + s/\sqrt{N}) - \tilde\cL(\thetas))} - \pi(\htheta\, ) \, e^{-\frac{1}{2}\,  \langle s, I(\thetas)\, s \rangle}\Big| \\
\leq &\, \pi(\htheta + s/\sqrt{N})\, e^{-N\epsilon} +  C\, e^{-\mu\,\delta^2\,  N/2}\, \pi(\htheta) \,e^{-\frac{1}{4}\,  \langle s, I(\thetas)\, s \rangle} \\
\leq &\, C\, R \, \pi(\htheta) \,e^{-\frac{1}{4}\,  \langle s, I(\thetas)\, s \rangle} +  \pi(\htheta + s/\sqrt{N})\, e^{-N\epsilon}.
\end{align*}

Putting the pieces together, we can prove \eqref{EqnLANForS} and therefore the claimed Bernstein-von Mises result for $\apost$.

\newpage

%
%

\end{document}